\documentclass[11pt]{article}

\usepackage[margin=1in]{geometry}

\usepackage{graphicx}

\usepackage{caption}
\usepackage{subcaption}

\usepackage{xcolor}
\definecolor{navy}{rgb}{0, 0, 0.75}

\usepackage{amsmath}
\usepackage{amsthm}
\usepackage{amssymb}
\usepackage{mathrsfs}
\usepackage{bm}

\usepackage[shortlabels]{enumitem}

\usepackage[colorlinks, allcolors=navy, linktocpage=true]{hyperref}
\usepackage[capitalize]{cleveref}

\usepackage{algorithm}[noend]
\usepackage{algpseudocode}

\numberwithin{equation}{section}
\theoremstyle{plain}
\newtheorem{theorem}[equation]{Theorem}
\newtheorem{corollary}[equation]{Corollary}
\newtheorem{lemma}[equation]{Lemma}

\theoremstyle{definition}
\newtheorem{definition}[equation]{Definition}

\newtheorem{remark}[equation]{Remark}

\DeclareMathOperator*{\argmin}{argmin}

\newcommand{\prox}{\mathrm{prox}}
\newcommand{\spn}{\mathrm{span}}

\newcommand{\abs}[1]{\lvert#1\rvert}
\newcommand{\norm}[1]{\lVert#1\rVert}
\newcommand{\Abs}[1]{{\left\lvert#1\right\rvert}}
\newcommand{\Norm}[1]{{\left\lVert#1\right\rVert}}

\newcommand{\R}{\mathbb{R}}
\newcommand{\N}{\mathbb{N}}

\newcommand{\cS}{\mathcal{S}}
\newcommand{\cT}{\mathcal{T}}
\DeclareMathOperator*{\E}{\mathbb{E}}

\newcommand{\iid}{\overset{\textit{iid}}{\sim}}
\newcommand{\toP}{\xrightarrow{\textit{P}}}

\newcommand{\hu}{\widehat{u}}
\newcommand{\hv}{\widehat{v}}
\newcommand{\hb}{\widehat{\beta}}
\newcommand{\tu}{\widetilde{u}}
\newcommand{\tv}{\widetilde{v}}
\newcommand{\tb}{\widetilde{\beta}}
\newcommand{\eps}{\varepsilon}

\newcommand{\pe}{{\varepsilon_{\mathrm{DP}}}}
\newcommand{\tpe}{{\widetilde{\varepsilon}_{\mathrm{DP}}}}
\newcommand{\hpe}{{\widehat{\varepsilon}_{\mathrm{DP}}}}
\newcommand{\pd}{{\delta_{\mathrm{DP}}}}
\newcommand{\pa}{{\alpha_{\mathrm{DP}}}}
\newcommand{\pr}{{\rho_{\mathrm{DP}}}}

\newcommand\blfootnote[1]{
    \begingroup \renewcommand\thefootnote{}\footnote{#1}
    \endgroup
}

\title{
Differentially Private Learning \\ Beyond the Classical Dimensionality Regime\blfootnote{This work was supported in part by Simons Foundation Grant 733782, Cooperative Agreement CB20ADR0160001 with the United States Census Bureau, and NSF CAREER DMS-2340241.}
}

\author{
    Cynthia Dwork \\
    Harvard University \\
    \url{dwork@seas.harvard.edu}
\and
    Pranay Tankala \\
    Harvard University \\
    \url{pranay_tankala@g.harvard.edu}
\and
    Linjun Zhang \\
    Rutgers University \\
    \url{linjun.zhang@rutgers.edu}
}

\date{February 18, 2025}

\begin{document}

\maketitle

\begin{abstract}
    We initiate the study of differentially private learning in the \emph{proportional dimensionality regime}, in which the number of data samples $n$  and problem dimension $d$ approach infinity at rates proportional to one another, meaning that $d / n \to \delta$ as $n \to \infty$ for an arbitrary, given constant $\delta \in (0, \infty)$. This setting is significantly more challenging than that of all prior theoretical work in high-dimensional differentially private learning, which, despite the name, has assumed that $\delta = 0$ or is sufficiently small for problems of sample complexity $O(d)$, a regime typically considered ``low-dimensional'' or ``classical'' by modern standards in high-dimensional statistics.

    We provide sharp theoretical estimates of the error of several well-studied differentially private algorithms for robust linear regression and logistic regression, including output perturbation, objective perturbation, and noisy stochastic gradient descent, in the proportional dimensionality regime. The $1 + o(1)$ factor precision of our error estimates enables a far more nuanced understanding of the price of privacy of these algorithms than that afforded by existing, coarser analyses, which are essentially vacuous in the regime we consider. Using our estimates, we discover a previously unobserved ``double descent''-like phenomenon in the training error of objective perturbation for robust linear regression.  We also identify settings in which output perturbation outperforms objective perturbation on average, and vice versa, demonstrating that the relative performance of these algorithms is less clear-cut than suggested by prior work.
    
    To prove our main theorems, we introduce several probabilistic tools that have not previously been used to analyze differentially private learning algorithms, such as a modern Gaussian comparison inequality and recent universality laws with origins in statistical physics.
\end{abstract}

\newpage
\tableofcontents

\newpage
\section{Introduction}
\label{sec:introduction}

Over the last two decades, a groundbreaking line of research has demonstrated the limitations of classical statistical theory for modern high-dimensional learning problems \cite{donoho2009message,donoho2011noise,bayati2011amp,bayati2011lasso,el2013robust,karoui2013asymptotic, sur2019modern,candes2020phase,han2023universality,wang2024universality,han2024entrywise}. For example, although the classical theory of logistic regression predicts that the maximum likelihood estimator (MLE) is asymptotically unbiased with variance given by the inverse Fisher information, these predictions fail dramatically in the \emph{proportional dimensionality regime}, in which the number data samples $n$ and problem dimension $d$ approach infinity at rates proportional to one another. Indeed, \cite{sur2019modern} showed that even if $n = 5d$ or $n = 10d$, the MLE can be biased, with greater variance than predicted by the classical theory. In a result we find particularly intriguing, they further show that the MLE fails to exist with high probability when the dimensionality ratio $d/n$ crosses a certain \emph{phase transition} threshold. Despite these challenges, this proportional high-dimensional regime is highly relevant in modern data science, where technological advancements generate datasets with feature dimensions comparable to, or even exceeding, sample sizes.
 
Motivated by these insights and the growing importance of privacy-preserving learning, we initiate the study of differentially private learning in the proportional dimensionality regime. 

As an illustrative example, consider the task of fitting a $d$-dimensional linear regression model to $n$ labeled data samples $(\bm{x}_i, y_i) \in \R^{d} \times \R$. For simplicity, suppose that $\bm{\beta}^\star \in [-1, +1]^d$ denotes the unknown regression coefficients, that the feature vectors $\bm{x}_i$ are subgaussian with covariance $\frac{1}{d}\bm{I}_d$, and that the labels $y_i$ have constant variance given $\bm{\beta^\star}$ and $\bm{x}_i$. In this setting, it is known that there exists an $(\pe, \pd)$-differentially private algorithm, due to \cite{liu2022hdptr}, such that for any sufficiently small $\alpha > 0$, the algorithm's output $\widehat{\bm{\beta}}$ has error\footnote{The factor of $1/d$ in the estimation error is appropriate because the $\bm{x}_i$ have covariance $\frac{1}{d}\bm{I}_d$. More generally, when the $\bm{x}_i$ have covariance $\bm{\Sigma} \in \R^{d \times d}$, a natural notion of estimation error is $\norm{\bm{\Sigma}^{1/2}(\bm{\beta} - \bm{\beta}^\star)}^2$. A good rule of thumb is that in regression, the ``scale'' of $\bm{\beta}$ ought to be inversely proportional to the ``scale'' of $\bm{x}$. All our statements regarding $\norm{\bm{\beta}^\star} \le O(\sqrt{d})$ and $\norm{\bm{x}_i} \le O(1)$ and error metric $\frac{1}{d}\norm{\bm{\beta} - \bm{\beta}^\star}^2$ can be straightforwardly reformulated as statements regarding $\norm{\bm{\beta}^\star} \le O(1)$ and $\norm{\bm{x}_i} \le O(\sqrt{d})$ and error metric $\norm{\bm{\beta} - \bm{\beta}^\star}^2$.} \(\frac{1}{d}\norm{\widehat{\bm{\beta}} - \bm{\beta}^\star}^2 \le \alpha^2\) with probability $99\%$, and the algorithm only uses a sample of size
\[
    n \le O{\left(\frac{d}{\alpha^2} + \frac{d + \log(1/\pd)}{\alpha \pe}\right)} \cdot \mathrm{polylog}{\left(\frac{1}{\alpha}\right)}.
\]
This sample complexity upper bound is tight up to polylogarithmic factors \cite{cai2021cost}, but nevertheless yields vacuous error bounds in the proportional regime. Indeed, even if we were to set $\pe = 1$ and $\pd = 1/n^2$, this upper bound would only guarantee that $\alpha \le O(1)$ because $d/n = \Theta(1)$ in the proportional regime. In the absence of knowledge of the constant suppressed by the $O(\cdot)$ notation, this bound on $\alpha$ is no better than that of the trivial estimator $\widehat{\bm{\beta}}_{\mathrm{trivial}} = \bm{0}$, which has error
\[
    \frac{1}{d}\norm{\bm{\hb}_{\mathrm{trivial}} - \bm{\beta}^\star}^2 = \frac{1}{d}\norm{\bm{\beta}^\star}^2 \le 1.
\]
Therefore, the above-mentioned upper bound leaves unanswered the question of whether the algorithm in question exhibits a nontrivial privacy-utility tradeoff in the proportional regime, unless the constant $\delta$ to which $d/n$ converges is sufficiently small, a regime considered ``low-dimensional'' or ``classical'' by modern standards \cite{wainwright2019statistics}. 
Similar issues arise in other recent works on private linear regression, which assume either that $n = \widetilde{O}(d)$ or in some cases $n = O(d)$ for a suppressed, unoptimized constant factor \cite{wang2018revisiting, sheffet2019techniques, varshney2022nearly, brown2024ols, brown2024gradient}. In this example of regression in the proportional regime, everything from the most trivial estimators to the most sophisticated algorithms achieves constant error, rather than error converging to zero. The relevant task then becomes understanding the precise values of these constants, which is of course not just of theoretical interest. Indeed, in genomic, neuroscience, and image data from various domains, the dimensionality $d$ of the data is not only proportional to $n$, but often even larger than $n$, and so in all of these settings, we must go beyond the standard $O(d/n)$ or $O(\sqrt{d/n})$ convergence rates. Instead, a more refined approach, relying on the solution to a system of nonlinear equations, is necessary to capture the nuanced dependence on $d/n$ in the proportional dimensionality regime.

This example extends beyond linear regression, with similar challenges arising in other forms of regression and in areas such as heavy-tailed moment estimation \cite{kamath2020heavy, liu2021robust, hopkins2022sosexp, agarwal2024person}, Gaussian parameter estimation \cite{kamath2019privately, liu2021robust, ashtiani2022gaussians, kothari2022robust, kamath2022unbounded, alabi2023gaussian, hopkins2023robustness, brown2023affine, kuditipudi2023pretty}, high-dimensional hypothesis testing and selection \cite{bun2019hypothesis, canonne2019hypothesis, gopi2020locally, narayanan2022hypothesis, pour2024locally}, and empirical risk minimization \cite{chaudhuri2011differentially, kifer2012private, bassily2014erm}, which have all received sustained attention from the differentially private algorithm design community.

\subsection{Results Overview}

Our first main contribution is the first tight utility analysis of several differentially private learning algorithms in the proportional dimensionality regime. We determine the utility, including the optimal constants, of three prominent private estimation algorithms:
\begin{enumerate}[(1)]
    \item In Sections~\ref{sec:main-body-objective-perturbation-linear} and \ref{sec:main-body-objective-perturbation-logistic}, we determine the error of the objective perturbation algorithm, due to \cite{chaudhuri2011differentially, kifer2012private}. Our bounds are tight up to a $1 \pm n^{-\Omega(1)}$ factor for robust linear regression and a $1 \pm e^{-(\log n)^{\Omega(1)}}$ factor for logistic regression, and they depend in nuanced ways on the problem parameters, such as $\pe$, $\pd$, and the dimensionality ratio $\delta$. Using these estimates, we discover a previously unobserved ``double descent''-like phenomenon in the \emph{training error} of objective perturbation for robust linear regression, which is distinct from the better-known double descent curve for the \emph{test error} of non-private linear regression.

    In \cref{sec:main-body-privacy}, we also give a refined privacy bound for the objective perturbation algorithm, which extends the analysis of \cite{redberg2023improving} to arbitrarily small regularization strengths $\lambda > 0$ and perturbation magnitudes $\nu > 0$.
    
    \item In \cref{sec:main-body-output-perturbation}, we determine similarly tight error estimates for the output perturbation algorithm, due to \cite{dwork2006dp, chaudhuri2011differentially}. Using our estimates, we show that the relative performance of objective perturbation and output perturbation is more nuanced than suggested by previous work. Indeed, while some prior analyses \cite{chaudhuri2011differentially, kifer2012private} suggest that objective perturbation is better overall, we identify some parameter settings for which output perturbation does better on average.
    
    \item In \cref{sec:main-body-dp-sgd}, we determine the error of a version of the noisy stochastic gradient descent algorithm (DP-SGD) of \cite{song2013dpsgd, bassily2014erm} that is restricted to a constant number of iterations. These estimates are accurate up to a $1 + o(1)$ factor, and they apply to certain variants of robust linear regression and logistic regression.
\end{enumerate}

Our second main contribution is the introduction of two powerful techniques from the high-dimensional statistics literature that have not previously been used to analyze differentially private learning algorithms. The first technique is the \emph{Convex Gaussian Minimax Theorem} (CGMT) of \cite{stojnic2013framework, thrampoulidis2015regularized}, which relates certain min-max optimization problems defined in terms of Gaussian random \emph{matrices} in $\R^{n \times d}$ to simpler problems defined in terms of Gaussian random \emph{vectors} in $\R^n$ and $\R^d$. The second technique is the \emph{universality law} of \cite{han2023universality, han2024entrywise}, which allow us to apply CGMT to a wide range of non-Gaussian data distributions, including bounded data arising from the feature- or gradient-clipping subroutines of differentially private algorithms. Both CGMT and universality were developed to analyze non-private estimators, and a major part of our proofs, which we sketch in \cref{sec:techniques}, is verifying that the differentially private algorithms of interest to us satisfy the technical requirements needed to apply these techniques.

\subsubsection{Utility of Objective Perturbation for Robust Linear Regression}
\label{sec:main-body-objective-perturbation-linear}

Our first result is a sharp utility analysis in the proportional regime of the classic \emph{objective perturbation} algorithm for empirical risk minimization, which was introduced for pure differential privacy by \cite{chaudhuri2011differentially}, extended to approximate differential privacy by \cite{kifer2012private}, and studied further in several subsequent theoretical and empirical works \cite{jain2014risk, iyengar2019convex, neel2020oracle, redberg2023improving}. Given a data set \((\bm{x}_1, y_1), \ldots, (\bm{x}_n, y_n) \in \R^d \times \R\) and a loss function $\ell(\bm{\beta}; (\bm{x}, y))$, the algorithm samples $\bm{\xi} \sim \mathcal{N}(\bm{0}, \bm{I}_d)$ and outputs
\[
    \widehat{\bm{\beta}} = \argmin_{\bm{\beta} \in \R^d}\, \sum_{i=1}^n \ell(\bm{\beta}; (\bm{x}_i, y_i)) + \frac{\lambda}{2}\norm{\bm{\beta}}^2 + \nu\langle \bm{\xi}, \bm{\beta}\rangle.
\]
We call $\lambda > 0$ and $\nu > 0$ the \emph{regularization} and \emph{perturbation} strengths, respectively.

Our first result concerns the error of objective perturbation for \emph{robust linear regression}, in which the standard squared loss function for linear regression is replaced with the $L$-Lipschitz \emph{Huber loss} function, due to \cite{huber1964robust}, defined as $\ell(\bm{\beta}; (\bm{x}, y)) = H_L(y - \langle \bm{x}, \bm{\beta}\rangle)$, where
\[
    H_L(r) = \begin{cases} \frac{1}{2}r^2 &\text{if }\abs{r} \le L \\ L\abs{r}-\frac{1}{2}L^2 &\text{if } \abs{r} \ge L.\end{cases}
\]

We consider data $(\bm{X}, \bm{y}) \in \R^{n \times d} \times \R^n$ generated from a \emph{linear model}, meaning that there exists a ground-truth vector $\bm{\beta}^\star \in \R^d$ and error vector $\bm{\eps}^\star \in \R^n$ independent of $\bm{X}$ for which
\(
    \bm{y} = \bm{X}\bm{\beta}^\star + \bm{\eps}^\star.
\)
We also assume for the utility analysis that the feature matrix $\bm{X}$ is average-case in a certain sense, which is standard in the differentially private estimation literature. For concreteness, in our informally stated \cref{thm:main-objective-huber-informal} below, the reader can imagine that the entries are independent, random elements of $\{-1/\sqrt{d}, +1/\sqrt{d}\}$, but the full version, \cref{thm:main-huber-objective-perturbation} in \cref{sec:objective-perturbation-linear}, will apply to various distributions with different conditions on the means and variances. As always, privacy holds in a worst-case sense and requires no distributional assumptions. We state the privacy bound in terms of \emph{zero-concentrated} differential privacy (zCDP), due to \cite{dwork2016concentrateddifferentialprivacy, bun2016zcdp}, which we define in \cref{sec:preliminaries}. We also use the abbreviation $[r]_L = \min(L, \max(-L, r))$ for scalars $r \in \R$ and $L > 0$.

\begin{theorem}[Informal Version of \cref{thm:main-huber-objective-perturbation}]
\label{thm:main-objective-huber-informal}
    Let $(\sigma^\star, \tau^\star)$ denote the solution to the following system of two scalar equations in two variables $(\sigma, \tau)$, which we write in terms of $L, \lambda, \nu, \delta \in (0, \infty)$, dummy variables $Z_1, Z_2 \iid \mathcal{N}(0, 1)$, and $\kappa^2 = \frac{1}{d}\norm{\bm{\beta}^\star}^2$ and $\eps^2 = \frac{1}{n}\norm{\bm{\eps}^\star}^2$ as follows:
    \begin{align*}
        \sigma^2 &= \tau^2\left(\frac{1}{\delta} \E\mathopen{}\left[\frac{\sigma Z_1 + \varepsilon Z_2}{1 + \tau}\right]^2_L + \lambda^2\kappa^2 + \nu^2\right), \\
        \tau &= \frac{1}{\lambda \delta}\left(\delta - \frac{\tau}{1+\tau}\Pr\mathopen{}\left[-L < \frac{\sigma Z_1 + \varepsilon Z_2}{1 + \tau} < L\right]\right).
    \end{align*}
    The output $\widehat{\bm{\beta}}$ of the objective perturbation algorithm with Huber loss satisfies $\pr$-zCDP for a constant $\pr = \pr(L, \lambda, \nu)$ independent of $n$ that is finite for any $L,\lambda,\nu \in (0, \infty)$. If the data are generated from a linear model as $n \to \infty$ and $d/n \to \delta$, then w.h.p., $\widehat{\bm{\beta}}$ has estimation error \[\frac{1}{d}\norm{\widehat{\bm{\beta}} - \bm{\beta}^\star}^2 =  (\sigma^\star)^2 \pm n^{-\Omega(1)}.\]
    
\end{theorem}

We will sketch the proof of \cref{thm:main-objective-huber-informal} in \cref{sec:techniques}, but we defer the full proof to \cref{sec:objective-perturbation-linear}. As already discussed, standard $O(d/n)$ or $O(\sqrt{d/n})$ error rates are vacuous in the proportional regime, so \cref{thm:main-objective-huber-informal} instead captures the nuanced dependence on $d/n$ through a small system of nonlinear, scalar equations that depend on $\delta$. To quickly gain intuition for these equations, it is instructive to consider simplified versions that arise in the limit as $L \to \infty$ or $\nu \to 0$ or $\lambda \to 0$ even though, strictly speaking, our theorem only holds for fixed values of $L, \lambda, \nu \in (0, \infty)$.

First, consider the limit as $\nu \to 0$, which corresponds to removing the random linear perturbation term that was introduced for privacy. In this case, we recover a result in Section 5.2 of \cite{thrampoulidis2018mestimators} for non-private ridge regression with Gaussian features, which was later extended to subgaussian features by Theorem 3.12 of \cite{han2023universality}. Indeed, these works characterize the error of the non-private, ridge regularized Huber regressor via the solution to a pair of equations in $(\sigma, \tau)$, which match ours in the case of $\nu = 0$. In Figure~\ref{fig:huber-objective-estimation}, we numerically compare the solution to our equations with $\nu > 0$ to those of prior work with $\nu = 0$ to precisely quantify the price of privacy paid by objective perturbation for robust linear regression.

Next, consider the limit $L \to \infty$, which corresponds to approximating Huber loss by squared loss. In this case, the equations simplify substantially:
\begin{align*}
    \sigma^2 = \tau^2\left(\frac{1}{\delta} \cdot \frac{\sigma^2 + \eps^2}{(1+\tau)^2} + \lambda^2\kappa^2 + \nu^2\right), \qquad \tau = \frac{1}{\lambda \delta}\left(\delta - \frac{\tau}{1+\tau}\right).
\end{align*}
In the non-private case ($\nu = 0$), solving for the error $\sigma^\star$ in the limit as $\lambda \to 0$ yields
\[
    (\sigma^\star)^2 \to \begin{cases}
        \frac{\delta}{1-\delta} \eps^2 &\text{if } \delta < 1, \\
        \frac{1}{\delta - 1} \eps^2 + \frac{\delta-1}{\delta}\kappa^2 &\text{if }\delta > 1.
    \end{cases}
\]
Similarly, in the private case ($\nu > 0$), in the limit as $\lambda \to 0$, the error satisfies
\[
    (\sigma^\star)^2 \to \begin{cases}
        \frac{\delta}{1-\delta} \eps^2 + \frac{\delta^2}{(1-\delta)^3}\nu^2 &\text{if } \delta < 1, \\
        +\infty &\text{if }\delta > 1.
    \end{cases}
\]
In other words, in the $\delta < 1$ regime, where $n > d$ objective perturbation incurs a mild additive penalty in squared error compared to the non-private case. However, in the $\delta > 1$ regime, where $n < d$, any constant noise level $\nu > 0$, no matter how small, sends $\sigma^\star$ from a finite constant to $+\infty$, a dramatic deterioration. This behavior is exhibited in the top two estimation error plots of Figure~\ref{fig:huber-objective-estimation}, where we see a sharp phase transition at  $n / (n + d) = 1/2$ (the region to the right of this boundary corresponds to $n > d$ or $\delta < 1$, and the region to the left of this boundary corresponds to  $n < d$ or $\delta > 1$). The ``spike'' in estimation error at $n = d$ in the non-private case is an example of the well-known \emph{double descent} phenomenon in modern machine learning \cite{belkin2019reconciling, belkin2020two, mei2022generalization}, and it is notable that it does not occur in private case of fixed $\nu > 0$.

However, we reach the opposite conclusion if we examine in-sample \emph{training error} (norm of the residuals $\bm{y} - \bm{X}\bm{\hb}$) instead of out-of-sample \emph{test error} (norm of $\bm{\hb} - \bm{\beta}^\star$). Indeed, while \cref{thm:main-objective-huber-informal} only estimates the squared norm of $\bm{\hb} - \bm{\beta}^\star$, the full version (\cref{thm:main-huber-objective-perturbation}) will also consider various other error metrics, including $\ell^p$ norms for all $p \ge 1$, as well as various residual norms. Again considering the limits $L \to \infty$ and $\lambda \to 0$, \cref{thm:main-huber-objective-perturbation} predicts that
\[
    \frac{1}{n}\norm{\bm{y} - \bm{X}\bm{\hb}}^2 \to \begin{cases}
        (1-\delta)\eps^2 + \frac{\delta^2}{1-\delta} \nu^2 &\text{if }\delta < 1,\\
        \frac{\delta}{\delta - 1}\nu^2 &\text{if }\delta > 1.
    \end{cases}
\]
Strikingly, this formula shows a descent-like ``spike'' in training error that emerges \emph{only} in the private case. Indeed, if $\nu = 0$, the residual error is $0$ for $n < d$ and monotonically increases to $\eps^2$ as $n$ grows from $d$ to $\infty$, but if $\nu > 0$, then there is a singularity at $\delta = 1$ on the order of $O(1/\abs{\delta-1})$. This behavior is exhibited in the bottom two residual error plots of Figure~\ref{fig:huber-objective-estimation}.

\begin{figure}[t]
    \centering
    \includegraphics[width=0.495\linewidth]{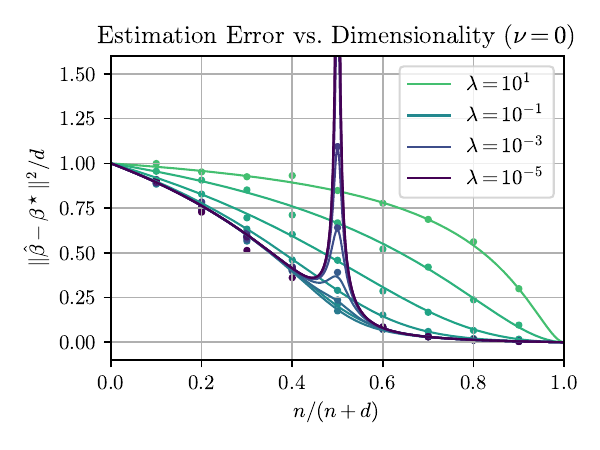}
    \includegraphics[width=0.495\linewidth]{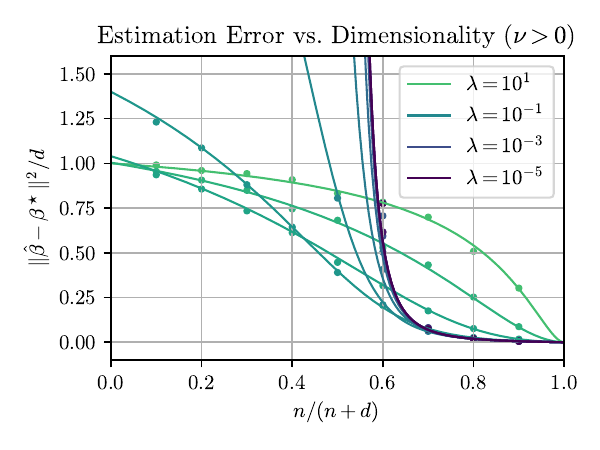}
    \includegraphics[width=0.495\linewidth]{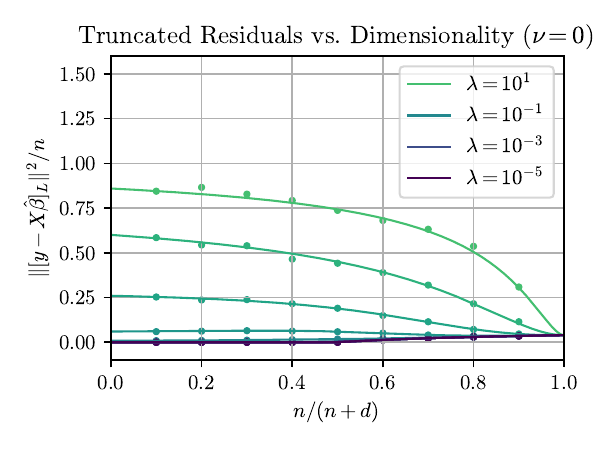}
    \includegraphics[width=0.495\linewidth]{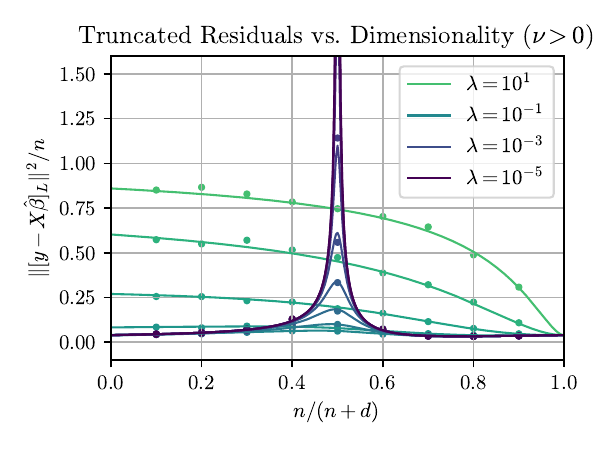}
    
    \caption{\cref{thm:main-objective-huber-informal}'s predictions of the \emph{estimation error} and \emph{truncated residuals} of objective perturbation with Huber loss. Larger $\delta \in (0, \infty)$ corresponds to smaller $n/(n+d) \in (0, 1)$. Curves indicate theoretical predictions, and dots indicate the mean over $100$ simulations of the algorithm on synthetic data with $n \times d = 1000$. In the left plots, the perturbation magnitude is $\nu = 0$, corresponding to the non-private case, but in the right plots, $\nu = 1/5$. All plots use $L = 10$, $\kappa = 1$, $\bm{\beta}^\star \sim \mathcal{N}(\bm{0}, \kappa^2\bm{I}_d)$, $\bm{\eps}^\star \sim \mathcal{N}(\bm{0}, (1/5)^2\bm{I}_n)$, $\bm{X} \sim \frac{1}{\sqrt{d}} \mathrm{Uniform}(\{-1,+1\}^{n \times d})$, and $\bm{y} = \bm{X}\bm{\beta}^\star + \bm{\eps}^\star$.}
    \label{fig:huber-objective-estimation}
\end{figure}

\begin{remark}
\label{rmk:assumptions}
    Although the incorporation of $\nu$ is the main quantitative difference between our theorem and those of \cite{thrampoulidis2018mestimators, han2023universality}, our theorem also comes with a number of qualitative advantages that will become more apparent once we have presented its full version (\cref{thm:main-huber-objective-perturbation}). For example, while Theorem 3.12 of \cite{han2023universality} assumes that $\bm{\beta}^\star$ and $\bm{\eps}^\star$ are random with i.i.d. coordinates, our theorem does not even require them to be random (although, in the deterministic case, we do require certain mild conditions on their coordinates as $n \to \infty$). Next, while \cite{thrampoulidis2018mestimators, han2023universality} only determine the limiting value of $\frac{1}{d}\norm{\bm{\hb} - \bm{\beta}^\star}^2$, we also use $\sigma^\star$ and $\tau^\star$ to characterize a wide range of other scalar quantities of interest regarding the algorithm's output $\bm{\hb}$, such as all of its $\ell^p$ distances to the ground-truth $\bm{\beta}^\star$, its correlation with the coefficients $\bm{\xi}$ of the random linear perturbation term, the error of its predictions $\bm{X}\bm{\hb}$ on the training data, and more.
\end{remark}

Further intuition for the meaning of $\sigma^\star$ and $\tau^\star$ can be gleaned from \cref{thm:main-huber-objective-perturbation}, which will roughly show that, for $\bm{z} \sim \mathcal{N}(\bm{0}, \bm{I}_d)$ independent of the perturbation vector $\bm{\xi}$,
\[
    \bm{\hb} \approx (1 - \tau^\star \lambda)\bm{\beta}^\star - \tau^\star \nu \bm{\xi} + \sqrt{(\sigma^\star)^2 - (\tau^\star)^2(\lambda^2\kappa^2+\nu^2)} \cdot \bm{z}.
\]
In this form, it is clear that a larger $\sigma^\star$ leads to a larger coefficient on the $\bm{z}$ term, so
$\sigma^\star$ is related to the \emph{variance} of $\bm{\hb} - \bm{\beta}^\star$. Similarly, $\tau^\star$ is related to the correlations of $\bm{\hb}$ with both $\bm{\xi}$ and $\bm{\beta^}\star$, the latter of which is related to the \emph{bias} of $\bm{\hb}$. Formalizing this intuition will be easy to do given the full version of the theorem.

\subsubsection{Utility of Objective Perturbation for Logistic Regression}
\label{sec:main-body-objective-perturbation-logistic}

Our next result concerns \emph{logistic regression}, where $\bm{y} \in \{0, 1\}^n$. The logistic loss of $\bm{\beta}$ on a sample $(\bm{x}, y)$ is $\rho(\langle \bm{x}, \bm{\beta}\rangle)$ if $y = 0$ and $\rho(-\langle \bm{x}, \bm{\beta}\rangle)$ if $y = 1$, where
\(
    \rho(t) = \log(1 + e^t).
\)
We consider data $(\bm{X}, \bm{y}) \in \R^{n \times d} \times \{0, 1\}^n$ generated from a \emph{logistic model}, meaning that there exists $\bm{\beta}^\star \in \R^d$ for which
\(
    \bm{y} \sim \mathrm{Bernoulli}(\rho'(\bm{X}\bm{\beta}^\star)).
\)
Here, $\rho'(t) = 1/(1+e^{-t})$ is the \emph{sigmoid} function, which outputs values between $0$ and $1$.

The subsequent theorem, much like the previous one, will involve a small number of scalar equations. We state them in terms of the \emph{proximal operator} of the function $\gamma \cdot \rho$ for a scalar $\gamma > 0$, denoted $\mathrm{prox}_{\gamma \rho}$. Intuitively, this function takes as input a point $s \in \R$ and outputs another point $t \in \R$ that approximately minimizes $\gamma \rho$ while maintaining proximity to $s$. For more detail, see \cref{sec:preliminaries} or \cite{parikh2014proximal}.

\begin{theorem}[Informal Version of \cref{thm:logistic-objective-perturbation}]
\label{thm:main-objective-logistic-informal}
    Let $(\alpha^\star, \sigma^\star, \gamma^\star)$ denote the solution to the following system of three scalar equations in three variables $(\alpha, \sigma, \gamma)$, which we write in terms of the algorithm's parameters $\lambda, \nu > 0$, the dimensionality ratio $\delta \in (0, \infty)$, dummy variables $Z_1, Z_2 \iid \mathcal{N}(0, 1)$ and $\kappa^2 = \frac{1}{d}\norm{\bm{\beta}^\star}^2$, as follows:
    \begin{align*}
        \sigma^2 &= \gamma^2\left(\frac{1}{\delta}\E\mathopen{}\left[2\rho'(-\kappa Z_1)\rho'(\prox_{\gamma\rho}(\kappa\alpha Z_1+\sigma Z_2) \big)^2\right] + \nu^2\right), \\
        \alpha &= -\frac{1}{\delta}\E[ 2\rho''(-\kappa Z_1)\prox_{\gamma \rho}\big(\kappa \alpha Z_1+\sigma Z_2\big)], \\
        \gamma &= \frac{1}{\lambda \delta}\mathopen{}\left(\delta - 1 + \E\mathopen{}\left[\frac{2\rho'(-\kappa Z_1)}{1+\gamma \rho''\big(\prox_{\gamma\rho}(\kappa\alpha Z_1 + \sigma Z_2)\big)}\right]\right).
    \end{align*}
    The output $\widehat{\bm{\beta}}$ of the objective perturbation algorithm with logistic loss satisfies $\pr$-zCDP for a constant $\pr = \pr(\lambda, \nu)$ independent of $n$ that is finite for any $\lambda, \nu \in (0, \infty)$. If the data are generated from a logistic model as $n \to \infty$ and $d/n \to \delta$, then w.h.p., $\widehat{\bm{\beta}}$ has estimation error \[\frac{1}{d}\norm{\widehat{\bm{\beta}} - \bm{\beta}^\star}^2 = (1 - \alpha^\star)^2\kappa^2 + (\sigma^\star)^2 \pm e^{-(\log n)^{\Omega(1)}}.\]
\end{theorem}

Much like our results for robust linear regression, \cref{thm:main-objective-logistic-informal} is a differentially private analogue of existing, non-private results in the literature. Taking $\nu \to 0$, we recover Theorem 2 of \cite{salehi2019impact} for regularized logistic regression with Gaussian features, which was later extended to subgaussian features by Theorem 4.3 of \cite{han2024entrywise}. Taking $\lambda \to 0$ as well, we recover the landmark result of \cite{sur2019modern} on the behavior of the (non-regularized) MLE, up to some minor differences in notation and scaling. In \cref{sec:techniques}, we sketch the proof of \cref{thm:main-objective-logistic-informal}. In \cref{sec:objective-perturbation-logistic}, we state a rigorous version of the theorem (\cref{thm:logistic-objective-perturbation}), provide a complete proof, and visualize the predictions (Figure~\ref{fig:logistic-objective-estimation}). We remark that the full version of the theorem will also have some qualitative advantages over prior results akin to those discussed in Remark~\ref{rmk:assumptions}.

As in the previous section, further intuition for the meaning of $\alpha^\star$, $\sigma^\star$, and $\gamma^\star$ can be gleaned from \cref{thm:logistic-objective-perturbation}, which will roughly show that, for $\bm{z} \sim \mathcal{N}(\bm{0}, \bm{I}_d)$ independent of $\bm{\xi}$, 
\[
    \bm{\hb} \approx \alpha^\star \bm{\beta}^\star - \gamma^\star \nu\bm{\xi} + \sqrt{(\sigma^\star)^2 - (\gamma^\star\nu)^2} \cdot \bm{z}.
\]
Consequently, $\alpha^\star$ is related to the bias of $\bm{\hb}$ and $\sigma^\star$ is related to the variance of the debiased estimate $\bm{\hb} - \alpha^\star \bm{\beta}^\star$. Similarly, $\gamma^\star$ is related to the correlation of $\bm{\hb}$ with $\bm{\xi}$.

\subsubsection{Privacy of Objective Perturbation with Small \texorpdfstring{$\lambda$}{Lambda}}
\label{sec:main-body-privacy}

Theorems \ref{thm:main-objective-huber-informal} and \ref{thm:main-objective-logistic-informal} precisely determine the utility of objective perturbation as a function of $\lambda$ and $\nu$, but these informal statements leave the algorithm's privacy loss parameter $\pr(\lambda, \nu)$ unspecified. The following result allows us to quantify the privacy-utility tradeoff of the algorithm:

\begin{theorem}[Informal Version of \cref{thm:rho-zcdp-bound}]
\label{thm:informal-privacy-bound}
    If $\ell$ is an $L$-Lipschitz, $s$-smooth, generalized linear model (GLM) loss, then objective perturbation satisfies $\pr$-zCDP for
    \[
        \pr = \log{\left(1 + \frac{s}{\lambda}\right)} + \frac{L^2}{2\nu^2} + \sqrt{\frac{2}{\pi}} \cdot \frac{L}{\nu}.
    \]
\end{theorem}

In \cref{sec:objective-perturbation-privacy}, we give a full statement (\cref{thm:rho-zcdp-bound}) and proof of this theorem. There, we also prove even tighter $(\pe, \pd)$-DP guarantees. Importantly, our DP guarantees hold for any fixed $\lambda, \nu > 0$ and do not degrade as $n,d \to \infty$. Our proofs refine closely related arguments of \cite{redberg2023improving}, which establish only privacy in the case where $\lambda > s$ (e.g. $s = 1$ for robust linear regression, or $s = 1/4$ for logistic regression). Note that the privacy loss $\pr$ shrinks to $0$ as $\lambda, \nu \to \infty$.

\subsubsection{Utility of Output Perturbation}
\label{sec:main-body-output-perturbation}

\begin{figure}
    \centering
    \includegraphics[width=0.495\linewidth]{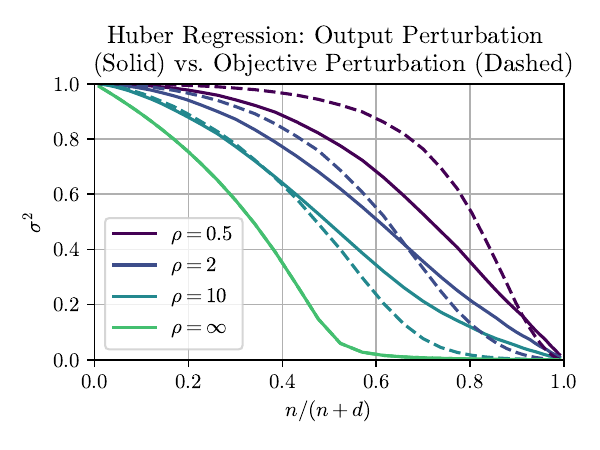}
    \includegraphics[width=0.495\linewidth]{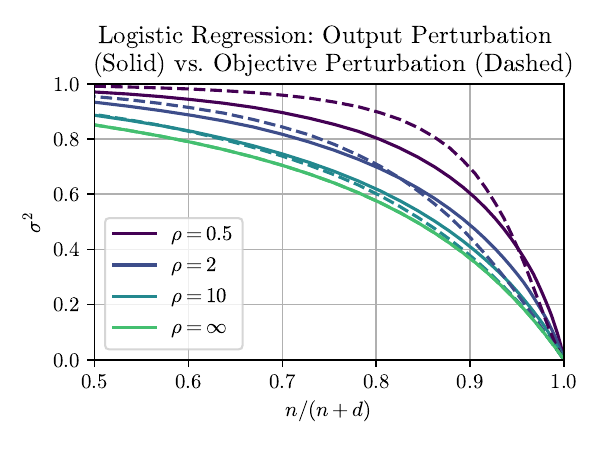}
    \caption{Left: Comparison of error estimates for output perturbation (Corollary~\ref{thm:main-output-huber-informal}) and objective perturbation (\cref{thm:main-objective-huber-informal}) on Huber regression with $L = 1$, $\bm{\eps}^\star \sim \mathcal{N}(\bm{0}, (1/10)^2\bm{I}_n)$. Right: Comparison of error estimates for output perturbation (Corollary~\ref{thm:main-output-logistic-informal}) and objective perturbation (\cref{thm:main-objective-logistic-informal}) for logistic regression. In both plots, $\kappa = 1$, $\bm{\beta}^\star \sim \mathcal{N}(\bm{0}, \kappa^2\bm{I}_d)$.}
    \label{fig:algorithm-comparison}
\end{figure}

Our results on objective perturbation can easily be modified to characterize the utility of \emph{output perturbation}, an even more foundational differentially private algorithm due to \cite{dwork2006our}, which samples $\bm{\xi} \sim \mathcal{N}(\bm{0}, \bm{I}_d)$ and outputs\[\widehat{\bm{\beta}} = {\left(\argmin_{\bm{\beta} \in \R^d}\, \sum_{i=1}^n \ell(\bm{\beta}; (\bm{x}_i, y_i)) + \frac{\lambda}{2}\norm{\bm{\beta}}^2\right)} + \nu\bm{\xi}.\] Here, the perturbation comes in the form of an additive noise term, rather than a random linear term in the objective. Extracting the $\nu^2$ term from the systems of equations of \cref{thm:main-objective-huber-informal} and \cref{thm:main-objective-logistic-informal} yields the following two corollaries, respectively.

Figure~\ref{fig:algorithm-comparison} uses these corollaries, along with the theorems of Sections~\ref{sec:main-body-objective-perturbation-linear} and \ref{sec:main-body-objective-perturbation-logistic}, to compare the privacy-utility tradeoffs of output perturbation and objective perturbation. Interestingly, for many privacy levels $\pr$, neither algorithm's error curve lies strictly below the other, suggesting that their relative performance \emph{depends} on the dimensionality ratio $\delta$. In other words, while prior work suggests that objective perturbation is generally the better algorithm of the two, Figure~\ref{fig:algorithm-comparison} tells a more complicated story, albeit one that is highly dependent on the chosen privacy bounds.

We give formal statements of these corollaries in \cref{sec:output-perturbation}, where we also verify their predictions against simulated data (Figure~\ref{fig:output-estimation}).

\begin{corollary}[Informal Version of Corollary~\ref{thm:main-huber-output-perturbation}]
\label{thm:main-output-huber-informal}
    Let $(\sigma^\star, \tau^\star)$ denote the solution to the system of equations from \cref{thm:main-objective-huber-informal} under the substitution $\nu = 0$. Then, under the same conditions as \cref{thm:main-objective-huber-informal}, the estimation error of the perturbed output $\bm{\hb}$ satisfies, w.h.p., \[\frac{1}{d}\norm{\widehat{\bm{\beta}} - \bm{\beta}^\star}^2 =(\sigma^\star)^2 + \nu^2 \pm n^{-\Omega(1)}.\]
\end{corollary}

\begin{corollary}[Informal Version of Corollary~\ref{thm:logistic-output-perturbation}]
\label{thm:main-output-logistic-informal}
    Let $(\alpha^\star, \sigma^\star, \gamma^\star)$ denote the solution to the system of equations from \cref{thm:main-objective-logistic-informal} under the substitution $\nu = 0$. Then, under the same conditions as \cref{thm:main-objective-logistic-informal}, the estimation error of the perturbed output $\widehat{\bm{\beta}}$ satisfies, w.h.p., \[\frac{1}{d}\norm{\widehat{\bm{\beta}} - \bm{\beta}^\star}^2 = (1 - \alpha^\star)^2\kappa^2 + (\sigma^\star)^2 + \nu^2 \pm e^{-(\log n)^{\Omega(1)}}.\]
\end{corollary}

\subsubsection{Utility of Noisy Stochastic Gradient Descent}
\label{sec:main-body-dp-sgd}

The \emph{noisy stochastic gradient descent} (a.k.a. DP-SGD) algorithm, which adds independent Gaussian noise to each iteration of SGD, is the cornerstone of modern, differentially private machine learning \cite{abadi2016deep}. In \cref{sec:dp-sgd}, we present results for DP-SGD that parallel our results for objective perturbation (Theorems \ref{thm:main-objective-huber-informal} and \ref{thm:main-objective-logistic-informal}) and for output perturbation (Corollaries \ref{thm:main-output-huber-informal} and \ref{thm:main-output-logistic-informal}). These results for DP-SGD follow directly from a black-box reduction to powerful, recent results in the non-private literature \cite{gerbelot2024meanfield, han2024entrywise}, and we validate them against simulated data. As such, we view this section more as a valuable point of reference than as a major technical contribution of our own. The results also come with some limitations. For example, certain technical requirements inherited from \cite{gerbelot2024meanfield} prevent us from studying the standard formulations of robust linear regression and logistic regression, let alone the broad range of convex and non-convex optimization problems to which DP-SGD is typically applied. Also, the system of equations now has size $O(T^2)$ for $T$ steps of DP-SGD, where $T$ must be held constant as $n \to \infty$ and $d/n \to \delta$. Overcoming these obstacles is a promising direction for future work.

\subsection{Technical Overview}
\label{sec:techniques}

In this section, we sketch the proofs of Theorems \ref{thm:main-objective-huber-informal} and \ref{thm:main-objective-logistic-informal}, our main results on objective perturbation. We carry out these proofs in Sections \ref{sec:objective-perturbation-linear} and \ref{sec:objective-perturbation-logistic}, respectively. As already discussed, our results for output perturbation are direct corollaries of these theorems, and our results for noisy stochastic gradient descent are similarly straightforward given prior results in the literature. Ultimately, for each of the three algorithms under consideration, we need two things:
\begin{itemize}
    \item A \emph{privacy proof}, which must hold for worst-case data points $(\bm{x}_i, y_i)$. These may be arbitrary points in $B_R(\bm{0}) \times \R$ in the case of robust linear regression, or arbitrary points in $B_R(\bm{0}) \times \{0, 1\}$ in the case of logistic regression, where $B_R(\bm{0})$ denotes the ball of radius $R$ centered at the origin in $\R^d$, i.e. \(B_R(\bm{0}) = \{\bm{x} \in \R^d : \norm{\bm{x}} \le R\} \subseteq \R^d.\)
    \item A \emph{utility proof}, which only needs to hold for data points $(\bm{x}_i, y_i)$ sampled independently from an appropriate distribution over $B_R(\bm{0}) \times \R$ in the case of robust linear regression, or over $B_R(\bm{0}) \times \{0, 1\}$ in the case of logistic regression.
\end{itemize}

Our privacy proofs build on ones that exist in the differential privacy literature, and we will hence inherit various boundedness requirements from those works \cite{bun2016zcdp, balle2018analytic, redberg2023improving}. Specifically, to apply our privacy analysis to robust linear regression and logistic regression, for which the gradient of the loss function scales in magnitude with $\norm{\bm{x}}$, we require the constraint that $\bm{x}_1, \ldots, \bm{x}_n \in B_R(\bm{0})$. This constraint can either be assumed as a priori knowledge, or enforced via a projection of unconstrained data in $\R^d$ onto $B_R(\bm{0})$. Bounding the norm of each $\bm{x}_i$ leads to a bound on the objective function's sensitivity, which ultimately yields our bound on the algorithm's privacy loss parameters $(\pe, \pd)$ or $\pr$.

Our utility proofs for objective perturbation are substantially more complex and comprise our main technical contribution. In these proofs, we initially focus on Gaussian features: \(\bm{x}_1, \ldots, \bm{x}_n \iid \mathcal{N}{\left(\bm{0}, \frac{1}{d}\bm{I}_d\right)}.\) The Gaussian distribution is not supported on $B_R(\bm{0})$, so we clearly must relax this assumption eventually. Nevertheless, Gaussianity is a fruitful starting point since it enables the use of a Gaussian comparison inequality called the \emph{Convex Gaussian Minimax Theorem} (CGMT), as discussed earlier in the introduction. We provide the formal statement of CGMT, along with a more detailed explanation of how it is used, in \cref{sec:preliminaries}.

Since our aim is to apply CGMT, the first step of our proof is to reformulate the objective perturbation algorithm as min-max optimization of the form \[\bm{\hb} = \argmin_{\bm{u} \in \cS_{\bm{u}}}\max_{\bm{v} \in \cS_{\bm{v}}}\; \underbrace{\langle \bm{X} \bm{u}, \bm{v}\rangle + \psi(\bm{u}, \bm{v})}_{Q_{\bm{u}, \bm{v}}}.\] Here,  $\cS_{\bm{u}} \subseteq \R^d$ and $\cS_{\bm{v}} \subseteq \R^n$ are constraint sets, $\bm{X} = [\bm{x}_1\,\cdots\,\bm{x}_n]^\top \in \R^{n \times d}$ denotes the feature matrix, which we also call the \emph{design matrix}, and $\psi : \cS_{\bm{u}} \times \cS_{\bm{v}} \to \R$ is a convex-concave function that may depend on the ground-truth $\bm{\beta}^\star$, the regression errors $\bm{\eps}^\star$, and the perturbation vector $\bm{\xi}$, but may \emph{not} depend on $\bm{X}$. We call $\psi$ the \emph{mean function} because \(\E[Q_{\bm{u}, \bm{v}}] = \psi(\bm{u}, \bm{v})\), where the expectation is taken only over the randomness of $\bm{X}$.

At this point, we have expressed $\widehat{\bm{\beta}}$ in terms of the extrema of $Q_{\bm{u}, \bm{v}}$ over $\bm{u} \in \cS_{\bm{u}}$ and $\bm{v} \in \cS_{\bm{v}}$. For the regression problems that we consider, the term $\psi(\bm{u}, \bm{v})$ will be very simple, so the main obstacle to optimizing $Q_{\bm{u}, \bm{v}}$ by hand is the presence of the random matrix in the bilinear term $\langle \bm{X} \bm{u}, \bm{v} \rangle$. Essentially, CGMT enables us to replace the Gaussian random matrix $\bm{X} \in \R^{n \times d}$ with two standard Gaussian random vectors $\bm{g} \in \R^d$ and $\bm{h} \in \R^n$, after which solving the min-max optimization by hand becomes a straightforward matter of calculus. Taking the derivatives of this \emph{auxiliary optimization} with respect to $\bm{u} \in \R^d$ and $\bm{v} \in \R^n$ and setting them equal to zero, while keeping in mind the constraint sets $\cS_{\bm{u}}$ and $\cS_{\bm{v}}$, yields a system of $n + d$ equations in $n + d$ variables characterizing the behavior of both this auxiliary optimizer and its corresponding dual variable. Ultimately, the guarantee of CGMT is that over the randomness of $\bm{X}$, the solution to this auxiliary problem is indistinguishable from the solution $\bm{\hb}$ to the original problem with respect to a wide range of \emph{pseudo-Lipschitz} test functions. To complete the proof in the case of Gaussian features, we show that this system of first-order optimality equations reduces to a fixed-size system of \emph{scalar} equations (two equations in two variables in the case of robust linear regression, or three equations in three variables in the case of logistic regression) in the limit as $n \to \infty$ with $d/n \to \delta$.

At this point, all that remains is to relax the Gaussianity assumption on $\bm{X}$, which we emphasize is crucial for the analysis of any differentially private algorithm that assumes its input is bounded in $B_R(\bm{0})$ or incorporates data clipping as a subroutine. To this end, we draw on recent, powerful theoretical results that formalize the notion of \emph{universality} from statistical physics \cite{han2023universality, han2024entrywise}. Informally speaking, universality is the idea that certain global properties of a system, such as the estimation error $\frac{1}{d}\norm{\bm{\hb} - \bm{\beta}^\star}^2$, should not depend too heavily on the local details of the system, such as the distribution of the entries of $\bm{X}$.

The first universality law we leverage is Corollary 2.6 of \cite{han2023universality}, which we call \emph{CGMT universality}. This result asserts that even when $\bm{X}$ is not Gaussian, the distribution of the extrema of $Q_{\bm{u}, \bm{v}}$ over $\bm{u} \in \cS_{\bm{u}}$ and $\bm{v} \in \cS_{\bm{v}}$ is essentially unchanged upon replacing $\bm{X}$ with a matrix $\bm{X}'$ with a few matching moments in the definition of $Q_{\bm{u}, \bm{v}}$. Taking $\bm{X}' = \bm{G}$ to be Gaussian allows us to apply CGMT even to non-Gaussian $\bm{X}$ drawn from any isotropic distribution over $B_R(\bm{0})$ with independent (not necessarily identically distributed!) entries. We believe our results can be extended to nonisotropic distributions, as well, which is a point we revisit in \cref{sec:objective-perturbation-linear}. Applying CGMT universality requires several technical conditions on the original optimization problem to be verified. For our purposes, this roughly amounts to showing that the output $\bm{\hb}$ of the objective perturbation algorithm satisfies $\norm{\bm{\hb}}_\infty \le n^{o(1)}$. We prove this in \cref{thm:objective-uniform-bound} using a leave-one-feature-out perturbation argument, a result that may be of independent interest.

Although CGMT universality as it appears in \cite{han2023universality} is exactly what we need to analyze the objective perturbation with robust linear loss, frustratingly, it cannot be directly applied to logistic regression. The reason is that, in the case of the logistic loss, even the initial step of reformulating the perturbed objective as $Q_{\bm{u}, \bm{v}} = \langle \bm{X}\bm{u}, \bm{v}\rangle + \psi(\bm{u}, \bm{v})$ requires the assumption of Gaussianity to ensure the independence of the first and second terms, a requirement for CGMT. Consequently, we need a universality law that addresses the perturbed logistic objective function in its original form, before transforming it into $Q_{\bm{u}, \bm{v}}$. For this, we use the more recent universality law of \cite{han2024entrywise}, called \emph{generalized first-order method} (GFOM) universality. To apply GFOM universality, we first construct a sequence of ``$\sigma$-smoothed'' iterates $\bm{\beta}^{(t)}_\sigma$ for $t \in \N$ and $\sigma > 0$ that start at $\bm{\beta}^{(0)} = \bm{0}$ but rapidly converge to $\bm{\hb}$ as $t \to \infty$ and $\sigma \to 0$. These iterates are \emph{not} part of the objective perturbation algorithm, but rather exist solely for the purpose of analysis, as they fit the description of a smooth ``generalized first-order method'' put forward in \cite{celentano2020gfom}. GFOM universality tells us that the behavior of these iterates (and hence their limit $\bm{\hb}$) is unchanged upon replacing $\bm{X}$ with a Gaussian matrix $\bm{G}$ with matching dimensions and a few matching moments. Our application of GFOM universality to the perturbed logistic objective mirrors an argument of \cite{han2024entrywise} in the non-private case. The key differences arise in verifying that various technical conditions still hold in the presence of the randomness introduced for privacy.

Ultimately, universality allows us to reconcile the Gaussian design assumptions of CGMT analysis with the necessarily different data distributions encountered by differentially private algorithms, including output perturbation, objective perturbation, and noisy stochastic gradient descent. In terms of our final error bounds, CGMT universality, which we use for robust linear regression, incurs only an additive $n^{-\Omega(1)}$ penalty. GFOM universality, which we use for logistic regression, incurs only an additive $e^{-(\log n)^{\Omega(1)}}$ penalty.

\subsection{Related Work}

Many prior works study algorithms for differentially private regression, but they do not to provide substantive guarantees in the proportional dimensionality regime as we do in this work \cite{dwork2009ptr, chaudhuri2011differentially, kifer2012private, mir2013landscape, bassily2014erm, wang2015privacy, sheffet2017ols, sheffet2019techniques, liu2022hdptr, varshney2022nearly, amin2023easy, brown2024ols, brown2024gradient}. Other works study the $d \gg n$ regime but require \emph{sparsity} assumptions \cite{steinke2017sparse, cai2021cost, zhou2022sparse, cai2023score, georgiev2024sparse, ma2024better, kent2024rateoptimalityphasetransition}. Since we do not assume sparsity, these works are not comparable to ours. Under sparsity assumptions, the analogous challenge would be to analyze the regime where $n$ is proportional to the $s \log(d/s)$ where $s$ is the sparsity level---the fact that this is the ``effective dimension'' can be argued via sample complexity, via $\eps$-nets, or via the statistical dimension of \cite{amelunxen2014living}.

Our proofs rely on tools from the non-private statistics literature that have not previously appeared in the privacy literature. The Convex Gaussian Minimax Theorem is based on the classical inequalities of \cite{slepian1962barrier, gordon1985processes}, but it was only recently presented in its current form and applied to modern statistical problems \cite{stojnic2013framework, thrampoulidis2015regularized, thrampoulidis2018mestimators, dhifallah2020preciseperformanceanalysislearning, deng2021double, miolane2021sparse, loureiro2021learning, wang2021slope, liang2022precise, zhang2022moderntheoryhighdimensionalcox, hu2022slope, montanari2023generalizationerrormaxmarginlinear, celentano2023lasso}. Prior to this, a leading technique for analysis in the proportional dimensionality regime was the \emph{approximate message passing} (AMP) framework of \cite{donoho2009message, bayati2011amp}. The universality laws we use are from \cite{han2023universality, han2024entrywise}, but there are many works on this topic \cite{korada2011lindeberg, karoui2013asymptotic, bayati2015polytope, panahi2017universal, omyak2017reduction, elkaroui2018geometry, abbasi2019linear, montanari2017elastic, dudeja2023amp, dudeja2024spectral}. Privacy analyses we draw on include the works of \cite{bun2016zcdp, balle2018analytic, redberg2021publishable, redberg2023improving}.

\section{Preliminaries}
\label{sec:preliminaries}

Throughout this paper, when working with an infinite sequence indexed by $n \in \N$, which may be a sequence of scalars, vectors, matrices, functions, or subsets of $\R^n$, we will often refer to the sequence as though it were a single object of that type. For example, when referring to the problem dimension $d$, it should be understood that $d$ is not merely a scalar, but rather a scalar sequence $d(n)$ (equivalently, $d_n$) indexed by $n$.

Consequently, by the \emph{proportional dimensionality regime}, we mean the setting in which
\[
    \lim_{n \to \infty}\frac{d(n)}{n} = \delta,
\]
where we denote the given dimensionality ratio by $\delta \in (0, \infty)$. In the proportional regime, any scalar sequence $f(n)$ satisfying $f(n) = \Theta(n^c)$ for a fixed constant $c > 0$ also satisfies $f(n) = \Theta(d^c)$, and vice versa. For consistency, when presented with such a notational choice, we shall always choose to express asymptotic bounds in terms of $n$, not $d$.

\subsection{Differential Privacy}

The algorithms we study in this paper all satisfy differential privacy (DP), the accepted definition of rigorous privacy protection for data analysis. We focus on the notion of \emph{approximate}, or $(\eps, \delta)$-differential privacy with respect to the \emph{add/remove-one} notion of adjacency:

\begin{definition}[\cite{dwork2006dp}]
\label{def:dp}
    Given a universe $\mathcal{X}$, we say data sets $\bm{x}, \bm{x}' \in \N^{\mathcal{X}}$ are \emph{adjacent} if $\norm{\bm{x} - \bm{x}'}_1 = 1$. We say a randomized algorithm $\mathcal{A} : \N^{\mathcal{X}} \to \mathcal{Y}$ satisfies \emph{$(\pe, \pd)$-differential privacy} if for all adjacent $\bm{x}, \bm{x}' \in \mathcal{X}^n$ and subsets $Y \subseteq \mathcal{Y}$,
    \[
        \Pr[\mathcal{A}(\bm{x}) \in Y] \le e^{\pe}\Pr[\mathcal{A}(\bm{x}') \in Y] + \pd.
    \]
\end{definition}

We shall also require two other, related notions of differential privacy, known as \emph{R\'{e}nyi} and \emph{zero-concentrated} DP, that have been used in major real-world deployments of DP. They are both defined in terms of the \emph{R\'{e}nyi divergence} of order $\pa > 1$ between two distributions $P$ and $Q$ over $\mathcal{Y}$:
\[
    D_\pa(P\|Q) = \frac{1}{\pa - 1}\log \E_{y \sim P} \log \left(\frac{P(y)}{Q(y)}\right)^{\pa - 1}.
\]
\begin{definition}[RDP, \cite{mironov2017renyi}]
\label{def:zcdp}
    Given $\pa > 1$ and $\pe \ge 0$, we say a randomized algorithm $\mathcal{A} : \N^{\mathcal{X}} \to \mathcal{Y}$ satisfies \emph{$(\pa, \pe)$-R\'{e}nyi DP} (RDP) if for all adjacent data sets $\bm{x}, \bm{x}' \in \mathcal{X}^n$,
    \[
        D_\pa(\mathcal{A}(\bm{x}) \| \mathcal{A}(\bm{x}')) \le \pe.
    \]
\end{definition}

Importantly, it is known that $(\pa, \pe)$-RDP implies $\left(\pe + \frac{\log(1/\pd)}{\pa-1}, \pd\right)$-DP for any $0 < \pd < 1$ \cite{mironov2017renyi}.

\begin{definition}[zCDP, \cite{bun2016zcdp}]
    We say a randomized algorithm $\mathcal{A} : \N^{\mathcal{X}} \to \mathcal{Y}$ satisfies \emph{$\pr$-zero-concentrated DP} (zCDP) if $\mathcal{A}$ satisfies $(\pa, \pr \cdot \pa)$-RDP for all $\pa > 1$.
\end{definition}

We use $\pe, \pd, \pa, \pr$ to refer to an algorithm's privacy parameters, rather than $\eps, \delta, \alpha, \rho$, in order to avoid confusion with other variable names. These include the regression errors $\eps$, the dimensionality ratio $\delta$, the bias variable $\alpha$, and the function $\rho$ used to define the logistic loss. For more background on differential privacy, including the motivation behind the definition, the protections it provides, interpretations of the privacy parameters, privacy composition theorems, and more, we refer the reader to \cite{dwork2014algorithmic, vadhan2017complexity}.

We collect here a handful of recent privacy analyses from the DP literature that we will repeatedly use in our paper. The first two results concern the \emph{Gaussian mechanism} $\mathcal{A}(\bm{x}) = f(\bm{x}) + \bm{z}$ where $\bm{z} \sim \mathcal{N}(\bm{0}, \nu^2 \bm{I}_d)$ \cite{dwork2006our}. To state the result, which we will eventually use to analyze the output perturbation algorithm, recall that a function $f : \mathcal{X}^n \to \R^d$ has $\ell^2$ sensitivity at most $\Delta$ if $\norm{f(\bm{x}) - f(\bm{x}')} \le \Delta$ for all adjacent $\bm{x}, \bm{x}' \in \mathcal{X}^n$.

\begin{theorem}[Lemma 2.5 of \cite{bun2016zcdp}]
\label{thm:gaussian-mechanism-zcdp}
    If $f : \N^{\mathcal{X}} \to \R^d$ has $\ell^2$ sensitivity at most $\Delta$, then the Gaussian mechanism w.r.t. $f$ satisfies $\pr$-zCDP for $\pr = \frac{\Delta^2}{2\nu^2}$.
\end{theorem}

\begin{theorem}[Theorem 8 of \cite{balle2018analytic}]
\label{thm:hockey-stick}
    If $f : \N^{\mathcal{X}} \to \R^d$ has $\ell^2$ sensitivity at most $\Delta$, then for any $\pe \ge 0$ and $\pd \in [0, 1]$, the Gaussian mechanism w.r.t. $f$ satisfies $(\pe, \pd)$-DP if and only if
    \[
        \delta \ge \mathrm{HockeyStick}{\left(\pe, \frac{\Delta}{\nu}\right)} = \Phi{\left(\frac{\Delta}{2\nu} - \frac{\pe \nu}{\Delta}\right)} - e^\pe  \Phi{\left(-\frac{\Delta}{2\nu} - \frac{\pe \nu}{\Delta}\right)}.
    \]
\end{theorem}

The name ``HockeyStick'' comes from the term \emph{hockey stick divergence}, defined for $\pe \ge 0$ and distributions $P$ and $Q$ on $\R$ as $H_{e^\pe}(P\|Q) = \int_{-\infty}^\infty \mathrm{max}(0, P(x) - e^{\pe}Q(x))\,dx$. As we have defined it, $\mathrm{HockeyStick}(\pe, \Delta/\nu)$ measures this divergence between $P = \mathcal{N}(\Delta, \nu^2)$ and $Q = \mathcal{N}(0, \nu^2)$, which only depends on $\Delta$ and $\nu$ through their ratio $\Delta / \nu$.

The next result we state is a recent privacy analysis of objective perturbation from \cite{redberg2023improving}. Although we will not use this result directly in our paper, this result and its proof are the starting point from which our own privacy analysis in \cref{sec:objective-perturbation-privacy} proceeds. Indeed, the statement and proof of our privacy result will be very similar to that of \cite{redberg2023improving}. The key difference is that ours will apply to any regularization strength $\lambda > 0$ and perturbation strength $\nu > 0$, but that of \cite{redberg2023improving} requires $\lambda > s$ for some strictly positive $s$ depending on the smoothness of the loss function.

The result of \cite{redberg2023improving}, which we now state, concerns loss functions of the form $\ell(\bm{\beta}; (\bm{x}, y)) = \ell_0(\langle \bm{x}, \bm{\beta}\rangle, y)$ for a function $\ell_0 : \R^2 \to \R$. Such loss functions are called \emph{generalized linear model} (GLM) loss functions, and as the authors of \cite{redberg2023improving} emphasize, most privacy proofs for objective perturbation in the DP literature require this GLM assumption, even if they do not state it explicitly.\footnote{The statement of \cref{thm:redberg-privacy} is \emph{not} identical to the statements of Theorems 3.1 and 3.2 in \cite{redberg2023improving} because we have corrected a few typos in their statement: first, in the $(\eps, \delta)$-DP bound, we have added absolute value signs around the term $\log(1 - s/\lambda)$ in the definition of $\tilde{\eps}$. Similarly, we have added a factor of $2$ in the denominator of $L^2/2\nu^2$. The absolute value signs and the factor of $2$ are both present throughout the proof of Theorem 3.1 of \cite{redberg2023improving}, and are only omitted in the theorem's statement.}
\begin{theorem}[Theorems 3.1 and 3.2 of \cite{redberg2023improving}]
\label{thm:redberg-privacy}
    Suppose that $\ell_0 : \R^2 \to \R$ satisfies $\abs{\partial_1 \ell_0(\eta, y)} \le L$ and $0 \le \partial_1^2 \ell_0(\eta, y) \le s$ for some constants $L, s > 0$ and for all $\eta, y \in \R$. Then objective perturbation (\cref{alg:objective-perturbation}) with $R = 1$, any $\lambda, \nu > 0$, and the GLM loss function $\ell(\bm{\beta}; (\bm{x}, y)) = \ell_0(\langle \bm{x}, \bm{\beta}\rangle, y)$ satisfies $(\pe, \pd)$-differential privacy for any $\pe \ge 0$ and
    \[
        \pd = \begin{cases}
            2 \cdot \mathrm{HockeyStick}(\tpe, \frac{L}{\nu}) &\text{if }\hpe \ge 0,\\
            (1 - e^{\hpe}) + 2e^{\hpe}\cdot \mathrm{HockeyStick}\left(\frac{L^2}{2\nu^2}, \frac{L}{\nu}\right) &\text{otherwise,}
        \end{cases}
    \]
    where we set $\tpe = \pe - \abs{\log(1 - s/\lambda)}$ and $\hpe = \tpe - L^2/2\nu^2$, provided that $\lambda > s$.

    The algorithm also satisfies $(\pa, \pe)$-R\'{e}nyi differential privacy for any $\pa > 1$ and
    \[
        \pe = -\log\left(1 - \frac{s}{\lambda}\right) + \frac{L^2}{2\nu^2} + \frac{1}{\pa - 1}\log \mathbb{E}_{X \sim \mathcal{N}{\left(0, \frac{L^2}{\nu^2}\right)}}\bigl[e^{(\pa - 1)\abs{X}}\bigr].
    \]

\end{theorem}

\subsection{Convex Gaussian Minimax Theorem}

In this section, we present the Convex Gaussian Minimax Theorem (CGMT), due to \cite{stojnic2013framework,thrampoulidis2015regularized}, which will play a key role in our analysis of objective perturbation.

As discussed in the technical overview in \cref{sec:techniques}, CGMT allows us to better understand the optimizers of a certain class of min-max optimization problems that are defined in terms of a Gaussian random matrix $\bm{G} \in \R^{n \times d}$. The theorem relates the original problem to an \emph{auxiliary} optimization problem in which the random matrix $\bm{G}$ has been replaced with two random vectors $\bm{g} \in \R^d$ and $\bm{h} \in \R^n$, whose dimensions match the number of columns and rows of $\bm{G}$, respectively. Remarkably, the theorem shows that the optimum value of the original problem, which is a scalar-valued function of $\bm{G}$, has a similar distribution to the optimum value of the auxiliary problem, which is a scalar-valued function of $(\bm{g}, \bm{h})$ that is much easier to determine directly via calculus. The main assumption needed for CGMT is that the original min-max problem is convex-concave and depends bilinearly on $\bm{G}$.

At first glance, it may seem that knowing the optimum value of the original problem is not very useful. After all, in the proof sketch we outlined in \cref{sec:techniques}, what we really wanted to understand was the \emph{optimizer} $\bm{\hb}$, which is the location at which the optimum value is achieved. It turns out that by carefully choosing the constraint sets of the problem to which CGMT is applied, we can get our hands on a remarkable number of quantitative properties about $\bm{\hb}$. For example, suppose we would like to show that the estimation error of $\bm{\hb}$ with respect to the ground-truth $\bm{\beta}^\star$ is approximately $(\sigma^\star)^2$ in the limit. To do so, consider a small slack factor $\eps_n > 0$ and define the constraint set
\[
    \cS_{\eps_n} = \left\{\bm{\beta} \in \R^d : \Bigl\lvert\frac{1}{d}\norm{\bm{\beta} - \bm{\beta}^\star}^2 - (\sigma^\star)^2\Bigr\rvert \ge \eps_n\right\}.
\]
If we can show that, with high probability, the optimum value of the auxiliary problem over $\cS_{\eps_n}$ is strictly larger than the optimum value of the auxiliary problem over $\R^d$, then by CGMT, the same must be true of original problem. Since $\bm{\hb}$ is the optimizer of the original problem, this would allow us to conclude that with high probability, $\bm{\hb} \notin \cS_{\eps_n}$, so $\frac{1}{d}\norm{\bm{\hb} - \bm{\beta}^\star}^2 = (\sigma^\star)^2 \pm \eps_n$.

Nothing about the previous example is particularly specific to $\ell^2$ estimation error, which suggests (correctly) that the argument can be carried out for a wide range of test functions. It turns out that we will be able to handle a broad class of \emph{pseudo-Lipschitz} test functions. We will be able to handle slack factors on the order of $\eps_n = n^{-\Omega(1)}$ in the case of robust linear regression, or $\eps_n = e^{-(\log n)^{\Omega(1)}}$ in the case of logistic regression, both of which satisfy $\eps_n \to 0$ as $n \to \infty$.

\begin{theorem}[Convex Gaussian Minimax Theorem, \cite{stojnic2013framework, thrampoulidis2015regularized}]
\label{thm:cgmt}
    Given compact sets $\mathcal{S}_{\bm{u}} \subseteq \R^d$ and $\mathcal{S}_{\bm{v}} \subseteq \R^n$, a continuous function $\psi : \cS_{\bm{u}} \times \cS_{\bm{v}} \to \R$, a matrix $\bm{G} \in \R^{n \times d}$, and vectors $\bm{g} \in \R^d$, $\bm{h} \in \R^n$, consider the following two random variables indexed by $\bm{u} \in \cS_{\bm{u}}$ and $\bm{v} \in \cS_{\bm{v}}$:
    \begin{align*}
        Q_{\bm{u},  \bm{v}} &= \langle \bm{G} \bm{u}, \bm{v} \rangle + \psi(\bm{u}, \bm{v}) \\
        Q'_{\bm{u}, \bm{v}} &= \norm{\bm{u}}\langle \bm{h}, \bm{v} \rangle - \norm{\bm{v}}\langle \bm{g}, \bm{u} \rangle +\psi(\bm{u}, \bm{v}).
    \end{align*}
    We call $Q_{\bm{u}, \bm{v}}$ and $Q'_{\bm{u}, \bm{v}}$ the \emph{primary and auxiliary objective functions}. If $\bm{G}$, $\bm{g}$, and $\bm{h}$ all have independent $\mathcal{N}(0, \sigma^2)$ entries for some $\sigma > 0$, then:
    \begin{enumerate}[(a)]
        \item For any threshold $t \in \R$, \[\Pr\left[\min_{\bm{u} \in \cS_{\bm{u}}} \max_{\bm{v} \in \cS_{\bm{v}}}\, Q_{\bm{u}, \bm{v}} < t\right] \le 2\Pr\left[\min_{\bm{u} \in \cS_{\bm{u}}} \max_{\bm{v} \in \cS_{\bm{v}}}\, Q'_{\bm{u}, \bm{v}} < t\right].\]
        \item If $\mathcal{S}_{\bm{u}}$ and $\mathcal{S}_{\bm{v}}$ are convex and $\psi$ is convex-concave on $\mathcal{S}_{\bm{u}} \times \mathcal{S}_{\bm{v}}$, then for any threshold $t \in \R$, \[\Pr\left[\min_{\bm{u} \in \cS_{\bm{u}}} \max_{\bm{v} \in \cS_{\bm{v}}}\, Q_{\bm{u}, \bm{v}} > t\right] \le 2\Pr\left[\min_{\bm{u} \in \cS_{\bm{u}}} \max_{\bm{v} \in \cS_{\bm{v}}}\, Q'_{\bm{u}, \bm{v}} > t\right].\]
    \end{enumerate}
\end{theorem}

\subsection{Universality}

Here, we state the two key universality laws from prior work that we will use in our analyses. The first result is Corollary 2.6 of \cite{han2023universality}, which we call CGMT universality. We have slightly rephrased the statement for clarity, at the cost of generality (e.g. by assuming that the mean function $\psi$ is differentiable). We will eventually use this result in our analyses of objective perturbation and output perturbation for robust linear regression. As we discussed earlier in the technical overview of \cref{sec:techniques}, this result allows us to extend CGMT analysis to non-Gaussian design matrices, as long as we have a few finite, matching moments.

\begin{theorem}[CGMT Universality: Corollary 2.6 of \cite{han2023universality}]
\label{thm:cgmt-universality}
    Consider random matrices $\bm{A}, \bm{B} \in \R^{n \times d}$ with independent entries satisfying the following three moment conditions as $n \to \infty$ with $d = \Theta(n)$:
    \begin{enumerate}[(1)]
        \item $\E\bm{A}_{ij} = \E\bm{B}_{ij} = 0$ for all $i \in [n]$ and $j \in [d]$.
        \item $\E\bm{A}_{ij}^2 = \E\bm{B}_{ij}^2 = 1$ for all $i \in [n]$ and $j \in [d]$.
        \item $\max_{i \in [n]} \max_{j \in [d]}\; \E\abs{\bm{A}_{ij}}^3 + \E\abs{\bm{B}_{ij}}^3 = O(1)$.
    \end{enumerate}
    Let $\mathcal{S}_{\bm{u}} \subseteq [-L_{\bm{u}}, L_{\bm{u}}]^d$ and $\mathcal{S}_{\bm{v}} \subseteq [-L_{\bm{v}}, L_{\bm{v}}]^n$ be measurable sets with $L_{\bm{u}}, L_{\bm{v}} \ge 1$, and set $L = L_{\bm{u}} + L_{\bm{v}}$. Let $\psi : \R^d \times \R^n \to \R$ be a differentiable function with \[\mathscr{M} = \max_{\bm{u} \in [-L, +L]^d} \max_{\bm{v} \in [-L, +L]^n}\,\norm{\nabla \psi(\bm{u}, \bm{v})}_1.\]
    Finally, consider the following random variables indexed by $\bm{u} \in \cS_{\bm{u}}$ and $\bm{v} \in \cS_{\bm{v}}$:\[Q^{\bm{A}}_{\bm{u}, \bm{v}} = \frac{1}{n^{3/2}}\langle \bm{A}\bm{u}, \bm{v}\rangle + \psi(\bm{u}, \bm{v}) \quad \text{and} \quad Q^{\bm{B}}_{\bm{u}, \bm{v}} = \frac{1}{n^{3/2}}\langle \bm{B}\bm{u}, \bm{v}\rangle + \psi(\bm{u}, \bm{v}).\]
    Then, for all thresholds $t \in \R$, all gaps $g > 0$, and all $\omega \ge n$,
    \[\Pr\left[\min_{\bm{u} \in \cS_{\bm{u}}} \max_{\bm{v} \in \cS_{\bm{v}}}\, Q^{\bm{A}}_{\bm{u}, \bm{v}} < t \right] - \Pr\left[\min_{\bm{u} \in \cS_{\bm{u}}} \max_{\bm{v} \in \cS_{\bm{v}}}\, Q^{\bm{B}}_{\bm{u}, \bm{v}}  < t + g\right] = O{\left({\left(1 + \frac{1}{g^3}\right)} {\left(\frac{\mathscr{M}}{\omega} + \frac{L^2\log^{2/3}(L\omega)}{n^{1/6}}\right)}\right)}.\]
\end{theorem}

The second result we build on is Theorem 3.2 of \cite{han2024entrywise}, namely generalized first-order method (GFOM) universality. This result establishes universality for a broad class of iterative methods, called \emph{first-order methods}, involving a random design matrix $\bm{X}$. The result ensures that many \emph{pseudo-Lipschitz} summaries of such iterates (see \cref{sec:mathematical-miscellany} for a definition of pseudo-Lipschitz) are essentially unchanged upon replacing $\bm{X}$ with a Gaussian matrix $\bm{G}$ of the same shape. As discussed in the technical overview of \cref{sec:techniques}, this will prove useful for our analysis of objective perturbation and output perturbation for logistic regression, as well our DP-SGD results for both robust linear regression and logistic regression.

We have modified the statement slightly from the version that appeared in \cite{han2024entrywise} to allow vector-valued iterates and different test functions per coordinate, changes that require only minimal, syntactic changes in the proof. Before stating the GFOM universality theorem, we give a formal definition of general first order methods:

\begin{definition}
    Given $\bm{u}^{(1)}, \ldots, \bm{u}^{(t)} \in \R^{m \times k}$, let $\bm{u}^{(s)}_i \in \R^k$ be the $i$\textsuperscript{th} row of $\bm{u}^{(s)}$. Let  $\bm{u}^{(1:t)} \in \R^{t \times m \times k}$ be the tensor comprising $\bm{u}^{(s)}$ for $s \in [t]$. Define $\bm{u}^{(1:t)}_i \in \R^{t \times k}$ similarly. Given functions $F_1, \ldots, F_m : \R^{t \times k} \to \R^k$, let $F(\bm{u}^{(1:t)}) \in \R^{m \times k}$ be the matrix with $i$\textsuperscript{th} row $F_i(\bm{u}^{(1:t)}_i) \in \R^k$.
\end{definition}

\begin{definition}
    A \emph{general first order method} (GFOM) consists of a matrix $\bm{X} \in \R^{n \times d}$, an \emph{initialization} $(\bm{u}^{(0)}, \bm{v}^{(0)}) \in \R^{d \times k} \times \R^{n \times \ell}$ with $k, \ell = O(1)$ and $d = \Theta(n)$, and functions $F_{1j}^{(t)},F_{2j}^{(t)} : \R^{t \times k} \to \R^k$, $G_{1i}^{(t)} : \R^{t \times \ell} \to \R^\ell$, and $G_{2i}^{(t)} : \R^{(t+1) \times \ell} \to \R^\ell$ for all $t \in \N$, $i \in [n]$, and $j \in [d]$. Its \emph{iterates} are
    \begin{equation}
    \begin{split}
        \bm{u}^{(t)} &= \bm{X} F_1^{(t)}(\bm{v}^{(0 : t - 1)}) + G_1^{(t)}(\bm{u}^{(0 : t-1)}) \in \R^{d \times k}, \\
        \bm{v}^{(t)} &= \bm{X}^\top G_2^{(t)}(\bm{u}^{(0:t)}) + F_2^{(t)}(\bm{v}^{(0:t-1)}) \in \R^{n \times \ell}.
    \end{split}
    \end{equation}
\end{definition}

\begin{theorem}[GFOM Universality, Theorem 3.2 of \cite{han2024entrywise}]
\label{thm:gfom-universality}
    Consider random matrices $\bm{A}, \bm{B} \in \R^{n \times d}$ with independent entries satisfying the following conditions as $n \to \infty$ with $d = \Theta(n)$:
    \begin{itemize}
        \item $\E \bm{A}_{ij} = \E\bm{B}_{ij} = 0$ for all $i \in [n]$ and $j \in [d]$.
        \item $\E\bm{A}_{ij}^2 = \E\bm{B}_{ij}^2$ for all $i \in [n]$ amd $j \in [d]$.
        \item $\max_{i \in [n]} \max_{j \in [d]}\, \norm{\bm{A}_{ij}}_{\psi_2} + \norm{\bm{B}_{ij}}_{\psi_2} = O(n^{-1/2})$.
    \end{itemize}
    Also consider a GFOM satisfying the following condition for some $\Lambda \ge 2$, stated in terms of the \emph{Lipschitz norm} $\norm{f}_{\mathrm{Lip}} = \inf\{L \in \R : \text{$f$ is $L$-Lipschitz}\}$ of a function $f$:
    \[
        \max_{s \in [t]} \max_{q \in \{1, 2\}} \max_{i \in [n]} \max_{j \in [d]}\; \norm{F^{(t)}_{qi}}_{\mathrm{Lip}} + \norm{G^{(t)}_{qj}}_{\mathrm{Lip}} + \abs{F^{(t)}_{qi}(0)} + \abs{G^{(t)}_{qj}(0)} \le \Lambda.
    \]
    Let $\bm{u}^{(t)}(\bm{X})$ and $\bm{v}^{(t)}(\bm{X})$ denote the output of this GFOM with matrix $\bm{X}$. Then, for any collection of $\Lambda$-pseudo-Lipschitz functions $\psi_{1j} : \R^{t \times k} \to \R$ and $\psi_{2i} : \R^{t \times \ell} \to \R$ of order $p$, for any $q \in \N$, there exists $C > 0$ such that
    \begin{align*}
        &\E\bigg\lvert{\frac{1}{d}\sum_{j=1}^d \Big(\psi_{1j}(\bm{u}^{(1:t)}_j(\bm{A})) - \psi_{1j}(\bm{u}^{(1:t)}_j(\bm{B}))\Big)}\bigg\rvert^q + \E\bigg\lvert{\frac{1}{n}\sum_{i=1}^n \Big( \psi_{2i}(\bm{v}^{(1:t)}_i(\bm{A})) - \psi_{2i}(\bm{v}^{(1:t)}_i(\bm{B}))\Big)}\bigg\rvert^q\\
        &\le {(\Lambda \log(n)(1 + \norm{\bm{u}^{(0)}}_\infty + \norm{\bm{v}^{(0)}}_\infty))^{Ct^3}}{n^{-1/(Ct^3)}}.
    \end{align*}
\end{theorem}
Note that in \cref{thm:gfom-universality}, if the initialization $(\bm{u}^{(0)}, \bm{v}^{(0)})$ is random, then the conclusion of the theorem still holds even after conditioning on the value of $(\bm{u}^{(0)}, \bm{v}^{(0)})$.

\subsection{Mathematical Miscellany}
\label{sec:mathematical-miscellany}

\paragraph{Basic Definitions.}

We write vectors, matrices, and higher-order tensors in boldface. Given $m \in \N$, we write $[m] = \{1, 2, \ldots, m\}$. Given a vector $\bm{x} \in \R^n$ and $p \ge 1$, we write $\norm{\bm{x}}_p = (\sum_{i=1}^n \abs{x_i}^p)^{1/p}$ and $\norm{\bm{x}}_\infty = \max_{i \in [n]} \abs{x_i}$ (by H\"{o}lder's inequality, if $p < q$, then $\norm{\bm{x}}_q \le \norm{\bm{x}}_p \le n^{\frac{1}{p}-\frac{1}{q}}\norm{\bm{x}}_q$). Given a matrix $\bm{A} \in \R^{n \times m}$, we define $\norm{\bm{A}}_{r\to s} = \max_{\norm{\bm{x}}_r \le 1} \norm{\bm{A}\bm{x}}_s$ and $\norm{\bm{A}}_p = \norm{\bm{A}}_{p\to p}$. By default, we let $\norm{\bm{x}} = \norm{\bm{x}}_2$ and $\norm{\bm{A}} = \norm{\bm{A}}_2$. Given $\bm{x}, \bm{y} \in \R^n$, we write $\bm{x} \odot \bm{y} = (x_1y_1, \ldots, x_ny_n) \in \R^n$ for their entrywise (a.k.a. Hadamard) product. Given a set $\cS \subseteq \R^n$, a scalar $c \in \R$ and a vector $\bm{x} \in \R^n$, we write
\(
    \bm{x} + c\cdot \cS = \{\bm{x} + c\bm{y} : \bm{y} \in \cS\}
\)
for the rescaled and translated version of $\cS$.

\begin{definition}
    $f : \R^m \to \R$ is \emph{$L$-pseudo-Lipschitz of order $k$} if for all $\bm{a}, \bm{b} \in \R^m$,
    \[
        \abs{f(\bm{a}) - f(\bm{b})} \le L(1 + \norm{\bm{a}} + \norm{\bm{b}})^{k-1}\norm{\bm{a} - \bm{b}}.
    \]
    We say $f$ is \emph{pseudo-Lipschitz} if there exist $L \ge 0$ and $k \in \N$ for which $f$ is $L$-pseudo-Lipschitz of order $k$. We say $f$ is \emph{$L$-Lipschitz} if it is $L$-pseudo-Lipschitz of order $1$. Let
    \[
        \norm{f}_{\mathrm{Lip}} = \inf\{L \in \R : \text{$f$ is $L$-Lipschitz}\}.
    \]
\end{definition}

The canonical example of an order-$k$ pseudo-Lipschitz function is $f(\bm{x}) = \norm{\bm{x}}^k$. Observe that if $f$ is $L$-pseudo-Lipschitz of order $k$, there exists $L' > 0$ such that $\abs{f(\bm{x})} \le L'(1 + \norm{\bm{x}}^k)$. We will often use the fact that $(1 + \norm{\bm{a}} + \norm{\bm{b}})^{k-1} \le (3\max\{1, \norm{\bm{a}}, \norm{\bm{b}}\})^{k-1} \le O_k(1 + \norm{\bm{a}}^{k-1} + \norm{\bm{b}}^{k-1})$.

\paragraph{High-Dimensional Probability.} We collect some notation, definitions, and lemmas from probability that we will repeatedly use. For more detail, we refer the reader to \cite{vershynin2018probability}.

First, we say a sequence of events $E_n$ holds \emph{with high probability (w.h.p.)} if \(\Pr[E_n] \ge 1 - n^{-\Omega(1)}\), or merely \emph{asymptotically almost surely (a.a.s.)} if \(\Pr[E_n] \ge 1 - o(1).\) We say a sequence of random variables $x_1, x_2, \ldots \in \R$ \emph{converges in probability to the random variable $x_0 \in \R$}, written $x_n \toP x_0$, if for all $c > 0$, we have that a.a.s., $\abs{x_n - x_0} \le c$. Given $\bm{p} \in [0, 1]^n$, we will denote by $\mathrm{Bernoulli}(\bm{p})$ the distribution over $\{0, 1\}^n$ with independent, $\mathrm{Bernoulli}(p_i)$ coordinates. We will denote the standard Gaussian PDF and CDF by $\varphi$ and $\Phi$, respectively. Given a vector $\bm{\mu} \in \R^n$ and positive semidefinite matrix $\bm{\Sigma} \in \R^{n \times n}$, we will denote the multivariate Gaussian distribution with mean $\bm{\mu}$ and covariance $\bm{\Sigma}$ by $\mathcal{N}(\bm{\mu}, \bm{\Sigma})$. The key properties of the standard Gaussian distribution that we will use are its rotational invariance and Gaussian integration by parts, also known as Stein's lemma:

\begin{lemma}[Rotational Invariance]
    If $\bm{z} \sim \mathcal{N}(\bm{0}, \bm{I}_n)$ and $\bm{Q} \in \R^{n \times n}$ is orthogonal ($\bm{Q}^{-1} = \bm{Q}^\top$), then $\bm{Q}\bm{z} \sim \mathcal{N}(\bm{0}, \bm{I}_n)$, as well.
\end{lemma}

\begin{lemma}[Stein's Lemma]
    If $Z \sim \mathcal{N}(0, 1)$ and $g : \R \to \R$ is differentiable, then $\E[g(Z)Z] = E[g'(Z)]$ if both expectations exist.
\end{lemma}

We also require the following notion of subgaussianity.

\begin{definition}[Subgaussian Norm]
\label{def:subgaussian-norm}
    The \emph{subgaussian norm} of a random variable $X \in \R$ is \[\norm{X}_{\psi_2} = \inf\left\{ c > 0  \;\Big|\; \E [e^{(X/c)^2}] \le 2\right\}.\]
    We say $X$ is \emph{subgaussian} if $\norm{X}_{\psi_2} < \infty$.
\end{definition}

One can check that if $X \sim \mathcal{N}(0, 1)$, then $\norm{X}_{\psi_2} = \Theta(1)$. Moreover, we have have $\norm{X}_{\psi_2} = O(1)$ if and only if $\Pr[\abs{X} > t] < 2\exp(-\Omega(t^2))$. Importantly, the subgaussian norm has the following relationship to projection onto a unit vector:

\begin{lemma}
\label{thm:subgaussian-inner-product}
    If the random vector $\bm{x} \in \R^d$ has independent, subgaussian components satisfying $\norm{x_1}_{\psi_2}, \ldots, \norm{x_d}_{\psi_2} \le \sigma$ and $\bm{u} \in \R^n$ is independent of $\bm{x}$ with $\norm{\bm{u}} = 1$, then \[\norm{\langle \bm{x}, \bm{u} \rangle}_{\psi_2} = O(\sigma).\]
\end{lemma}

\paragraph{Convex Analysis.}

We now briefly review some concepts from convex analysis that we will frequently use. For the most basic definitions of convexity, as well as more detail on the concepts discussed here, we refer the reader to \cite{DBLP:journals/ftml/Bubeck15, parikh2014proximal}.

\begin{definition}[Subdifferential]
\label{def:subdifferential}
    The \emph{subdifferential} of a function $f : \R^n \to \R$ at a point $\bm{x} \in \R^n$ is its set of \emph{subgradients} $\bm{g}$:
    \[
        \partial f(\bm{x}) = \{\bm{g} \in \R^n : \forall \bm{y} \in \R^n, f(\bm{y}) \ge f(\bm{x}) + \langle \bm{g}, \bm{y} - \bm{x} \rangle\}.
    \]
\end{definition}

Minimizing a (strongly convex) function amounts to finding a point $\bm{x}$ with a (nearly) vanishing subgradient $\bm{g}$:

\begin{lemma}[Convex Minimization]
\label{thm:convex-minimization}
    For any $f : \R^n \to \R$, the point $\bm{x} \in \R^n$ is a global minimizer of $f$ if and only if $\bm{0} \in \partial f(\bm{x})$. Relatedly, if $f$ is $m$-strongly convex and $\bm{g} \in \partial f(\bm{x})$, then
    \[
        f(\bm{x}) - \inf_{\bm{y} \in \R^n} f(\bm{y}) \le \frac{1}{2m}\norm{\bm{g}}^2.
    \]
\end{lemma}

We shall repeatedly use the minimax theorem, a cornerstone result of game theory.
\begin{lemma}[Sion's Minimax Theorem, \cite{sion1958minimax}]
    Suppose $\cS_{\bm{x}} \subseteq \R^d$ and $\cS_{\bm{y}} \subseteq \R^n$ are convex and the function $f : \R^d \times \R^n \to \R$ satisfies the following two properties:
    \begin{itemize}
        \item The map $\bm{x} \mapsto f(\bm{x}, \bm{y})$ is a closed, proper, convex function for all $\bm{y} \in \cS_{\bm{y}}$.
        \item The map $\bm{y} \mapsto -f(\bm{x}, \bm{y})$ is a closed, proper, convex function for all $\bm{x} \in \cS_{\bm{x}}$.
    \end{itemize}
    If either $\cS_{\bm{x}}$ or $\cS_{\bm{y}}$ is compact, then \[\inf_{\bm{x} \in \cS_{\bm{x}}} \sup_{\bm{y} \in \cS_{\bm{y}}}\, f(\bm{x}, \bm{y}) = \sup_{\bm{y} \in \cS_{\bm{y}}}\inf_{\bm{x} \in \cS_{\bm{x}}}\, f(\bm{x}, \bm{y})\]
\end{lemma}

We next introduce a bit of notation that will simplify the statements of our results.

\begin{definition}[Moreau Envelope, Proximal Operator]
\label{def:proximal}
    Given $\tau \ge 0$, the \emph{Moreau envelope} of a closed, proper, convex function $f$ is
    \[
        e_f(\bm{x}; \tau) = \begin{cases}
            \min_{\bm{y} \in \R^n} \left\{\frac{1}{2\tau}\norm{\bm{y} - \bm{x}}^2 + f(\bm{y})\right\} &\text{if } \tau > 0,\\
            f(\bm{x}) & \text{if }\tau = 0.
        \end{cases}
    \]
    Its \emph{proximal operator} is the point $\bm{y}$ that minimizes $e_f(\bm{x}; \tau)$: \[\prox_{\tau f}(\bm{x}) = \argmin_{\bm{y} \in \R^n} \left\{\frac{1}{2}\norm{\bm{y} - \bm{x}}^2 + \tau f(\bm{y})\right\}.\]
\end{definition}
Here, the subscript $\tau f$ refers to the product $\tau \cdot f$, i.e. the function $f$ scaled by $\tau$. In the case that $\tau = 1$, we will simply write $\prox_f(\bm{x})$. When $\tau = 0$, so that $\tau f$ vanishes everywhere, the function $\prox_{\tau f}$ coincides with the identity function, i.e. $\prox_0(\bm{x}) = \bm{x}$.

Some remarks are in order. The proximal operator can be interpreted as a generalization of the projection operator. Indeed, suppose there is a nonempty, compact, convex set $C \subseteq \R^d$ for which \[f(\bm{x}) = \begin{cases} 0 &\text{if }\bm{x} \in C \\ +\infty &\text{if }\bm{x} \notin C.\end{cases}\]
Then, $f$ is a closed, proper, convex function and \[\prox_f(\bm{x}) = \argmin_{\bm{y} \in C}\, \norm{\bm{y} - \bm{x}}.\] Alternatively, the proximal operator can be interpreted as a \emph{backward} gradient step with step size $\tau$. Indeed, if $f$ is differentiable and $\bm{y} = \prox_{\tau f}(\bm{x})$, then \[\bm{y} = \bm{x} - \tau \nabla f(\bm{y}).\] For general $f$, we have that $\bm{y}$ is the unique point such that \[\bm{y} \in \bm{x} - \tau \partial f(\bm{y}),\] where $\partial f$ denotes the subgradient (see Definition~\ref{def:subdifferential}). A priori, it is not apparent that this equation has a unique solution for $\bm{y}$, but this follows from the definition of $\prox_{f}$ in Definition~\ref{def:proximal} and the strong convexity of $\frac{1}{2}\norm{\bm{y} - \bm{x}}^2$. For comparison, the standard equation for a \emph{forward} gradient step is \[\bm{y} = \bm{x} - \tau \nabla f(\bm{x}).\]

We remark that the proximal operator has a useful relationship to the convex conjugate operation, and that its derivatives can be expressed concisely:

\begin{definition}[Convex Conjugate]
    The \emph{convex conjugate} of an extended real-valued  function $f : \R^n \to \R \cup \{\pm \infty\}$ is the convex function \[f^\star(\bm{y}) = \sup_{\bm{x} \in \R^n} \langle \bm{x}, \bm{y}\rangle - f(\bm{x}).\]
    If $n = 1$ and $f$ is convex, then $f^\star$ is also called the \emph{Legendre transform} of $f$, and $(f^\star)' = (f')^{-1}$.
\end{definition}

\begin{lemma}[Biconjugation]
    For a closed, proper, convex function $f$, we have $f^{\star\star} = f$.
\end{lemma}

\begin{lemma}[Proximal Operator of Conjugate]
\label{thm:prox-conjugate} $\prox_{f^\star}(\bm{x}) = \bm{x} - \prox_{f}(\bm{x})$.
\end{lemma}

\begin{lemma}[Moreau Envelope Derivatives, \cite{salehi2019impact}]
    \[
        \nabla_{\bm{x}} e_f(\bm{x}; \tau) = \frac{\bm{x} - \prox_{\tau f}(\bm{x})}{\tau}, \qquad \nabla_\tau e_f(\bm{x}; \tau) = - \frac{1}{2}\Norm{\frac{\bm{x} - \prox_{\tau f}(\bm{x})}{\tau}}^2
    \]
\end{lemma}

\section{Utility of Objective Perturbation for Robust Linear Regression}
\label{sec:objective-perturbation-linear}

\begin{algorithm}[t]
    \caption{Objective Perturbation}
    \label{alg:objective-perturbation}
    \begin{algorithmic}[1]
        \State \textbf{input:} design matrix $\bm{X} = [\bm{x}_1\,\cdots\,\bm{x}_n]^\top \in \R^{n \times d}$ with $\norm{\bm{x}_i} \le R$, response vector $\bm{y} \in \R^n$, loss function $\ell : \R^d \times \R^{d + 1} \to \R$, regularization strength $\lambda > 0$, perturbation strength $\nu > 0$.

        \vspace{0.5cm}
    
        \State Randomly sample the coefficients of the linear perturbation term: \[\bm{\xi} \sim \mathcal{N}(\bm{0}, \bm{I}_d)\]
        
        \State Optimize the regularized and perturbed objective function: \[\widehat{\bm{\beta}} = \argmin_{\bm{\beta} \in \R^d}\, \sum_{i=1}^n \ell(\bm{\beta}; (\bm{x}_i, y_i)) + \frac{\lambda}{2}\norm{\bm{\beta}}^2 + \nu\langle\bm{\xi}, \bm{\beta}\rangle\]
        
        \State \Return $\widehat{\bm{\beta}}$
    \end{algorithmic}
\end{algorithm}

In this section, we precisely characterize the privacy-utility tradeoff for the well-known objective perturbation algorithm \cite{chaudhuri2011differentially, kifer2012private, redberg2023improving} when applied to the problem of robust linear regression. The version of the algorithm that we consider is described in \cref{alg:objective-perturbation}. Before stating the main result of this section (\cref{thm:main-huber-objective-perturbation}), we briefly review the setup for robust linear regression, comment on the assumptions that will be necessary for our utility analysis, and give insight into the meaning of the theorem's conclusion, which is significantly stronger than in the simplified statement \cref{thm:main-objective-huber-informal} that we gave in the introduction.

\paragraph{Robust Linear Regression Model.} We are given a dataset $(\bm{X}, \bm{y})$ of $n$ samples. In the worst case, $(\bm{x}_i, y_i)$ are arbitrary points in $\R^{d + 1}$ with $\norm{\bm{x}_i} \le R$, but in the average case, there exists a ground-truth coefficient vector $\bm{\beta}^\star \in \R^d$ and regression error vector $\bm{\varepsilon}^\star \in \R^n$ independent of $\bm{X}$ such that
\[
    \bm{y} = \bm{X}\bm{\beta}^\star + \bm{\varepsilon}^\star.
\]

To dampen the effect of outliers in the worst-case setting, we do not seek the least squares estimate, but rather a vector $\bm{\hb}$ that minimizes the average \emph{Huber loss} (robust linear loss) on the dataset:

\begin{definition}
    The $L$-Lipschitz \emph{Huber loss} is \[H_L(r) = \begin{cases} \frac{1}{2}r^2 &\text{if }\abs{r} \le L \\ L\abs{r}-\frac{1}{2}L^2 &\text{if } \abs{r} \ge L.\end{cases}\]
    We denote the \emph{$L$-truncation} of $r \in \R$ by \[[r]_L = H'_L(r) = \begin{cases}
        -L &\text{if } r \le -L, \\
        r &\text{if }r \in [-L, +L], \\
        +L &\text{if }r \ge L.
    \end{cases}.\]
    For $\bm{r} \in \R^n$, we define $H_L(\bm{r}) = \sum_{i=1}^n H_L(r) \in \R$ and $[\bm{r}]_L = \nabla H_L(\bm{r}) = ([r_1]_L, \ldots, [r_n]_L) \in \R^n$.
\end{definition}

\begin{figure}
    \centering
    \includegraphics{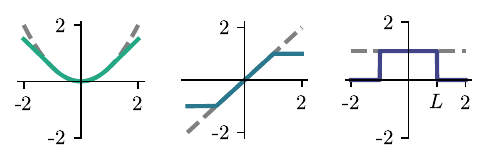}
    \caption{The functions $H_L$, $H_L'$, and $H_L''$ with $L = 1$.}
    \label{fig:huber-loss}
\end{figure}

\paragraph{Assumptions.} The main assumption for our utility analysis will be that the matrix $\bm{X}$, which we call the \emph{design matrix} or \emph{feature matrix}, is subgaussian in the sense defined below.

\begin{definition}[Subgaussian Design]
\label{def:subgaussian-design}
    We say a random matrix $\bm{X}\in B_R(\bm{0})^n$ with $d = \Theta(n)$ follows a \emph{subgaussian design} if its entries $x_{ij}$ are independent and satisfy the following conditions.
    \begin{enumerate}[(a)]
        \item For all $i \in [n]$ and $j \in [d]$, we have $\E[x_{ij}] = 0$ and $\E(x_{ij})^2 = 1/d$.
        \item $\max_{i \in [n]} \max_{j \in [d]}\, \norm{x_{ij}}_{\psi_2} = O(1/\sqrt{n}).$
    \end{enumerate}
\end{definition}

Note that any distribution supported on $\left[-\frac{R}{\sqrt{d}}, +\frac{R}{\sqrt{d}}\right]^{n \times d} \subseteq B_R(\bm{0})^n$ with independent, variance $1/d$ entries automatically satisfies the conditions of Definition~\ref{def:subgaussian-design}. For simplicity, we have required that the variance per coordinate is precisely $1/d$, but this is merely a scaling convention, and our results extend in a straightforward manner to other isotropic distributions. In fact, given a recent line of work analyzing estimators in the proportional regime for design matrices with nonisotropic covariances or otherwise dependent entries \cite{zhao2022arbitrary, fan2022amp, bao2023leaveoneoutapproachapproximatemessage, gerbelot2024meanfield}, we suspect that the result of \cref{thm:main-huber-objective-perturbation} can be further extended to nonisotropic $\bm{X}$.

More interestingly, although we do impose some mild conditions on the ground-truth $\bm{\beta}^\star$ and the regression errors $\bm{\varepsilon}^\star$, we do \emph{not} require them to be random! This is important because several previous works indeed assume that $\bm{\beta}^\star$ and $\bm{\eps}^\star$ are random with i.i.d. coordinates (e.g. \cite{sur2019modern, salehi2019impact}, or Theorem 3.12 of \cite{han2023universality}). We will clarify the conditions on $\bm{\beta}^\star$ and $\bm{\eps}^\star$ shortly once we have defined a certain notion of convergence.

\paragraph{Notion of Convergence.} \cref{thm:main-huber-objective-perturbation} will provide estimates for a large collection of scalar quantities that one might wish to compute on the coefficient vector output by the objective perturbation algorithm, along with non-asymptotic error bounds for those estimates that hold with high probability and decay at a $n^{-\Omega(1)}$ rate. In order to state the theorem as concisely as possible, we introduce the notation $\rightsquigarrow$ to describe this type of convergence.\footnote{Although the $\rightsquigarrow$ abbreviation is our own, studying convergence with respect to pseudo-Lipschitz tests is certainly not new; it is considered, for example, in \cite{bayati2011amp, celentano2020gfom, han2024entrywise}, among many other works.}

\begin{definition}[Polynomial-Rate Pseudo-Lipschitz Convergence]
\label{def:pl-convergence}
    Consider a random variable $x_0 \in \R$ with finite moments of all orders and a vector $\bm{x} \in \R^n$. We write $\bm{x} \rightsquigarrow x_0$, if for all pseudo-Lipschitz functions $f : \R \to \R$,
    \begin{equation}
    \label{eq:convergence}
        \exists c > 0 \quad\text{such that}\quad \Pr\mathopen{}\left[\Abs{\frac{1}{n}\sum_{i=1}^n f(x_i) - \E[f(x_0)]} \le n^{-c}\right] \ge 1 - O(n^{-c}).
    \end{equation}
    For $\bm{x} \in \R^m$ with $m = \Theta(n)$, we write $\bm{x} \rightsquigarrow x_0$ if \eqref{eq:convergence} holds with $m$ in place of $n$. Similarly, if $\bm{x}_0 \in \R^\ell$ and $\bm{X} = [\bm{x}_1\,\cdots\,\bm{x}_{m}]^\top \in \R^{m \times \ell}$ for $\ell = \Theta(1)$, we write $\bm{X} \rightsquigarrow \bm{x}_0$ if \eqref{eq:convergence} holds for all pseudo-Lipschitz $f : \R^\ell \to \R$ with $m$ in place of $n$.
\end{definition}

Observe that Definition~\ref{def:pl-convergence} can be applied either to a sequence of deterministic vectors $\bm{x} \in \R^m$, or a sequence of random vectors $\bm{x} \in \R^m$. In the case of deterministic vectors, the condition that the absolute difference is bounded by $n^{-c}$ with probability $1 - O(n^{-c})$ becomes equivalent to the condition that the absolute difference is deterministically bounded by $n^{-c}$. Next, observe that Definition~\ref{def:pl-convergence} is a strengthening of the familiar notion of \emph{convergence in distribution}, applied to the empirical distribution with mass $1/m$ on each of the coordinates of $\bm{x}$. Convergence in distribution is weaker than Definition~\ref{def:pl-convergence} because the former only requires \eqref{eq:convergence} to hold for functions $f$ that are both bounded and Lipschitz. Also, convergence in distribution only requires an $o(1)$ convergence rate, but Definition~\ref{def:pl-convergence} requires a much stronger $n^{-\Omega(1)}$ convergence rate. The following lemmas provide intuition for the meaning of $\bm{x} \rightsquigarrow x_0$. We defer their proofs to \cref{sec:pl-proofs}.

\begin{lemma}
\label{thm:pl-from-randomness}
    If $x_0 \in \R$ has finite moments and $x_1, \ldots, x_m \iid x_0$, then $(x_1, \ldots, x_m) \rightsquigarrow x_0$.
\end{lemma}

\begin{lemma}
\label{thm:uniform-bound-from-convergence}
    If $x_0 \in \R$ has finite moments and $\bm{x} \rightsquigarrow x_0$, then for all $c > 0$, w.h.p., $\norm{\bm{x}}_\infty \le O(n^c)$.\footnote{Moreover, if $\bm{x} \in \R^m$ is deterministic and $\bm{x} \rightsquigarrow x_0$, then $\norm{\bm{x}}_\infty \le n^{o(1)}$ holds deterministically.}
\end{lemma}

In the following theorem, we will phrase not only the conclusion in terms of pseudo-Lipschitz convergence, but also the assumptions on $\bm{\beta}^\star$ and $\bm{\eps}^\star$. For a somewhat contrived example demonstrating the breadth of $\bm{\beta}^\star$ to which our theorem can be applied, suppose the first $50\%$ of the coordinates of $\bm{\beta}^\star$ are deterministically equal to $0$, the next $20\%$ of the coordinates of $\bm{\beta}^\star$ are sampled i.i.d. from $\mathcal{N}(-3, 16)$, and the final $30\%$ deterministically comprise an arithmetic progression starting at $-10$ and ending at $10$. In this example, we have that $\bm{\beta}^\star \rightsquigarrow \beta^\star_0$ a.s., where $\beta^\star_0 \in \R$ follows a mixture distribution of a point mass at $0$, the distribution $\mathcal{N}(-3, 16)$, and the uniform distribution on $[-10, 10]$. 
To verify the condition \eqref{eq:convergence} with respect to, say, $f(t) = t^2$, one can check that w.h.p., $\frac{1}{d}\norm{\bm{\beta}^\star}^2 = \E(\beta^\star_0)^2 \pm O(n^{-1/2})$, where
\[
    \E(\beta^\star_0)^2 = 50\% \cdot 0 + 20\% \cdot (3^2 + 16) + 30\% \cdot \frac{1}{12}(10-(-10))^2 = 15.
\]
For an example of $\bm{\beta}^\star \in \R^d$ that our theorem does \emph{not} handle, consider $\bm{\beta}^\star = (\sqrt{d}, 0, \ldots, 0)$. On the one hand, the empirical distribution of the coordinates of this $\bm{\beta}^\star$ converges in distribution to the point mass at $\bm{0}$, so if $\bm{\beta}^\star$ were to converge with respect to pseudo-Lipschitz tests, it would also have to be toward the point mass at $\bm{0}$. On the other hand, we have $\frac{1}{d}\norm{\bm{\beta}^\star}^2 = 1 \neq 0 \pm n^{-\Omega(1)}$, so $\bm{\beta}^\star$ fails to satisfy the condition \eqref{eq:convergence} for $f(t) = t^2$.

We are now ready to state and prove the main theorem of this section.

\begin{theorem}
\label{thm:main-huber-objective-perturbation}
    Let $\widehat{\bm{\beta}}$ denote the output of \cref{alg:objective-perturbation} with parameters $R, \lambda, \nu > 0$ and instantiated with the $L$-Lipschitz Huber loss function: \[\ell(\bm{\beta}; (\bm{x}, y)) = H_L(y - \langle \bm{x}, \bm{\beta}\rangle).\]
    \begin{enumerate}[(a)]
        \item \emph{\textbf{(Privacy)}} 
        \label{thm:main-huber-objective-privacy}
        $\widehat{\bm{\beta}}$ satisfies $(\pe, \pd)$-differential privacy for any $\pe \ge 0$ and
        \[
            \pd = \begin{cases}
                2 \cdot \mathrm{HockeyStick}(\tpe, \frac{LR}{\nu}) &\text{if }\hpe \ge 0,\\
                (1 - e^{\hpe}) + 2e^{\hpe}\cdot \mathrm{HockeyStick}\left(\frac{L^2R^2}{2\nu^2}, \frac{LR}{\nu} \right) &\text{otherwise,}
            \end{cases}
        \]
        where we set $\tpe = \pe - \log(1 + R^2/\lambda)$ and $\hpe = \tpe - L^2R^2/2\nu^2$.
        
        \item \label{thm:main-huber-objective-utility} \emph{\textbf{(Utility)}} Suppose the following hold for some $\bm{\beta}^\star \in \R^d$ and $\bm{\varepsilon}^\star \in \R^n$ as $n \to \infty$ and $d/n \to \delta$:
        \begin{enumerate}[(i)]
            \item $\bm{X} \in B_R(\bm{0})^n \subseteq \R^{n \times d}$ follows a subgaussian design and $\bm{y} = \bm{X}\bm{\beta}^\star + \bm{\varepsilon}^\star$.
            \item There exist random variables $\beta^\star_0, \varepsilon^\star_0\in \R$ such that $\bm{\beta}^\star \rightsquigarrow \beta^\star_0$ and $\bm{\varepsilon}^\star \rightsquigarrow \varepsilon^\star_0$.
        \end{enumerate}
        Suppose there exist $\sigma^\star, \tau^\star > 0$ solving the following system of two scalar equations in two variables $(\sigma, \tau)$, which we write in terms of a dummy variable $Z \sim \mathcal{N}(0, 1)$ and $\kappa^2 = \E(\beta^\star_0)^2$ as
        \begin{subequations}
        \begin{align}
            \sigma^2 &= \tau^2\left(\frac{1}{\delta} \E\mathopen{}\left[\frac{\sigma Z + \varepsilon^\star_0}{1 + \tau}\right]_L^2 + \lambda^2\kappa^2 + \nu^2\right), \label{eq:huber-1} \\
            \tau &= \frac{1}{\lambda \delta}\left(\delta - \frac{\tau}{1+\tau}\Pr\mathopen{}\left[-L < \frac{\sigma Z + \varepsilon^\star_0}{1 + \tau} < L\right]\right). \label{eq:huber-2}
        \end{align}
        \end{subequations}
        Then, in terms of $(\sigma^\star, \tau^\star)$, the estimation error $\widehat{\bm{\beta}} - \bm{\beta}^\star$ satisfies, for $\xi_0 \sim \mathcal{N}(0, 1)$,
        \[
            (\bm{\beta}^\star, \, \bm{\xi}, \, \widehat{\bm{\beta}} - \bm{\beta}^\star) \rightsquigarrow \left(\beta_0^\star, \; \xi_0, \; \tau^\star{\left(\sqrt{\frac{1}{\delta} \E\mathopen{}\left[\frac{\sigma^\star Z + \varepsilon^\star_0}{1 + \tau^\star}\right]_L^2} Z - \lambda \beta^\star_0 - \nu \xi_0\right)}\right).
        \]
        Similarly, the $L$-truncated prediction error $[\bm{y} - \bm{X}\widehat{\bm{\beta}}]_L$ satisfies
        \[
            (\bm{\varepsilon}^\star, \, [\bm{y} - \bm{X}\widehat{\bm{\beta}}]_L) \rightsquigarrow \left(\varepsilon^\star_0, \; \left[\frac{\sigma^\star Z + \varepsilon^\star_0}{1 + \tau^\star}\right]_L\right).
        \]
        
    \end{enumerate}
\end{theorem}

Several remarks about \cref{thm:main-huber-objective-perturbation} are in order. First, we comment that part \ref{thm:main-huber-objective-privacy}, whose proof we defer to \cref{sec:objective-perturbation-privacy}, is a variant of Theorem 3.1 in \cite{redberg2023improving} (for convenience, we included a statement of that prior result in \cref{sec:preliminaries}). Essentially, part \ref{thm:main-huber-objective-privacy} establishes a link between any constant values of $L, \lambda, \nu > 0$ and a family of $(\pe, \pd)$-DP privacy guarantees. In \cref{sec:objective-perturbation-privacy}, we also give a single-parameter $\pr$-zCDP privacy guarantee with a constant $\pr = \pr(L, \lambda, \nu)$, which is arguably easier to interpret but not as tight.

Next, we comment that part \ref{thm:main-huber-objective-utility}, whose proof will comprise the bulk of this section, yields several interesting corollaries about the behavior of $\bm{\hb}$ when combined with Definition~\ref{def:pl-convergence}. For example, suppose we would like to know the \emph{bias} of $\widehat{\bm{\beta}}$. If we take the order-$2$ pseudo-Lipschitz function $f(\beta^\star, u) = (\beta^\star)^2 + \beta^\star u$, then \cref{thm:main-huber-objective-perturbation}\ref{thm:main-huber-objective-utility} with Definition~\ref{def:pl-convergence} tells us that the average value of $f$ over the coordinates of $(\bm{\beta}^\star, \bm{\hb} - \bm{\beta}^\star)$ can be approximated by the expected value of $f$ over the randomness of $(\beta^\star_0, \tau^\star(\sigma^\star_v Z - \lambda \beta^\star_0 - \nu\xi_0))$ for an appropriate constant $\sigma^\star_v > 0$. Since $Z$, $\beta^\star_0$, and $\xi_0$ are independent, we deduce that with high probability over the design matrix $\bm{X}$ and the randomness $\bm{\xi}$ of the algorithm,
\begin{align*}
    \frac{1}{d}\langle \bm{\hb}, \bm{\beta}^\star \rangle &= \frac{1}{d}\sum_{j=1}^d f(\beta^\star_j, \hb_j - \beta^\star_j) = \E[(\beta^\star_0)^2 + \beta^\star_0 \cdot \tau^\star(\sigma^\star_vZ - \lambda \beta^\star_0 - \nu\xi_0)] \pm n^{-\Omega(1)} \\
    &= (1 - \tau^\star \lambda)\kappa^2 \pm n^{-\Omega(1)}.
\end{align*}
It is apparent from equation \eqref{eq:huber-2} that $1 - \tau^\star \lambda \in (0, 1)$. Therefore, we deduce that $\bm{\hb}$ is positively correlated with $\bm{\beta}^\star$, but ``shrunk'' toward $\bm{0}$ by a factor of $1 - \tau^\star \lambda \pm n^{-\Omega(1)}$. This is intuitively what we would expect given that the regularization term $\frac{\lambda}{2}\norm{\bm{\beta}}^2$ is centered at the origin.

Similarly, suppose we are interested in the \emph{mean squared error} of $\bm{\hb}$. If we take the order-$2$ pseudo-Lipschitz function $f(u) = u^2$, then  by \cref{thm:main-huber-objective-perturbation}\ref{thm:main-huber-objective-utility} and Definition~\ref{def:pl-convergence}, with high probability,
\begin{align*}
    \frac{1}{d}\norm{\widehat{\bm{\beta}} - \bm{\beta}^\star}^2 &= \frac{1}{d}\sum_{j=1}^d f(\hb_j - \beta^\star_j) = \E(\tau^\star(\sigma^\star_vZ - \lambda \beta^\star_0 - \nu\xi_0))^2 \pm n^{-\Omega(1)} \\
    &= (\tau^\star)^2((\sigma^\star_v)^2 + \lambda^2\kappa^2 + \nu^2) \pm n^{-\Omega(1)}.
\end{align*}
By equation \eqref{eq:huber-1}, the above expression is precisely $(\sigma^\star)^2 \pm n^{-\Omega(1)}$, so we deduce that the root mean squared error of $\bm{\beta}^\star$ is precisely $\sigma^\star \pm n^{-\Omega(1)}$. Similar computations can be carried out in this manner to determine, for example, the correlation of $\bm{\hb}$ with the perturbation vector $\bm{\xi}$ or the fact that the residual error of the predictions $\bm{X}\bm{\hb}$ satisfies, w.h.p., \[\frac{1}{n}\Norm{[\bm{y} - \bm{X}\bm{\hb}]_L}^2 = \E\mathopen{}\left[\frac{\sigma^\star Z + \varepsilon^\star_0}{1 + \tau^\star}\right]_L^2 \pm n^{-\Omega(1)}.\] The previous two examples both measure utility using squared $\ell^2$ norm, but this was arbitrary. Indeed, by choosing different pseudo-Lipschitz test functions, we could just as easily have expressed many other possible utility metrics, including $\frac{1}{d}\norm{\bm{\hb} - \bm{\beta}^\star}_p^p$ or $\frac{1}{n}\norm{[\bm{y} - \bm{X}\bm{\hb}]_L}_p^p$ for any $p \ge 1$, in terms of $(\sigma^\star, \tau^\star)$.
We validate these corollaries against synthetic data in Figure~\ref{fig:huber-objective-estimation}. 

Next, we remark that equations \eqref{eq:huber-1}, \eqref{eq:huber-2} can be simplified if one assumes approximately Gaussian errors. Indeed, if $\varepsilon^\star_0 \sim \mathcal{N}(0, \sigma_\varepsilon^2)$ for some $\sigma_\eps > 0$, then we have \[\frac{\sigma Z + \varepsilon^\star_0}{1 + \tau} \sim \mathcal{N}\left(0, \frac{\sigma^2 + \sigma^2_\varepsilon}{(1 + \tau)^2}\right).\] Thus, we can express the variance of its $L$-truncation, as well as the probability it falls in the interval $[-L, +L]$, using rescaled versions of the following formulas for the truncated Gaussian distribution: if $Z \sim \mathcal{N}(0, 1)$ and $s > 0$, then in terms of the standard Gaussian PDF $\varphi$ and CDF $\Phi$,
\begin{align*}
    \Pr[ s Z \in [-L, +L]] &= 2\Phi(L/s) - 1. \\
    \E[(s Z)^2 \mid s Z \in [-L, +L]] &= s^2\left(1 - \frac{2(L/s)\varphi(L/s)}{2\Phi(L/s)-1}\right).
\end{align*}

\subsection{Special Case of the Utility Proof}

Before we begin the full utility proof, we will first prove the following special case. This special case assumes not only that we are given average-case features $\bm{X}$, but also that the ground-truth coefficient vector $\bm{\beta}^\star$ and error vector $\bm{\varepsilon}^\star$ have i.i.d. coordinates. The conclusion of this special case consists of a single scalar, the limit of $\frac{1}{d}\norm{\widehat{\bm{\beta}} - \bm{\beta}^\star}^2$. Eventually, for the full proof of \cref{thm:main-huber-objective-perturbation}\ref{thm:main-huber-objective-utility}, we will remove these unnecessary i.i.d. assumptions on $\bm{\beta}^\star$ and $\bm{\eps}^\star$ and determine the limit of both the estimation error $\widehat{\bm{\beta}} - \bm{\beta}^\star$ and prediction error $[\bm{y} - \bm{X}\widehat{\bm{\beta}}]_L$ with respect to all pseudo-Lipschitz test functions. However, we choose to first emphasize this special case because of its exceedingly simple proof given prior results in the high-dimensional statistics literature on non-private robust linear regression.

\begin{theorem}
\label{thm:main-huber-special-case}
    Consider the setting of \cref{thm:main-huber-objective-perturbation}\ref{thm:main-huber-objective-utility}, but also assume that $\beta^\star_1, \ldots, \beta^\star_d \iid \beta^\star_0$ and $\varepsilon^\star_1, \ldots, \varepsilon^\star_n \iid \varepsilon^\star_0$ and that $\varepsilon^\star_0$ is continuous. Define $\sigma^\star$ as in \cref{thm:main-huber-objective-perturbation}. Then, as $n \to \infty$, \[\frac{1}{d}\norm{\widehat{\bm{\beta}} - \bm{\beta}^\star}^2 \toP (\sigma^\star)^2.\]
\end{theorem}

The i.i.d. assumption on $\bm{\beta}^\star$ and $\bm{\varepsilon}^\star$ is stronger than assuming $\bm{\beta}^\star \rightsquigarrow \beta^\star_0$ and $\bm{\varepsilon}^\star \rightsquigarrow \varepsilon^\star_0$, which we saw in \cref{thm:pl-from-randomness}. Similarly, the conclusion of \cref{thm:main-huber-special-case} is weaker than that of \cref{thm:main-huber-objective-perturbation}.

The idea behind \cref{thm:main-huber-special-case} is a simple reduction to the non-private case. Indeed, the asymptotic behavior in the non-private case was derived by \cite{thrampoulidis2018mestimators} for Gaussian $\bm{X}$ and extended by \cite{han2023universality} to subgaussian designs.

\begin{proof}(\cref{thm:main-huber-special-case})
    In the absence of perturbation, which is the case when $\nu = 0$, Theorem 3.12 of \cite{han2023universality} implies that under the stated assumptions, the scaled error of $\widehat{\bm{\beta}}$ converges to the solution $\sigma^\star$ to the following system of two scalar equations in two variables $(\sigma, \tau)$. In terms of a dummy variable $Z \sim \mathcal{N}(0, 1)$, these equations are
    \begin{align*}
        \frac{1}{\delta} - 1 + \tau \lambda &= \frac{1}{\delta} \E[\prox_{\tau H_L}'(\sigma Z + \varepsilon^\star_0)],\\
        \sigma^2 &= \frac{1}{\delta}\E[(\sigma Z + E - \prox_{\tau H_L} (\sigma Z + \varepsilon^\star_0))^2] + \lambda^2 \tau^2 \kappa^2.
    \end{align*}
    A simple calculation using the definition of the proximal operator (Definition~\ref{def:proximal}) shows that for any $\tau > 0$ and $s \in \R$,
    \[s - \prox_{\tau H_L}(s) = \tau\mathopen{}\left[\frac{s}{1+\tau}\right]_L.\]
    If $\abs{s} \neq  (1 + \tau)L$, then a similarly simple calculation shows that
    \begin{align*}
        1 - \prox_{\tau H_L}'(s) &= \begin{cases} \frac{\tau}{1 + \tau} &\text{if } \abs{s} < (1 + \tau)L, \\ 0 &\text{if } \abs{s} > (1 + \tau)L.\end{cases}
    \end{align*}
    Substituting these formulas for $s - \prox_{\tau H_L}(s)$ and $1 - \prox_{\tau H_L}'(s)$ with $s = \sigma Z + \varepsilon^\star_0$ into the above system of equations yields the following simplification:
    \begin{equation}
    \label{eq:huber-non-private}
    \begin{split}
        \sigma^2 &= \tau^2\left(\frac{1}{\delta} \E\mathopen{}\left[\frac{\sigma Z + \varepsilon^\star_0}{1 + \tau}\right]_L^2 + \lambda^2\kappa^2\right), \\
        \tau &= \frac{1}{\lambda \delta}\left(\delta - \frac{\tau}{1 + \tau}\Pr\mathopen{}\left[-L < \frac{\sigma Z + \varepsilon^\star_0}{1 + \tau} < L\right]\right).
    \end{split}
    \end{equation}
    This proves \cref{thm:main-huber-special-case} in the non-private setting, where $\nu = 0$. To extend this result to the private setting, where $\nu > 0$, we carry out a straightforward reduction. To begin, we first rewrite the perturbed objective in \cref{alg:objective-perturbation} by completing the square to combine the quadratic regularization term with the linear perturbation term. Indeed,
    \[\widehat{\bm{\beta}} = \argmin_{\bm{\beta} \in \R^d}\, H_L(\bm{y} - \bm{X}\bm{\beta}) + \frac{\lambda}{2} \mathopen{}\left\lVert \bm{\beta} + \frac{\nu}{\lambda} \bm{\xi} \right\rVert^2\mathclose{}.\] Introducing the notation $\bm{u} = \bm{\beta} - \bm{\beta}^\star$ and $\widehat{\bm{u}} = \widehat{\bm{\beta}} - \bm{\beta}^\star$, and recalling that $\bm{\varepsilon}^\star = \bm{y} - \bm{X}\bm{\beta}^\star$, this becomes
    \[\widehat{\bm{u}} = \argmin_{\bm{u} \in \R^d}\, H_L(\bm{\varepsilon}^\star - \bm{X}\bm{u}) + \frac{\lambda}{2} \mathopen{}\left\lVert \bm{u} + \left(\bm{\beta}^\star + \frac{\nu}{\lambda} \bm{\xi}\right) \right\rVert^2\mathclose{}.\] In this formulation, since $\bm{X}$, $\bm{\beta}^\star$, and $\bm{\varepsilon}^\star$ are independent, it is clear that $\bm{\beta}^\star$ only influences the final, quadratic term of the optimization objective. Consequently, $\widehat{\bm{u}}$ has the same value in the case of $\nu > 0$ when the ground-truth vector is $\bm{\beta}^\star$ as in the non-private case when ground-truth vector is \[\bm{\beta}^\star + \frac{\nu}{\lambda}\bm{\xi}.\] In equation \eqref{eq:huber-non-private}, which we have already shown characterizes the limiting behavior in the non-private case, this amounts to defining $\xi_0 \sim \mathcal{N}(0, 1)$ and replacing each occurrence of $\E(\beta^\star_0)^2 = \kappa^2$ with \[\E\mathopen{}\left(\beta_0^\star+ \frac{\nu}{\lambda}\xi_0\right)^2 = \kappa^2 + \left(\frac{\nu}{\lambda}\right)^2.\] Doing so yields the system of equations from the statement of \cref{thm:main-huber-objective-perturbation}, which read
    \begin{align*}
        \sigma^2 &= \tau^2\left(\frac{1}{\delta} \E\mathopen{}\left[\frac{\sigma Z + E}{1 + \tau}\right]_L^2 + \lambda^2\kappa^2 + \nu^2\right), \\
        \tau &= \frac{1}{\lambda \delta}\left(\delta - \frac{\tau}{1+\tau}\Pr\mathopen{}\left[-L < \frac{\sigma Z + E}{1 + \tau} < L\right]\right).
    \end{align*}
    Letting $\sigma^\star$ be the solution for $\sigma$, we conclude that \[\frac{1}{\sqrt{d}}\norm{\widehat{\bm{\beta}} - \bm{\beta}^\star} \toP \sigma^\star.\]
\end{proof}

\subsection{Proof of Theorem~\ref{thm:main-huber-objective-perturbation}}
\label{sec:huber-objective-utility}

To prove \cref{thm:main-huber-objective-perturbation} in its full generality, we can no longer apply a black-box reduction to prior work in the non-private case since this work assumes that $\bm{\beta}^\star$ and $\bm{\varepsilon}^\star$ are i.i.d. and does not fully determine the asymptotic behavior of $\widehat{\bm{\beta}} - \bm{\beta}^\star$ and $[\bm{y} - \bm{X}\widehat{\bm{\beta}}]_L$. Instead, we will take a white-box approach, roughly following the proof strategy of \cite{thrampoulidis2018mestimators, han2023universality} and making several changes as necessary to accommodate the differences between their setting and ours. To reiterate, the main differences are the random linear perturbation term for differential privacy, our weakened assumptions on $\bm{\beta}^\star$ and $\bm{\varepsilon}^\star$, and our conclusion that includes a description of $[\bm{y} - \bm{X}\widehat{\bm{\beta}}]_L$ and involves a stronger notion of convergence. Consequently, the present proof is substantially more involved than the preceding one.

\paragraph{Proof Overview.}
Our proof proceeds in three steps, inspired by \cite{han2023universality}:
\begin{enumerate}[(1)]
    \item Relate the perturbed Huber objective to a certain min-max optimization problem.
    \item Use universality to relate this problem to another involving a Gaussian design matrix.
    \item Use CGMT to analyze this min-max optimization with a Gaussian design matrix as $n \to \infty$.
\end{enumerate}

For the first step, we will take the \emph{Legendre transform} (a.k.a. convex conjugate) of the Huber loss function, which in our case will have the effect of isolating the randomness of $\bm{X} \in \R^{n \times d}$ in a single, bilinear term.

For the second step, we will make critical use of Corollary 2.6 of \cite{han2023universality}, a powerful result that the authors call \emph{universality of the Gordon’s max-min (min-max) cost optimum}. The name ``Gordon'' comes from \emph{Gordon's inequality}, a precursor to the convex Gaussian minimax theorem (CGMT) that pertains to the same families of Gaussian random variables and their extrema. For convenience, we have included a simplified rephrasing of this result, which we refer to as \emph{CGMT universality}, as \cref{thm:cgmt-universality}.

For the third step, we will use CGMT, roughly following the template in \cite{thrampoulidis2018mestimators}, to determine the behavior of the estimation error $\widehat{\bm{\beta}} - \bm{\beta}^\star$ as $n \to \infty$. Unlike \cite{thrampoulidis2018mestimators}, our more detailed analysis will account for the random linear perturbation term, establish convergence in a stronger sense, and also explain the behavior of the residual error, $\bm{y} - \bm{X}\widehat{\bm{\beta}}$. For a statement of CGMT, see \cref{thm:cgmt}.

We remark that both steps 2 and 3 will be made more challenging by the fact that we do not assume that $\bm{\beta}^\star$ and $\bm{\varepsilon}^\star$ have i.i.d. coordinates.

\subsubsection{Step 1: Legendre Transform}

The goal of this step is to prove the following lemma, which relates the output of \cref{alg:objective-perturbation} with $L$-Lipschitz Huber loss to the min-max optimization of a certain random variable. Technically, the lemma holds in a worst-case sense, as neither the randomness of $\bm{X}$ and $\bm{\xi}$ nor any assumptions on $\bm{\beta}^\star$ or $\bm{\varepsilon}^\star$ are required for the proof. Eventually, the randomness of the optimization problem will arise from the randomness of the design matrix $\bm{X}$.

\begin{lemma}
\label{thm:objective-huber-legendre}
    Let $\widehat{\bm{\beta}}$ be the output of \cref{alg:objective-perturbation} when instantiated with \[\ell(\bm{\beta}; (\bm{x}, y)) = H_L(y - \langle \bm{x}, \bm{\beta}\rangle).\] Fix $\bm{\beta}^\star \in \R^d$ and define $\bm{\eps}^\star = \bm{y} - \bm{X}\bm{\beta}^\star$. Given $\bm{u} \in \R^d$ and $\bm{v} \in [-L, +L]^n$, define $Q_{\bm{u}, \bm{v}}$ as: \[Q_{\bm{u}, \bm{v}} = -\langle \bm{X} \bm{u}, \bm{v}\rangle + \psi(\bm{u}, \bm{v}) \qquad \text{for} \qquad \psi(\bm{u}, \bm{v}) = \frac{\lambda}{2}\norm{\bm{u}}^2 + \langle \lambda \bm{\beta}^\star + \nu\bm{\xi}, \bm{u}\rangle - \frac{1}{2}\norm{\bm{v}}^2 + \langle \bm{\varepsilon}^\star, \bm{v}\rangle.\] Then,
    \((\widehat{\bm{u}}, \widehat{\bm{v}}) = (\widehat{\bm{\beta}} - \bm{\beta}^\star, [\bm{y} - \bm{X}\widehat{\bm{\beta}}]_L)\)
    is the unique point in $\R^d \times [-L, +L]^n$ satisfying \[\max_{\bm{v} \in [-L, +L]^n}\, Q_{\widehat{\bm{u}}, \bm{v}} = \min_{\bm{u} \in \R^d}\, Q_{\bm{u}, \widehat{\bm{v}}}.\] We call it the \emph{saddle point} or \emph{Nash equilibrium} of $Q_{\bm{u}, \bm{v}}$.
\end{lemma}

\begin{proof}
    \cref{alg:objective-perturbation} simply computes
    \[\widehat{\bm{\beta}} = \argmin_{\bm{\beta} \in \R^d}\, H_L(\bm{y} - \bm{X}\bm{\beta}) + \frac{\lambda}{2}\norm{\bm{\beta}}^2 + \nu\langle \bm{\xi}, \bm{\beta}\rangle.\]
    Making the substitution $\bm{u} = \bm{\beta} - \bm{\beta}^\star$, we rewrite this as
    \[\widehat{\bm{u}} = \argmin_{\bm{u} \in \R^d}\, H_L(\bm{y} - \bm{X}(\bm{\beta}^\star + \bm{u})) + \frac{\lambda}{2}\norm{\bm{\beta}^\star + \bm{u}}^2 + \nu\langle \bm{\xi}, \bm{\beta}^\star + \bm{u}\rangle.\] Expanding the quadratic term, recalling that $\widehat{\bm{u}} = \widehat{\bm{\beta}} - \bm{\beta}^\star$ and $\bm{\varepsilon}^\star = \bm{y} - \bm{X}\bm{\beta}^\star$, and dropping terms that do not depend on $\bm{u}$, we simplify this to
    \begin{equation}
    \label{eq:huber-change-of-variables}
        \widehat{\bm{u}} = \argmin_{\bm{u} \in \R^d}\, H_L(\bm{\varepsilon}^\star - \bm{X}\bm{u}) + \frac{\lambda}{2}\norm{\bm{u}}^2 + \langle \lambda \bm{\beta}^\star + \nu\bm{\xi}, \bm{u}\rangle.
    \end{equation}
    Next, we take the \emph{Legendre transform} of $H_L$ in equation \eqref{eq:huber-change-of-variables}. In our case, this simply amounts to applying the following identity regarding $H_L$, which is valid for all $\bm{r} \in \R^n$: \[H_L(\bm{r}) = \max_{\bm{v} \in [-L, +L]^n}\, \langle \bm{\bm{r}, \bm{v}}\rangle - \frac{1}{2}\norm{\bm{v}}^2.\] The maximum is achieved iff $\bm{v} = [\bm{r}]_L$. Setting $\bm{r} = \bm{\varepsilon}^\star - \bm{X}\bm{u}$,
    \[
        \widehat{\bm{u}} = \argmin_{\bm{u} \in \R^d} \max_{\bm{v} \in [-L, +L]^n}\, \langle \bm{\varepsilon}^\star - \bm{X}\bm{u}, \bm{v}\rangle - \frac{1}{2}\norm{\bm{v}}^2 + \frac{\lambda}{2}\norm{\bm{u}}^2 + \langle \lambda \bm{\beta}^\star + \nu\bm{\xi}, \bm{u}\rangle.
    \]
    Note that this objective function is simply $Q_{\bm{u}, \bm{v}}$. Given $\bm{u}$, the maximum over $\bm{v}$ is achieved iff \[\bm{v} = [\bm{r}]_L = [\bm{\varepsilon}^\star - \bm{X}\bm{u}]_L.\] This coincides with $\widehat{\bm{v}} = [\bm{y} - \bm{X}\widehat{\bm{\beta}}]_L$ when $\bm{u} = \widehat{\bm{u}}$.
    The terms $\frac{\lambda}{2}\norm{\bm{u}}^2$ and $-\frac{1}{2}\norm{\bm{v}}^2$ ensure that $Q_{\bm{u}, \bm{v}}$ is $\lambda$-strongly convex in $\bm{u}$ and $1$-strongly concave in $\bm{v}$, so we conclude by the minimax theorem that $(\widehat{\bm{u}}, \widehat{\bm{v}})$ is the unique saddle point of $Q_{\bm{u}, \bm{v}}$.
\end{proof}

\subsubsection{Step 2: CGMT Universality}

The goal of this step is to use CGMT universality (\cref{thm:cgmt-universality} in this paper, or Corollary 2.6 of \cite{han2023universality}) to prove the following lemma, which relates the random variable $Q_{\bm{u}, \bm{v}}$ to the random variable $Q'_{\bm{u}, \bm{v}}$ that replaces $-\bm{X}$ with $\frac{1}{\sqrt{d}}\bm{G}$, where $\bm{G}$ has the same shape as $\bm{X}$ but independent, standard Gaussian entries:
\[Q'_{\bm{u}, \bm{v}} = \frac{1}{\sqrt{d}}\langle \bm{G} \bm{u}, \bm{v}\rangle + \psi(\bm{u}, \bm{v}).\]
Unlike $\bm{X}$, whose rows have $\ell^2$ norm bounded by $R$ (as required for differential privacy), the entries of the matrix $\bm{G}$ are unbounded in the worst case. Despite this, the lemma shows that the random variable obtained by minimizing over $\bm{u}$ and maximizing over $\bm{v}$  has nearly the same CDF regardless of whether we start with $Q_{\bm{u}, \bm{v}}$ or $Q'_{\bm{u}, \bm{v}}$. Note that the order of minimization over $\bm{u}$ and maximization over $\bm{v}$ below can be exchanged by the minimax theorem.

\begin{lemma}
\label{thm:objective-huber-universality}
    If $\bm{X}$ follows a subgaussian design and $\bm{y} = \bm{X}\bm{\beta}^\star + \bm{\eps}^\star$ for $\bm{\beta}^\star \rightsquigarrow \beta^\star_0$ and $\bm{\varepsilon}^\star \rightsquigarrow \varepsilon^\star_0$, then for all compact $\cS_{\bm{v}} \subseteq [-L, +L]^n$, for all compact $\cS_{\bm{u}} \subseteq [-L_{\bm{u}}, +L_{\bm{u}}]^d$ with $L_{\bm{u}} \le O(n^c)$ for a sufficiently small constant $c > 0$, and for all $\tau \in \R$, 
    \[\Pr\mathopen{}\left[\min_{\bm{u} \in \cS_{\bm{u}}} \max_{\bm{v} \in \cS_{\bm{v}}}\, Q'_{\bm{u}, \bm{v}} < \tau \right] \le \Pr\mathopen{}\left[\min_{\bm{u} \in \cS_{\bm{u}}} \max_{\bm{v} \in \cS_{\bm{v}}}\, Q_{\bm{u}, \bm{v}} < \tau + n^{1 - \Omega(1)}\right] + n^{-\Omega(1)}.\] The above inequality also holds if we swap $Q_{\bm{u}, \bm{v}}$ and $Q'_{\bm{u}, \bm{v}}$.
\end{lemma}

\begin{proof}
    We calculate the gradient of $\psi$, as defined in \cref{thm:objective-huber-legendre}: \[\nabla \psi(\bm{u}, \bm{v}) = \left(\lambda \bm{u} + \lambda \bm{\beta}^\star + \nu\bm{\xi},\, -\bm{v} + \bm{\varepsilon}^\star\right).\]
    Since $\bm{\beta}^\star \rightsquigarrow \beta^\star_0$ and $\bm{\varepsilon}^\star \rightsquigarrow \varepsilon^\star_0$, \cref{thm:uniform-bound-from-convergence} implies that w.h.p., their $\ell^\infty$ norms grow no faster than $O(n^{c})$. The same is true of $\bm{\xi}$ with high probability by Gaussianity, and for $\bm{u}$ and $\bm{v}$ by the assumptions on $\cS_{\bm{u}}$ and $\cS_{\bm{v}}$, respectively. It follows that w.h.p., \[\norm{\nabla \psi(\bm{u}, \bm{v})}_1 = O(n^{1 + c}).\] Rescaling the objective by $\frac{1}{n}$ and applying \cref{thm:cgmt-universality} with $\bm{A} = -\sqrt{d}\bm{X}$ and $\bm{B} = \bm{G}$ (and vice versa) and $g = n^{-1/21}$, we conclude that \[\Pr\mathopen{}\left[\min_{\bm{u} \in \cS_{\bm{u}}} \max_{\bm{v} \in \cS_{\bm{v}}}\, Q'_{\bm{u}, \bm{v}} < \tau\right] = \Pr\mathopen{}\left[\min_{\bm{u} \in \cS_{\bm{u}}} \max_{\bm{v} \in \cS_{\bm{v}}}\, Q_{\bm{u}, \bm{v}} < \tau \pm O(n^{20/21})\right] \pm O(n^{-\frac{1}{6} + \frac{1}{7} + 2c + o(1)}).\]
\end{proof}

The key nontrivial assumption of \cref{thm:objective-huber-universality} is that the $\ell^\infty$ diameter of the constraint set $\cS_{\bm{u}}$ satisfies $L_{\bm{u}} \le O(n^c)$ for a sufficiently small $c > 0$. Therefore, to make effective use of this result, we must verify that for all $c > 0$, w.h.p., \[\norm{\widehat{\bm{\beta}} - \bm{\beta}^\star}_\infty = O(n^c).\] By the triangle inequality, it suffices to check that $\norm{\bm{\beta}^\star}_\infty = O(n^c)$ and $\norm{\widehat{\bm{\beta}}}_\infty = O(n^c)$. For the former, \cref{thm:uniform-bound-from-convergence} and $\bm{\beta}^\star \rightsquigarrow \beta^\star_0$ imply that $\norm{\bm{\beta}^\star}_\infty = O(n^c)$ with high probability. For the latter, we prove the following lemma.

\begin{lemma}
\label{thm:objective-uniform-bound}
    Let $\bm{\hb}$ be the output of \cref{alg:objective-perturbation} with $L$-Lipschitz Huber loss $\ell(\bm{\beta}; (\bm{x}, y)) = H_L(y - \langle \bm{x}, \bm{\beta}\rangle)$. Fix any $\bm{\eps}^\star \in \R^n$ and $\bm{\beta}^\star \rightsquigarrow \beta^\star_0$. If $\bm{X}$ follows a subgaussian design and $\bm{y} = \bm{X}\bm{\beta}^\star + \bm{\eps}^\star$, then for all $c > 0$, w.h.p. over $\bm{X}$ and $\bm{\xi}$,
    \[\norm{\widehat{\bm{\beta}}}_\infty \le O(n^c).\]
\end{lemma}

\begin{proof}
    We use the leave-one-(feature)-out technique, as in \cite{han2023universality}, while accounting for the noise for differential privacy. First, denote the perturbed objective function by $H$:
    \[
        H(\bm{\beta}) = H_L(\bm{y} - \bm{X}\bm{\beta}) + \frac{\lambda}{2}\norm{\bm{\beta}}^2 + \nu\langle \bm{\xi}, \bm{\beta}\rangle.
    \]
    We wish to study the estimator $\widehat{\bm{\beta}}^{(j)}$ obtained by omitting the $j$\textsuperscript{th} feature from all data points. By equation \eqref{eq:huber-change-of-variables}, in which $\bm{u} = \bm{\beta} - \bm{\beta}^\star$ and $\bm{\hu} = \bm{\hb} - \bm{\beta}^\star$, we see that dropping the $j$\textsuperscript{th} column from $\bm{X}$ is equivalent to enforcing the constraint $u_j = 0$, or $\beta_j = \beta^\star_j$:
    \[
        \widehat{\bm{\beta}}^{(j)} = \argmin_{\bm{\beta} \in \mathbb{R}^d} \; H(\bm{\beta}) \quad\text{s.t.}\quad \beta_j = \beta^\star_j.
    \]
    To bound the $j$\textsuperscript{th} coordinate of $\widehat{\bm{\beta}}$, we will relate $\widehat{\bm{\beta}}$ to $\widehat{\bm{\beta}}^{(j)}$, whose $j$\textsuperscript{th} coordinate is $\beta^\star_j$ by definition. To this end, note that $H$ is $\lambda$-strongly convex, so
    \[
        \|\widehat{\bm{\beta}} - \widehat{\bm{\beta}}^{(j)}\| \le \frac{1}{\lambda}\|\nabla H(\widehat{\bm{\beta}}) - \nabla H(\widehat{\bm{\beta}}^{(j)})\|.
    \]
    The distance $\|\widehat{\bm{\beta}} - \widehat{\bm{\beta}}^{(j)}\|$ is at least $|\widehat{\beta}_j - \widehat{\beta}^{(j)}_j|$. 
    Since $\widehat{\bm{\beta}}$ minimizes $H$, the gradient $\nabla H(\widehat{\bm{\beta}})$ vanishes. Similarly, the gradient $\nabla H(\widehat{\bm{\beta}}^{(j)})$ vanishes in all but its $j$\textsuperscript{th} coordinate, so the inequality simplifies to
    \[
        |\widehat{\beta}_j - \beta^\star_j| \le \left|\frac{1}{\lambda}\nabla_j H(\widehat{\bm{\beta}}^{(j)})\right|.
    \]
    All that remains is to bound the partial derivative on the right side. By direct calculation,
    \[\nabla_j H(\bm{\beta}) = \left\langle [\bm{y} - \bm{X}\bm{\beta}]_L, \bm{X}_j\right\rangle + \lambda \beta_j + \nu\xi_j.\]
    Here, $\bm{X}_j \in \mathbb{R}^n$ denotes the $j$\textsuperscript{th} column of $\bm{X}$ (\emph{not} the $j$\textsuperscript{th} data point). Substituting $\bm{\beta} = \widehat{\bm{\beta}}^{(j)}$, whose $j$\textsuperscript{th} coordinate is $\beta^\star_j$, we see that
    \[\abs{\widehat{\beta}_j - \beta^\star_j} \le \left\lvert \frac{1}{\lambda}\nabla_j H(\widehat{\bm{\beta}}^{(j)}) \right\rvert = {\left\lvert \frac{1}{\lambda} \bigl\langle[\bm{y} - \bm{X}\bm{\hb}^{(j)}]_L, \bm{X}_j\bigr\rangle + \beta^\star_j + \frac{\nu}{\lambda}\xi_j \right\rvert}.\]
    Observe that we can rewrite the truncated residual vector  as $[\bm{y} - \bm{X}\bm{\hb}^{(j)}]_L = [\bm{\eps}^\star - \bm{X}(\bm{\hb}^{(j)} - \bm{\beta}^\star)]_L$, which, along with the fact that $\bm{\hb}^{(j)} - \bm{\beta}^\star$ vanishes in its $j$\textsuperscript{th} coordinate, makes it clear that it is independent of the vector $\bm{X}_j$. Since this truncated vector lies in $[-L,+L]^n$ and is independent of $\bm{X}_j$, we have by subgaussianity (\cref{thm:subgaussian-inner-product}) and the triangle inequality that
    \[
        \bigl\lVert \abs{\widehat{\beta}_j - \beta^\star_j} - \abs{\beta^\star_j} \bigr\rVert_{\psi_2} \le \frac{L\sqrt{n}}{\lambda}\norm{\bm{X}_j}_{\psi_2} + \frac{\nu}{\lambda} \norm{\xi_j}_{\psi_2}.
    \]
    In a subgaussian design, $\norm{\bm{X}_j}_{\psi_2} = O(1/\sqrt{n})$. Since $\bm{\xi}$ is Gaussian, $\norm{\xi_j}_{\psi_2} = O(1)$. It follows that $ \bigl\lVert \abs{\widehat{\beta}_j - \beta^\star_j} - \abs{\beta^\star_j} \bigr\rVert_{\psi_2} \le O(1)$ for each $j \in [d]$. Hence, by a union bound over all $j \in [d]$, as $t \to \infty$, \[\Pr\mathopen{}\left[\max_{j \in [d]}\, \Abs{\abs{\widehat{\beta}_j - \beta^\star_j} - \abs{\beta^\star_j}} > t\right] \le n \cdot \exp(-\Omega(t^2)).\] Setting $t = \log n$ and recalling that $\norm{\bm{\beta}}_\infty = \max_{j \in [d]}\,\abs{\beta_j}$, we see that with probability $1 - n^{-\Omega(\log n)}$,
    \[
        \norm{\bm{\hb}}_\infty \le \log n + 2\norm{\bm{\beta}^\star}_\infty.
    \]
    To complete the proof, recall from \cref{thm:uniform-bound-from-convergence} that for all $c > 0$, w.h.p., $\norm{\bm{\beta}^\star}_\infty = O(n^c)$ as well.
\end{proof}

\subsubsection{Step 3: CGMT Analysis}

In this step, we analyze the random variable $Q'_{\bm{u}, \bm{v}}$ using CGMT (\cref{thm:cgmt}), roughly following the strategy of \cite{thrampoulidis2018mestimators}.

In some ways, our proof is simpler than the proof in \cite{thrampoulidis2018mestimators}. This is partly because  we focus on the special case of robust linear regression with $\ell^2$ regularization, as opposed to a general loss function $\mathcal{L}$ and general regularizer $f$. It is also simpler because our proof circumvents the introduction of several extraneous scalar- and vector-valued variables that are used in prior work.

In other ways, our proof is somewhat more complex than prior papers in this line of work. This is partly because we consider the random perturbation $\bm{\xi}$ introduced for differential privacy, and partly because in order to make do with our weaker assumptions and still prove a stronger conclusion, we must take great care to bound the magnitude of various error terms.

To begin, we recall and analyze the \emph{auxiliary} random variable for CGMT, defined as
\[
    Q''_{\bm{u}, \bm{v}} = \frac{\norm{\bm{u}}}{\sqrt{d}}\langle \bm{h}, \bm{v}\rangle - \frac{\norm{\bm{v}}}{\sqrt{d}}\langle \bm{g}, \bm{u}\rangle + \psi(\bm{u}, \bm{v}) \qquad \text{for} \qquad \bm{g} \sim \mathcal{N}(\bm{0}, \bm{I}_{d}) \text{ and } \bm{h} \sim \mathcal{N}(\bm{0}, \bm{I}_{n}).
\]
Eventually, we will use CGMT (\cref{thm:cgmt}) to relate $Q_{\bm{u}, \bm{v}}$ via $Q'_{\bm{u}, \bm{v}}$ to $Q''_{\bm{u}, \bm{v}}$.

\begin{lemma}
\label{thm:huber-objective-aux}
    Define $\sigma^\star, \tau^\star > 0$ as in \cref{thm:main-huber-objective-perturbation}\ref{thm:main-huber-objective-utility}, and consider the pair $(\widetilde{\bm{u}}, \widetilde{\bm{v}})$ with
    \begin{align*}
        \widetilde{\bm{u}} = \tau^\star\mathopen{}\left(\sqrt{\frac{1}{\delta} \E\mathopen{}\left[\frac{\sigma^\star Z + \varepsilon^\star_0}{1 + \tau^\star}\right]_L^2} \bm{g} - \lambda\bm{\beta}^\star - \nu\bm{\xi}\right)\mathclose{},  \qquad \widetilde{\bm{v}} = \left[\frac{\sigma^\star \bm{h} + \bm{\varepsilon}^\star}{1+\tau^\star}\right]_L\mathclose{}.
    \end{align*}
    Then, under the assumptions of \cref{thm:main-huber-objective-perturbation}\ref{thm:main-huber-objective-utility}, there exists a constant $c^\star \in \R$ such that w.h.p.,
    \begin{itemize}
        \item The function $\bm{u} \mapsto Q''_{\bm{u}, \bm{\tv}}$ is $\lambda$-strongly convex in $\bm{u} \in \R^d$,
        \item The function $\bm{v} \mapsto Q''_{\bm{\tu}, \bm{v}}$ is $1$-strongly concave in $\bm{v} \in [-L, +L]^n$,
        \item The pair $(\bm{\tu}, \bm{\tv})$ satisfies 
        \[
            c^\star n - n^{1 - \Omega(1)} \le \min_{\bm{u} \in \R^d} Q''_{\bm{u}, \widetilde{\bm{v}}} \le  Q''_{\widetilde{\bm{u}}, \widetilde{\bm{v}}} \le \max_{\bm{v} \in [-L, +L]^n} Q''_{\widetilde{\bm{u}}, \bm{v}} \le c^\star n + n^{1 - \Omega(1)}.
        \]
    \end{itemize}
    We call $(\widetilde{\bm{u}}, \widetilde{\bm{v}})$ an \emph{approximate saddle point} of $Q''_{\bm{u}, \bm{v}}$.
\end{lemma}

\begin{proof}
    To show that $(\widetilde{\bm{u}}, \widetilde{\bm{v}})$ is an approximate saddle point of $Q''_{\bm{u}, \bm{v}}$, it suffices to show that each of $\widetilde{\bm{u}}$ and $\widetilde{\bm{v}}$ is an \emph{approximately best response} to the other, and that $\abs{Q''_{\widetilde{\bm{u}}, \widetilde{\bm{v}}} - c^\star n} \le n^{1 - \Omega(1)}$. Formally, we say that $\widetilde{\bm{u}}$ is an approximate best response to $\widetilde{\bm{v}}$ if
    \[
        Q''_{\widetilde{\bm{u}}, \widetilde{\bm{v}}} \le \min_{\bm{u} \in \R^d} Q''_{\bm{u}, \widetilde{\bm{v}}} + n^{1 - \Omega(1)}.
    \]
    Similarly, we say that $\widetilde{\bm{v}}$ is an approximate best response to $\widetilde{\bm{u}}$ if
    \[
        Q''_{\widetilde{\bm{u}}, \widetilde{\bm{v}}} \ge \max_{\bm{v} \in [-L, +L]^n} Q''_{\widetilde{\bm{u}}, \bm{v}} - n^{1 - \Omega(1)}.
    \]
    In order to prove these inequalities, we will check that certain derivatives approximately vanish.
    \begin{enumerate}[(a)]
        \item \textbf{($\widetilde{\bm{u}}$ is an approximate best response to $\widetilde{\bm{v}}$)} We study the minimum over $\bm{u} \in \R^d$ of
        \[
            Q''_{\bm{u}, \widetilde{\bm{v}}} = \frac{\norm{\bm{u}}}{\sqrt{d}}\langle \bm{h}, \widetilde{\bm{v}}\rangle - \frac{\norm{\widetilde{\bm{v}}}}{\sqrt{d}}\langle \bm{g}, \bm{u}\rangle + \psi(\bm{u}, \widetilde{\bm{v}}).
        \]
        If we can show that $Q''_{\bm{u}, \widetilde{\bm{v}}}$ is a $\lambda$-strongly convex function of $\bm{u}$, then minimizing $Q''_{\bm{u}, \widetilde{\bm{v}}}$ over $\bm{u} \in \R^d$ reduces to finding a point $\bm{u} \in \R^d$ at which the gradient  $\nabla_{\bm{u}} Q''_{\bm{u}, \widetilde{\bm{v}}}$ has small norm. To this end, observe that $\psi$ contains the term $\frac{\lambda}{2}\norm{\bm{u}}^2$, so $\psi(\bm{u}, \widetilde{\bm{v}})$ is $\lambda$-strongly convex in $\bm{u}$. The term $-\frac{\norm{\widetilde{\bm{v}}}}{\sqrt{d}}\langle \bm{g}, \bm{u}\rangle$ is a linear function of $\bm{u}$, and hence convex. Finally, the term $\frac{\norm{\bm{u}}}{\sqrt{d}}\langle\bm{h},\widetilde{\bm{v}}\rangle$ will be convex, as well, if we can show that $\langle \bm{h},\widetilde{\bm{v}}\rangle$ is positive with high probability.
        
        Quantities like $\langle \bm{h},\widetilde{\bm{v}}\rangle$, $\langle \bm{g}, \widetilde{\bm{u}} \rangle$, $\norm{\bm{\tv}}$, and $\norm{\bm{\tu}}$ can be easily computed in the limit by combining the definitions of $\widetilde{\bm{u}}$ and $\widetilde{\bm{v}}$ in this lemma's statement with the assumptions on $\bm{g}$, $\bm{h}$, $\bm{\beta}^\star$, $\bm{\varepsilon}^\star$, and $\bm{\xi}$. By the definition of pseudo-Lipschitz convergence (Definition~\ref{def:pl-convergence}), w.h.p.,
        \[
        \begin{array}{cc}
           \displaystyle \frac{\norm{\widetilde{\bm{u}}}}{\sqrt{d}} = \tau^\star\sqrt{(\sigma^\star_v)^2 + \lambda^2\kappa^2 + \nu^2} \pm n^{- \Omega(1)}, &
           \displaystyle \frac{1}{d}\langle \bm{g}, \widetilde{\bm{u}} \rangle = \tau^\star \sigma^\star_v \pm n^{-\Omega(1)}, \\
           \\
           \displaystyle \frac{\norm{\widetilde{\bm{v}}}}{\sqrt{d}} = \sigma^\star_v \pm n^{-\Omega(1)}, &
           \displaystyle \frac{1}{d}\langle \bm{h}, \widetilde{\bm{v}} \rangle = \frac{1}{\delta}\E\mathopen{}\left[\left[\frac{\sigma^\star Z + \varepsilon^\star_0}{1+\tau^\star}\right]_LZ\right] \pm n^{-\Omega(1)}.
        \end{array}
        \]
        Above, for convenience, we have introduced the abbreviation
        \[
            \sigma^\star_v = \sqrt{\frac{1}{\delta}\E\mathopen{}\left[\frac{\sigma^\star Z + \varepsilon^\star_0}{1 + \tau^\star}\right]_L^2}.
        \]
        Using equation \eqref{eq:huber-1} (see \cref{thm:main-huber-objective-perturbation}\ref{thm:main-huber-objective-utility}), we can simplify the expression for $\norm{\widetilde{\bm{u}}}$ to
        \[
            \frac{\norm{\widetilde{\bm{u}}}}{\sqrt{d}} = \sigma^\star \pm n^{- \Omega(1)}.
        \]
        Using Stein's lemma followed by equation \eqref{eq:huber-2}, we can simplify the expression for $\langle \bm{h}, \widetilde{\bm{v}}\rangle$ to
        \[
            \frac{1}{d}\langle \bm{h}, \widetilde{\bm{v}}\rangle = \frac{\sigma^\star}{\delta(1+\tau^\star)}\Pr\mathopen{}\left[-L < \frac{\sigma^\star Z + \varepsilon^\star_0}{1 + \tau^\star} < L\right] \pm n^{-\Omega(1)} = \left(\frac{1}{\tau^\star} - \lambda^\star\right)\sigma^\star \pm n^{-\Omega(1)}.
        \]
        By examining equation \eqref{eq:huber-2}, we see that $1/\tau^\star > \lambda^\star$, so the quantity above is positive with high probability, which establishes that $Q''_{\bm{u}, \widetilde{\bm{v}}}$ is $\lambda$-strongly convex in $\bm{u}$ with high probability. Therefore, all that remains is to evaluate its gradient at $\widetilde{\bm{u}}$. To this end, we calculate
        \[
            \nabla_{\bm{u}} Q''_{\bm{u}, \bm{v}} = \frac{1}{\sqrt{d}} \frac{\bm{u}}{\norm{\bm{u}}} \langle \bm{h}, \bm{v}\rangle - \frac{\norm{\bm{v}}}{\sqrt{d}} \bm{g} + \lambda\bm{u} + \lambda\bm{\beta}^\star + \nu\bm{\xi}.
        \]
        Evaluating at $(\widetilde{\bm{u}}, \widetilde{\bm{v}})$ and substituting our expressions for $\langle \bm{h}, \widetilde{\bm{v}}\rangle$, $\langle \bm{g}, \widetilde{\bm{u}} \rangle$, $\norm{\bm{\tv}}$, and $\norm{\bm{\tu}}$ yields
        \[
            \nabla_{\bm{u}} Q''_{\widetilde{\bm{u}}, \widetilde{\bm{v}}} = \left(\frac{1}{\tau^\star} \pm n^{-\Omega(1)}\right)\widetilde{\bm{u}} - (\sigma^\star_v \pm n^{-\Omega(1)})\bm{g} + \lambda\bm{\beta}^\star + \nu\bm{\xi}.
        \]
        By substituting the definition of $\widetilde{\bm{u}}$ in this lemma's statement and using the fact that for any constant $c > 0$, we have $\norm{\widetilde{\bm{u}}} + \norm{\bm{g}} + \norm{\bm{\beta}^\star} + \norm{\bm{\xi}} \le O(n^{\frac{1}{2} + c})$ w.h.p. (\cref{thm:uniform-bound-from-convergence}), we see that the above gradient clearly has $\ell^2$ norm
        \[
            \norm{\nabla_{\bm{u}} Q''_{\widetilde{\bm{u}}, \widetilde{\bm{v}}}} \le n^{\frac{1}{2} - \Omega(1)}.
        \]
        Thus, by $\lambda$-strong convexity, we have with high probability that
        \[
            Q''_{\widetilde{\bm{u}}, \widetilde{\bm{v}}} \le \min_{\bm{u} \in \R^d} Q''_{\bm{u}, \widetilde{\bm{v}}} + n^{1 - \Omega(1)}.
        \]
        \item \textbf{($\widetilde{\bm{v}}$ is an approximate best response to $\widetilde{\bm{u}}$)} We study the max over $\bm{v} \in [-L, +L]^n$ of
        \[
            Q''_{\widetilde{\bm{u}}, \bm{v}} = \frac{\norm{\widetilde{\bm{u}}}}{\sqrt{d}}\langle \bm{h}, \bm{v}\rangle - \frac{\norm{\bm{v}}}{\sqrt{d}}\langle \bm{g}, \widetilde{\bm{u}}\rangle + \psi(\widetilde{\bm{u}}, \bm{v}).
        \]
        Analogously to the previous part, we first verify $1$-strong concavity with respect to $\bm{v}$. To this end, observe that $\psi$ contains the term $-\frac{1}{2}\norm{\bm{v}}^2$, so $\psi(\widetilde{\bm{u}}, \bm{v})$ is $1$-strongly concave in $\bm{v}$. The term $\frac{\norm{\widetilde{\bm{u}}}}{\sqrt{d}}\langle\bm{h},\bm{v}\rangle$ is a linear function of $\bm{v}$, and hence concave. Finally, the term $-\frac{\norm{\bm{v}}}{\sqrt{d}}\langle\bm{g}, \widetilde{\bm{u}}\rangle$ is concave, as well, since we have already shown that $\langle \bm{g}, \widetilde{\bm{u}}\rangle$ is positive with high probability. Now, all that remains is to evaluate the gradient of $Q''_{\widetilde{\bm{u}}, \bm{v}}$ at $\widetilde{\bm{v}}$. To this end, we calculate
        \[
            \nabla_{\bm{v}}Q''_{\bm{u}, \bm{v}} = \frac{\norm{\bm{u}}}{\sqrt{d}}\bm{h} - \frac{1}{\sqrt{d}}\frac{\bm{v}}{\norm{\bm{v}}}\langle \bm{g}, \bm{u}\rangle + \bm{\varepsilon^\star}.
        \]
        Evaluating at $(\widetilde{\bm{u}}, \widetilde{\bm{v}})$ and substituting our expressions for $\langle \bm{h}, \widetilde{\bm{v}}\rangle$, $\langle \bm{g}, \bm{\tu} \rangle$, $\norm{\bm{\tv}}$, and $\norm{\bm{\tu}}$ yields
        \[
            \nabla_{\bm{v}} Q''_{\widetilde{\bm{u}}, \widetilde{\bm{v}}} = (\sigma^\star \pm n^{-\Omega(1)})\bm{h} + \bm{\varepsilon}^\star - (1 + \tau^\star \pm n^{-\Omega(1)})\widetilde{\bm{v}}.
        \]
        By substituting the definition of $\widetilde{\bm{v}}$ and using the fact that for any constant $c > 0$, we have $\norm{\widetilde{\bm{v}}} + \norm{\bm{h}} + \norm{\bm{\eps}^\star} \le O(n^{\frac{1}{2} + c})$ w.h.p. (\cref{thm:uniform-bound-from-convergence}), we see that there exists a vector of $\ell^2$ norm at most $O(n^{\frac{1}{2} + c})$ in the set \[\nabla_{\bm{v}} Q''_{\widetilde{\bm{u}}, \widetilde{\bm{v}}} - \partial \bm{1}_{[-L, +L]^n}(\widetilde{\bm{v}}).\]
        This set is a translation of the \emph{subdifferential set} $\partial \bm{1}_{[-L, +L]^n}(\bm{v}) = \prod_{i=1}^n \partial \bm{1}_{[-L, +L]}(v_i)$, where
        \[
            \partial \bm{1}_{[-L, +L]}(v_i) = \begin{cases}
                \{0\} & \text{if } \abs{v_i} < L, \\
                (-\infty, 0] & \text{if } v_i = -L, \\
                [0, +\infty) & \text{if } v_i = L, \\
                \varnothing & \text{if } \abs{v_i} > L. \\      
            \end{cases}
        \]
        By $1$-strong concavity on $[-L, +L]^n$, we have with high probability that
        \[
            Q''_{\widetilde{\bm{u}}, \widetilde{\bm{v}}} \ge \max_{\bm{v} \in [-L, +L]^n} Q''_{\widetilde{\bm{u}}, \bm{v}} - n^{1 - \Omega(1)}.
        \]
    \end{enumerate}
    To conclude the proof, we remark that plugging our estimates for $\langle \bm{h}, \widetilde{\bm{v}}\rangle$, $\langle \bm{g}, \widetilde{\bm{u}} \rangle$, $\norm{\bm{\tv}}$, and $\norm{\bm{\tu}}$ into the definition of $Q''_{\bm{u}, \bm{v}}$ similarly implies that there exists a constant $c^\star \in \R$ such that w.h.p.,
    \[
        \abs{Q''_{\widetilde{\bm{u}}, \widetilde{\bm{v}}} - c^\star n} \le n^{1 - \Omega(1)}.
    \]
\end{proof}

\subsubsection{Putting Steps 1, 2, and 3 Together}
\label{sec:huber-put-steps-together}

\begin{proof}(\cref{thm:main-huber-objective-perturbation})
    Note that part \ref{thm:logistic-objective-privacy} follows immediately from \cref{thm:our-objective-perturbation-privacy}, the observation that $\rho'' : \R \to [0, 1]$, and a change of variables in the case that $R \neq 1$. Therefore, we focus on part \ref{thm:main-huber-objective-utility}. By \cref{thm:huber-objective-aux}, there exist constants $c^\star \in \R$ and $c_{\mathrm{cgmt}} > 0$ such that w.h.p,
    \begin{equation}
    \label{eq:huber-aux-consequence}
        c^\star n - O(n^{1 - c_{\mathrm{cgmt}}}) \le \min_{\bm{u} \in \R^d} Q''_{\bm{u}, \widetilde{\bm{v}}} \le  Q''_{\widetilde{\bm{u}}, \widetilde{\bm{v}}} \le \max_{\bm{v} \in [-L, +L]^n} Q''_{\widetilde{\bm{u}}, \bm{v}} \le c^\star n + O(n^{1 - c_{\mathrm{cgmt}}}).
    \end{equation}
    By \cref{thm:uniform-bound-from-convergence} and \cref{thm:objective-uniform-bound}, for any arbitrarily small constant $c_{\mathrm{diam}} > 0$, there exists an upper bound $L_{\bm{u}}= O(n^{c_{\mathrm{diam}}})$ such that w.h.p.,
    \[
        \bm{\tu}, \bm{\hu} \in [-L_{\bm{u}}, +L_{\bm{u}}]^d.
    \]
    For brevity, set $\cS_{\bm{u}} = [-L_{\bm{u}}, +L_{\bm{u}}]^d$ and $\cS_{\bm{v}} = [-L, +L]^n$. Then, $\bm{\tu} \in \cS_{\bm{u}}$ implies that w.h.p.,
    \[
        \min_{\bm{u} \in \cS_{\bm{u}}} \max_{\bm{v} \in \cS_{\bm{v}}}\, Q''_{\bm{u}, \bm{v}} \le c^\star n + O(n^{1 - c_{\mathrm{cgmt}}}).
    \]
    By CGMT (\cref{thm:cgmt}), we similarly have that w.h.p.,
    \[
       \min_{\bm{u} \in \cS_{\bm{u}}} \max_{\bm{v} \in \cS_{\bm{v}}}\, Q'_{\bm{u}, \bm{v}} \le c^\star n + O(n^{1 - c_{\mathrm{cgmt}}}).
    \]
    By universality (\cref{thm:objective-huber-universality}), there is a constant $c_{\mathrm{univ}} \in (0, c_{\mathrm{cgmt}}]$ such that w.h.p.,
    \begin{equation}
    \label{eq:huber-objective-lower}
       \min_{\bm{u} \in \cS_{\bm{u}}} \max_{\bm{v} \in \cS_{\bm{v}}}\, Q_{\bm{u}, \bm{v}} \le c^\star n + O(n^{1 - c_{\mathrm{univ}}}).
    \end{equation}
    By the definition of $\widetilde{\bm{u}}$ in the statement of \cref{thm:huber-objective-aux},
    along with our assumption that $\bm{\beta}^\star \rightsquigarrow \beta^\star_0$,
    \begin{equation*}
        (\bm{\beta}^\star, \bm{\xi}, \widetilde{\bm{u}}) \rightsquigarrow \left(\beta_0^\star, \; \xi_0, \; \tau^\star{\left(\sqrt{\frac{1}{\delta} \E\mathopen{}\left[\frac{\sigma^\star Z + \varepsilon^\star_0}{1 + \tau^\star}\right]_L^2} Z - \lambda \beta^\star_0 - \nu \xi_0\right)}\right).
    \end{equation*}
    For brevity, let $u_0 \in \R$ denote the third random variable in the above triple. Then, the above assertion is that $(\bm{\beta}^\star, \bm{\xi}, \widetilde{\bm{u}}) \rightsquigarrow (\beta_0^\star, \xi_0, u_0)$, and we want to show that $(\bm{\beta}^\star, \bm{\xi}, \bm{\hb} - \bm{\beta}^\star) \rightsquigarrow (\beta_0^\star, \xi_0, u_0)$, as well. To this end, fix an order-$k$ pseudo-Lipschitz function $f : \R^3 \to \R$. By Definition~\ref{def:pl-convergence}, the above statement means that there exists a constant $c_{\mathrm{pL}} \in (0, c_{\mathrm{univ}}]$ such that w.h.p., 
    \begin{equation}
    \label{eq:u-tilde-convergence}
        \Abs{\bar{f}(\bm{\beta}^\star, \bm{\xi}, \bm{\tu}) -\E[f(\beta^\star_0, \xi_0, u_0)]} \le O(n^{-c_{\mathrm{pL}}}),
    \end{equation}
    where we have defined the function $\bar{f} : \R^{d \times 3} \to \R$ by
    \(
        \bar{f}(\bm{\beta}^\star, \bm{\xi}, \bm{u}) = \frac{1}{d}\sum_{j=1}^d f(\beta^\star_j, \xi_j, u_j).
    \)
    Next, let $c_{\mathrm{slack}} > 0$ be any strictly positive constant satisfying the strict inequality $c_{\mathrm{slack}} < c_{\mathrm{pL}}/2 - (k-1)c_{\mathrm{diam}}$. Note that such a choice of $c_{\mathrm{slack}}$ is always possible because, as we noted earlier, we can take the constant $c_{\mathrm{diam}}$ to be arbitrarily small. Next, define the open set
    \begin{equation}
    \label{eq:huber-objective-constraint-set}
        \cT_{\bm{u}} = \left\{\bm{u} \in \R^d : \Abs{\bar{f}(\bm{\beta}^\star, \bm{\xi}, \bm{u}) - \E[f(\beta^\star_0, \xi_0, u_0)]} < n^{-c_{\mathrm{slack}}}\right\}.
    \end{equation}
    Observe that by \cref{thm:huber-objective-aux}, the auxiliary objective function is $\lambda$-strongly convex in $\bm{u}$ when $\bm{v} = \bm{\tv}$. Also, since $\cS_{\bm{u}} = [-L_{\bm{u}}, +L_{\bm{u}}]^d$, the order-$k$ pseudo-Lipschitz function $\bar{f}$ is in fact Lipschitz continuous in $\bm{u}$ over $\cS_{\bm{u}}$ with Lipschitz constant
    \[
        L_{\bar{f}} \le \frac{1}{d} \cdot \sqrt{d} \cdot (1 + 2L_{\bm{u}})^{k-1} = O(n^{-\frac{1}{2} + (k-1)c_{\mathrm{diam}}}).
    \]
    Together, $\Omega(1)$-strong convexity, $O(n^{-\frac{1}{2} + (k-1)c_{\mathrm{diam}}})$-Lipschitzness, \eqref{eq:huber-aux-consequence}, \eqref{eq:u-tilde-convergence}, and \eqref{eq:huber-objective-constraint-set} imply that if we excise $\cT_{\bm{u}}$ from $\cS_{\bm{u}}$ to form a new compact constraint set $\cS_{\bm{u}} \setminus \cT_{\bm{u}}$, then w.h.p.,
    \[
        \min_{\bm{u} \in \cS_{\bm{u}} \setminus \cT_{\bm{u}}} \max_{\bm{v} \in \cS_{\bm{v}}}\, Q''_{\bm{u}, \bm{v}} \ge c^\star n + \Omega{\left(\left(\frac{n^{-c_{\mathrm{slack}}}}{n^{-\frac{1}{2} + (k-1)c_{\mathrm{diam}}}}\right)^2\right)} - O(n^{1 - c_{\mathrm{pL}}}).
    \]
    The $\Omega(\cdot)$ term on the right side simplifies to $\Omega(n^{1-c_{\mathrm{total}}})$ where $c_{\mathrm{total}} = 2c_{\mathrm{slack}} + 2(k-1)c_{\mathrm{diam}}$. By our choice of $c_{\mathrm{slack}}$, we have that $\Omega(n^{1 - c_{\mathrm{total}}})$ dominates $O(n^{1 - c_{\mathrm{pL}}})$, so w.h.p.,
    \[
        \min_{\bm{u} \in \cS_{\bm{u}} \setminus \cT_{\bm{u}}} \max_{\bm{v} \in \cS_{\bm{v}}}\, Q''_{\bm{u}, \bm{v}} \ge c^\star n + \Omega(n^{1 - c_{\mathrm{total}}}).
    \]
    Applying CGMT (\cref{thm:cgmt}) and universality (\cref{thm:objective-huber-universality}) as before (note that $c_{\mathrm{total}} < c_{\mathrm{cgmt}}$ and $c_{\mathrm{total}} < c_{\mathrm{univ}}$), we have w.h.p. that
    \begin{equation}
    \label{eq:huber-objective-upper}
       \min_{\bm{u} \in \cS_{\bm{u}} \setminus \cT_{\bm{u}}} \max_{\bm{v} \in \cS_{\bm{v}}}\, Q_{\bm{u}, \bm{v}} \ge c^\star n + \Omega(n^{1 - c_{\mathrm{total}}}).
    \end{equation}
    At this point, \eqref{eq:huber-objective-lower} and \eqref{eq:huber-objective-upper} imply that w.h.p.,
    \[
        \min_{\bm{u} \in \cS_{\bm{u}} \setminus \cT_{\bm{u}}} \max_{\bm{v} \in \cS_{\bm{v}}}\, Q_{\bm{u}, \bm{v}} > \min_{\bm{u} \in \cS_{\bm{u}}} \max_{\bm{v} \in \cS_{\bm{v}}}\, Q_{\bm{u}, \bm{v}}.
    \]
    (note the strict inequality). Thus, the minimizer, which by \cref{thm:objective-huber-legendre} is $\bm{\hu} = \widehat{\bm{\beta}} - \bm{\beta}^\star \in \cS_{\bm{u}}$, w.h.p. lies inside $\cT_{\bm{u}}$. Unpacking the definition of $\cT_{\bm{u}}$ in \eqref{eq:huber-objective-constraint-set}, we conclude that
    \[
        (\bm{\beta}^\star, \bm{\xi}, \widehat{\bm{\beta}} - \bm{\beta}^\star) \rightsquigarrow (\beta^\star_0, \xi_0, u_0).
    \]
    This completes the characterization of the estimation error of $\bm{\hb}$ in part \ref{thm:main-huber-objective-utility}. For the truncated residual error, we now carry out an entirely analogous dual argument with the roles of $\bm{u}$ and $\bm{v}$ exchanged. As before, we start by noting that by CGMT (\cref{thm:cgmt}) and universality (\cref{thm:objective-huber-universality}), we have w.h.p. that
    \[
        \max_{\bm{v} \in \cS_{\bm{v}}} \min_{\bm{u} \in \cS_{\bm{u}}} Q_{\bm{u}, \bm{v}} \ge c^\star n - O(n^{1 - c_{\mathrm{univ}}}).
    \]
    By the definition of $\bm{\tv}$ in the statement of \cref{thm:huber-objective-aux}, along with our assumption that $\bm{\eps}^\star \rightsquigarrow \eps^\star_0$,
    \[
        (\bm{\eps}^\star, \bm{\tv}) \rightsquigarrow \left(\eps^\star_0, \, \left[\frac{\sigma^\star Z + \eps^\star_0}{1 + \tau^\star}\right]_L\right).
    \]
    For brevity, let $v_0 \in \R$ denote the second random variable in the above pair. Then, the above assertion is that $(\bm{\eps}^\star, \bm{\tv}) \rightsquigarrow (\eps^{\star}_0, v_0)$, and we want to show that $(\bm{\eps}^\star, [\bm{y} - \bm{X}\bm{\hb}]_L) \rightsquigarrow (\eps^\star_0, v_0)$ as well. To this end, as before, we take any order-$k$ pseudo-Lipschitz function $f : \R^2 \to \R$, and excise from $\cS_{\bm{v}}$ the open set $\cT_{\bm{v}}$ such that $\bm{v} \in \cT_{\bm{v}}$ iff the average value of $f$ over the coordinates of $(\bm{\eps}^{\star}, \bm{v})$ differs from the expected value of $f$ over the randomness of $(\eps^\star_0, v_0)$ by strictly less than $n^{-c'_{\mathrm{slack}}}$, for a sufficiently small constant $c'_{\mathrm{slack}} > 0$. Doing so yields (again by CGMT, universality, pseudo-Lipschitzness, and the strong concavity afforded by \cref{thm:huber-objective-aux}), w.h.p.,
    \[
        \max_{\bm{v} \in \cS_{\bm{v}} \setminus \cT_{\bm{v}}} \min_{\bm{u} \in \cS_{\bm{u}}} Q_{\bm{u}, \bm{v}} \le c^\star n - \Omega(n^{1 - c'_{\mathrm{total}}}),
    \]
    for an appropriate positive constant $c'_{\mathrm{total}} < c_{\mathrm{univ}}$. At this point, we have that w.h.p.,
    \[
        \max_{\bm{v} \in \cS_{\bm{v}} \setminus \cT_{\bm{v}}} \min_{\bm{u} \in \cS_{\bm{u}}} Q_{\bm{u}, \bm{v}} < \max_{\bm{v} \in \cS_{\bm{v}}} \min_{\bm{u} \in \cS_{\bm{u}}} Q_{\bm{u}, \bm{v}}.
    \]
    Thus, the maximizer, which by \cref{thm:objective-huber-legendre} is $\bm{\hv} = [\bm{y} - \bm{X}\bm{\hb}]_L \in \cS_{\bm{v}}$, w.h.p. lies inside $\cT_{\bm{v}}$, so
    \[
        (\bm{\varepsilon}^\star, \, [\bm{y} - \bm{X}\widehat{\bm{\beta}}]_L) \rightsquigarrow \left(\varepsilon^\star_0, \; \left[\frac{\sigma^\star Z + \varepsilon^\star_0}{1 + \tau^\star}\right]_L\right).
    \]
\end{proof}

\subsection{Additional Proofs}
\label{sec:pl-proofs}

\begin{proof}(\cref{thm:pl-from-randomness})
    If $f$ is order-$k$ pseudo-Lipschitz, then $\abs{f(x)} \le L(1 + \abs{x}^k)$ for some $L$. Thus, $f(x_0)$ has finite mean and variance as long as $x_0$ has finite $2k$\textsuperscript{th} moment. By Chebyshev's inequality, independence of $x_1, \ldots, x_m$, and the fact that $m = \Theta(n)$,
    \[
        \Pr\left[\Abs{\frac{1}{m}\sum_{i=1}^m f(x_i) - \E[f(x_0)]} \ge n^{-\frac{1}{3}}\right] \le \frac{\mathrm{Var}(f(x_0))}{m(n^{-1/3})^2} = O(n^{-\frac{1}{3}}).
    \]
\end{proof}

\begin{proof}(\cref{thm:uniform-bound-from-convergence})
    For simplicity, suppose first that $\bm{x}$ is deterministic. Then, by Definition~\ref{def:pl-convergence},
    \begin{equation}
    \label{eq:pl-convergence-uniform-bound-inequality}
        \frac{1}{m}\sum_{i=1}^m \abs{x_i}^k \le \E\abs{x_0}^k + n^{-\Omega(1)}.
    \end{equation}
    Consequently, for any threshold $t > 0$,
    \[
        \frac{1}{m}\sum_{i=1}^m \bm{1}\left[\abs{x_i}^k > t\right] \le \frac{\E\abs{x_0}^k + n^{-\Omega(1)}}{t}.
    \]
    Substituting $t = m^2$ and rearranging terms,
    \[
        \bm{1}\left[\max_{i \in [m]}\,\abs{x_i} \le m^{\frac{2}{k}}\right] \ge 1 - \frac{\E\abs{x_0}^k + n^{-\Omega(1)}}{m}.
    \]
    Since $m = \Theta(n)$, the right hand side is strictly positive for large enough $n$. Therefore,
    \[
        \max_{i \in [m]}\, \abs{x_i} = O_k(n^{\frac{2}{k}}).
    \]
    Since this holds for all $k \in \N$, we conclude that
    \[
        \norm{\bm{x}}_\infty = \max_{i \in [m]}\, \abs{x_i} = n^{o(1)}.
    \]
    In the case that $\bm{x}$ is random, then Definition~\ref{def:pl-convergence} instead implies that \eqref{eq:pl-convergence-uniform-bound-inequality} holds \emph{with high probability}, and we conclude by the same argument that for any $c > 0$, w.h.p., $\norm{\bm{x}}_\infty = O(n^c)$.
\end{proof}

\section{Utility of Objective Perturbation for Logistic Regression}
\label{sec:objective-perturbation-logistic}

In this section, we study the privacy-utility tradeoff for objective perturbation, applied to the problem of \emph{logistic} regression. The version of the algorithm that we consider is still that of \cref{alg:objective-perturbation}, but we will instantiate it with a different loss function $\ell$. Before stating the main result of this section (\cref{thm:logistic-objective-perturbation}), we briefly review the setup for logistic regression and comment on some slight differences in the notion of convergence that will arise in the theorem's conclusion.

\paragraph{Logistic Regression Model.} We are given a dataset $(\bm{X}, \bm{y})$ of $n$ samples. In the worst case, $(\bm{x}_i, y_i)$ are arbitrary points in $\R^{d + 1}$ with $\norm{\bm{x}_i} \le R$ and $y \in \{0, 1\}$, but in the average case, there is a ground-truth coefficient vector $\bm{\beta}^\star \in \R^d$ independent of $\bm{X}$ such that
\[
    \bm{y}|(\bm{X}, \bm{\beta}^\star) \sim \mathrm{Bernoulli}(\rho'(\bm{X}\bm{\beta}^\star)).
\]
Here, $\rho(t) = \log(1 + e^t)$, so that $\rho'(t) = 1/(1 + e^{-x})$ is the well-known \emph{sigmoid} function. For $\bm{t} \in \R^n$, we define $\rho(\bm{t}) = \sum_{i=1}^n \rho(t_i) \in \R$ and $\rho'(\bm{t}) = (\rho'(t_1), \ldots, \rho'(t_n)) \in \R^n$. The logistic loss of a coefficient vector $\bm{\beta}$ on a data point $(\bm{x}, y)$ is \[\rho(\langle \bm{x}, \bm{\beta}\rangle) - y\langle \bm{x}, \bm{\beta}\rangle =\rho(-(2y-1)\langle \bm{x}, \bm{\beta}\rangle),\] which simplifies to $\rho(\langle \bm{x}, \bm{\beta}\rangle)$ if $y = 0$ or $\rho(-\langle \bm{x}, \bm{\beta}\rangle)$ if $y = 1$.

\begin{figure}
    \centering
    \includegraphics{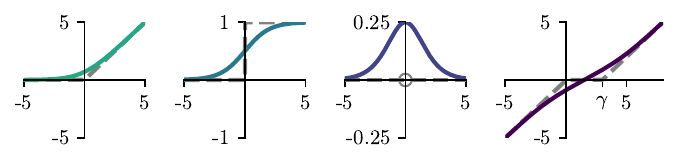}
    \caption{The functions $\rho$, $\rho'$, $\rho''$, and $\prox_{\gamma \rho}$ with $\gamma = 3$.}
    \label{fig:logistic-loss}
\end{figure}

\paragraph{Notion of Convergence.} Since the universality results applicable to logistic regression in the existing statistics literature currently have slower convergence rates than those for robust linear regression, it will be convenient to introduce the following variant $\dashrightarrow$ of the $\rightsquigarrow$ notation we introduced in \cref{sec:objective-perturbation-linear}.

\begin{definition}[Slow Pseudo-Lipschitz Convergence]
\label{def:slow-pl-convergence}
    Consider a random variable $x_0 \in \R$ with finite moments of all orders and a random vector $\bm{x} \in \R^m$ with $m = \Theta(n)$. We write $\bm{x} \dashrightarrow x_0$, if for all $f : \R \to \R$ that are either bounded and continuous, or pseudo-Lipschitz of order $2$,
    \begin{equation*}
        \exists c > 0 \quad\text{such that}\quad \Pr\mathopen{}\left[\Abs{\frac{1}{m}\sum_{i=1}^m f(x_i) - \E[f(x_0)]} \le e^{-(\log n)^{c}}\right] \ge 1 - O(n^{-c}).
    \end{equation*}
    Define $\bm{X} \dashrightarrow \bm{x}_0$ analogously for $\bm{x}_0 \in \R^\ell$, as in Definition~\ref{def:pl-convergence}.
\end{definition}

We are now ready to state and prove the main theorem of this section.

\begin{theorem}
\label{thm:logistic-objective-perturbation}
    Let $\bm{\hb}$ denote the output of \cref{alg:objective-perturbation} with parameters $R, \lambda, \nu > 0$ and instantiated with the logistic loss function: \[\ell(\bm{\beta}; (\bm{x}, y)) = \rho(\langle \bm{x}, \bm{\beta}\rangle) - y\langle\bm{x}, \bm{\beta}\rangle.\]
    \begin{enumerate}[(a)]
        \item \emph{\textbf{(Privacy)}}
        \label{thm:logistic-objective-privacy}
        $\bm{\hb}$ satisfies $(\pe, \pd)$-differential privacy for any $\pe \ge 0$ and
        \[
            \pd = \begin{cases}
                2 \cdot \mathrm{HockeyStick}(\tpe, \frac{LR}{\nu}) &\text{if }\hpe \ge 0,\\
                (1 - e^{\hpe}) + 2e^{\hpe}\cdot \mathrm{HockeyStick}\left(\frac{L^2R^2}{2\nu^2}, \frac{LR}{\nu}\right) &\text{otherwise,}
            \end{cases}
        \]
        where we set $\tpe = \pe - \log(1 + R^2/4\lambda)$ and $\hpe = \tpe - L^2R^2/2\nu^2$.
        \item \emph{\textbf{(Utility)}}
        \label{thm:logistic-objective-utility}
        Suppose the following hold for some $\bm{\beta}^\star \in \R^d$ as $n \to \infty$ and $d/n \to \delta$:
        \begin{enumerate}[(i)]
            \item $\bm{X} \in B_R(\bm{0})^n \subseteq \R^{n \times d}$ follows a subgaussian design and $\bm{y} \sim \mathrm{Bernoulli}(\rho'(\bm{X}\bm{\beta}^\star))$.
            \item There exists a random variable $\beta^\star_0 \in \R$ such that $\bm{\beta}^\star \rightsquigarrow \beta^\star_0$.
        \end{enumerate}
        
        Suppose there exist $\sigma^\star, \alpha^\star, \gamma^\star > 0$ solving the following system of three scalar equations in three variables $(\sigma, \alpha, \gamma)$, which we write in terms of dummy variables $Z_1, Z_2 \iid \mathcal{N}(0, 1)$ and $\kappa^2 = \E(\beta^\star_0)^2$ as
        \begin{subequations}
        \begin{align}
            \sigma^2 &= \gamma^2\left(\frac{1}{\delta}\E\mathopen{}\left[2\rho'(-\kappa Z_1)\rho'\bigl(\prox_{\gamma\rho}(\kappa\alpha Z_1+\sigma Z_2) \bigr)^2\right] + \nu^2\right), \label{eq:logistic-1} \\
            \alpha &= -\frac{1}{\delta}\E[ 2\rho''(-\kappa Z_1)\prox_{\gamma \rho}\big(\kappa \alpha Z_1+\sigma Z_2\big)], \label{eq:logistic-2} \\
            \gamma &= \frac{1}{\lambda \delta}\mathopen{}\left(\delta - 1 + \E\mathopen{}\left[\frac{2\rho'(-\kappa Z_1)}{1+\gamma \rho''\big(\prox_{\gamma\rho}(\kappa\alpha Z_1 + \sigma Z_2)\big)}\right]\right). \label{eq:logistic-3}
        \end{align}
        \end{subequations}
        The estimation error $\widehat{\bm{\beta}} - \bm{\beta}^\star$ satisfies, for $Z,\xi_0 \iid \mathcal{N}(0, 1)$,
        \[
            (\bm{\beta}^\star, \, \bm{\xi}, \, \widehat{\bm{\beta}}) \dashrightarrow \left(\beta^\star_0, \, \xi_0, \, \alpha^\star \beta^\star_0 + \sqrt{(\sigma^\star)^2 - (\gamma^\star \nu)^2} Z - \gamma^\star\nu \xi_0\right),
        \]
        The difference $\rho'(\bm{X}\bm{\beta}^\star) - \rho'(\bm{X} \widehat{\bm{\beta}})$ satisfies, for $Z_1, Z_2 \iid \mathcal{N}(0, 1)$ and $y_0|Z_1 \sim \mathrm{Bernoulli}(\rho'(\kappa Z_1))$,
        \[
            (\bm{X}\bm{\beta}^\star, \, \rho'(\bm{X} \widehat{\bm{\beta}})) \dashrightarrow \left(\kappa Z_1, \; \alpha^\star \kappa Z_1 + \sigma^\star Z_2 + y_0 - \prox_{\gamma^\star \rho}(\alpha^\star \kappa Z_1 + \sigma^\star Z_2 + \gamma^\star y_0)\right).
        \]

    \end{enumerate}
\end{theorem}

As in the case of robust linear regression (\cref{thm:main-huber-objective-perturbation}), part \ref{thm:logistic-objective-privacy} follows from the results of \cref{sec:objective-perturbation-privacy}, and combining \cref{thm:logistic-objective-perturbation}\ref{thm:logistic-objective-utility} with Definition~\ref{def:pl-convergence} yields several interesting corollaries. For example, for the \emph{bias} of $\widehat{\bm{\beta}}$, we see that w.h.p.,
\begin{align*}
    \frac{1}{d}\langle \bm{\hb}, \bm{\beta}^\star \rangle &= \E[(\alpha^\star \beta^\star_0 + \sqrt{(\sigma^\star)^2 - (\gamma^\star \nu)^2}Z + \gamma^\star\nu\xi_0)\cdot \beta^\star_0] \pm e^{-(\log n)^{\Omega(1)}} \\
    &= \alpha^\star\kappa^2 \pm e^{-(\log n)^{\Omega(1)}}.
\end{align*}
Therefore, we deduce that $\bm{\hb}$ is positively correlated with $\bm{\beta}^\star$, which is intuitively what we would expect, and this correlation is captured by $\alpha^\star$. Similarly, for the \emph{variance} of $\bm{\hb}$, we see that w.h.p.,
\begin{align*}
    \frac{1}{d}\norm{\bm{\hb} - \alpha^\star \bm{\beta}^\star}^2 &= \E\mathopen{}\left(\sqrt{(\sigma^\star)^2 - (\gamma^\star \nu)^2}Z + \gamma^\star\nu\xi_0\right)^2 \pm e^{-(\log n)^{\Omega(1)}} \\
    &= (\sigma^\star)^2 \pm e^{-(\log n)^{\Omega(1)}}.
\end{align*}
In total, we see that the mean squared error of $\bm{\hb}$ is, w.h.p., \[\frac{1}{d} \norm{\bm{\hb} - \bm{\beta}^\star}^2 = (1 - \alpha^\star)^2\kappa^2 + (\sigma^\star)^2 \pm e^{-(\log n)^{\Omega(1)}}.\] Similar calculations can be carried out for the difference $\rho'(\bm{X}\bm{\beta}^\star) - \rho'(\bm{X}\widehat{\bm{\beta}})$. We validate these corollaries against simulated data in Figure \ref{fig:logistic-objective-estimation}. For more commentary on how to interpret \cref{thm:logistic-objective-perturbation} and Figure \ref{fig:logistic-objective-estimation}, we refer the reader to the discussion surrounding (the very similarly-worded) \cref{thm:main-huber-objective-perturbation} and Figure \ref{fig:huber-objective-estimation} on robust linear regression, in \cref{sec:objective-perturbation-linear}.

\begin{figure}
    \centering
    \includegraphics[width=0.495\linewidth]{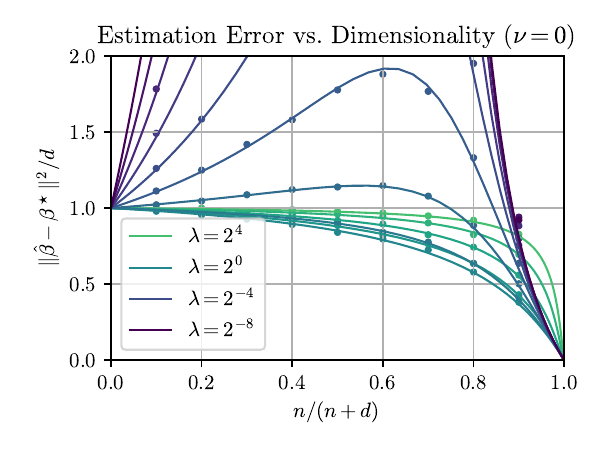}
    \includegraphics[width=0.495\linewidth]{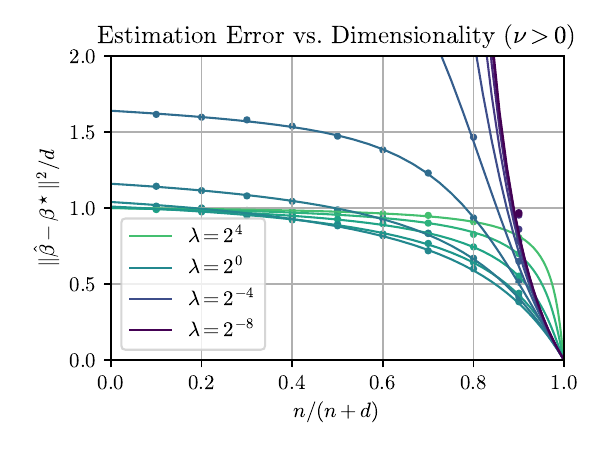}
    \includegraphics[width=0.495\linewidth]{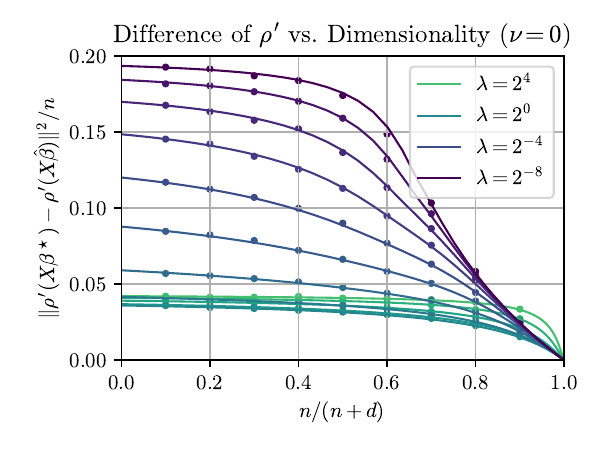}
    \includegraphics[width=0.495\linewidth]{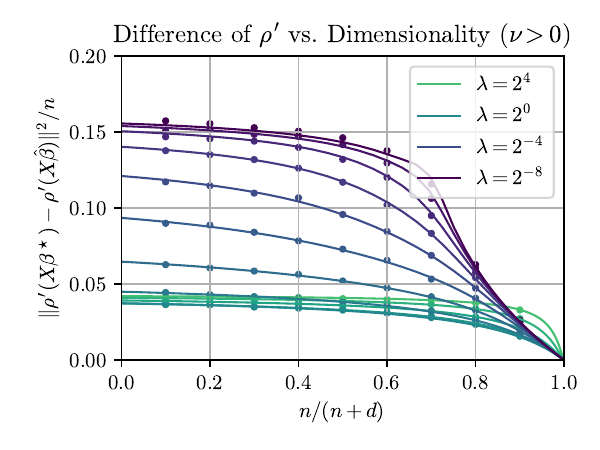}
    \caption{\cref{thm:logistic-objective-perturbation}'s predictions of the error of \cref{alg:objective-perturbation} with logistic loss. \emph{Estimation error} refers to $\frac{1}{d}\norm{\bm{\hb} - \bm{\beta}^\star}^2$. \emph{Difference of $\rho'$} refers to $\frac{1}{n}\norm{\rho'(\bm{X}\bm{\beta}^\star) - \rho'(\bm{X}\bm{\hb})}^2$, which is related to the residual vector $\bm{y} - \rho'(\bm{X}\bm{\hb})$ since $\bm{y} \sim \mathrm{Bernoulli}(\rho'(\bm{X}\bm{\beta}^\star))$. Curves correspond to theoretical predictions, and dots correspond to the mean over $1000$ simulations of the algorithm on synthetic data with $n \times d = 1000$. In the left plots, the perturbation magnitude is $\nu = 0$, but in the right plots, $\nu = 1/5$. In all plots, the signal strength is $\kappa = 1$, and we consider $\bm{\beta}^\star \sim \mathcal{N}(\bm{0}, \kappa^2\bm{I}_d)$, along with $\bm{X} \sim \frac{1}{\sqrt{d}} \mathrm{Uniform}(\{-1,+1\}^{n \times d})$ and $\bm{y} \sim \mathrm{Bernoulli}(\rho'(\bm{X}\bm{\beta}^\star))$.}
    \label{fig:logistic-objective-estimation}
\end{figure}

\subsection{Heuristic Derivation}

Before we begin the utility proof, we give a non-rigorous, heuristic derivation to motivate the three equations in the statement of \cref{thm:logistic-objective-perturbation}\ref{thm:logistic-objective-utility}. The derivation is based on Theorem 1 of \cite{salehi2019impact}, which determines the asymptotic behavior of $\widehat{\bm{\beta}}$ for regularized logistic regression, provided that $\bm{\beta}^\star$ has i.i.d. coordinates and the regularizer is a deterministic function of the form $\sum_{j=1}^d f(\beta_j)$.

In order to apply Theorem 1 of \cite{salehi2019impact}, one first assumes that the following system of six equations in six variables has a unique solution $(\alpha^\star, \sigma^\star, \gamma^\star, \theta^\star, \tau^\star, r^\star)$. These equations are stated in terms of $Z, Z_1, Z_2 \iid \mathcal{N}(0, 1)$ and $\bm{\beta}^\star \iid \beta^\star_0$, as follows.\footnote{We have replaced $\delta$ with $1/\delta$ and $\lambda$ with $\lambda \delta$ to account for differences in notation and scaling in \cite{salehi2019impact}.}

\begin{subequations}
\begin{align}
    \kappa^2 \alpha &= \E\mathopen{}\left[\beta^\star_0\; \prox_{\lambda\delta\sigma\tau f}\mathopen{}\left(\sigma\tau\left(\theta \beta^\star_0 +r\sqrt{\delta}Z\right)\right)\right]\mathclose{}, \label{eq:salehi-1} \\
    \gamma &= \frac{\sqrt{\delta}}{r} \E\mathopen{}\left[Z\,\prox_{\lambda\delta\sigma\tau f}\mathopen{}\left(\sigma\tau\left(\theta \beta^\star_0+r\sqrt{\delta}Z\right)\right)\right]\mathclose{}, \label{eq:salehi-2} \\
    \kappa^2\alpha^2 + \sigma^2 &= \E\mathopen{}\left[\prox_{\lambda\delta\sigma\tau f}\mathopen{}\left(\sigma\tau\left(\theta \beta^\star_0+r\sqrt{\delta}Z\right)\right)^2\right]\mathclose{}, \label{eq:salehi-3} \\
    \gamma^2 &= \frac{1}{r^2} \E\mathopen{}\left[2\rho'(-\kappa Z_1)\left(\kappa\alpha Z_1+\sigma Z_2 - \prox_{\gamma\rho}(\kappa\alpha Z_1+\sigma Z_2) \right)^2\right]\mathclose{}, \label{eq:salehi-4} \\
    \theta\gamma &= -\E\mathopen{}\left[ 2\rho''(-\kappa Z_1)\;\prox_{\gamma \rho}\big(\kappa \alpha Z_1+\sigma Z_2\big)\right]\mathclose{}, \label{eq:salehi-5} \\
    1-\frac{\gamma}{\sigma\tau} &= \E\mathopen{}\left[\frac{2\rho'(-\kappa Z_1)}{1+\gamma \rho''\big(\prox_{\gamma\rho}(\kappa\alpha Z_1 + \sigma Z_2)\big)}\right]\mathclose{}.\label{eq:salehi-6}
\end{align}
\end{subequations}

The conclusion of \cite{salehi2019impact}, informally speaking, is that the coordinates of $\widehat{\bm{\beta}}$ behave like
\begin{equation}
\label{eq:derivation-conclusion}
    \prox_{\lambda\delta \sigma^\star \tau^\star f}\left(\sigma^\star\tau^\star\left(\theta^\star \beta^\star_0 + r^\star\sqrt{\delta} Z\right)\right).
\end{equation}
Next, observe that logistic regression with objective perturbation can be recast as regularized logistic regression with a ``randomized regularizer,'' as follows. First, sample $\xi_1, \ldots, \xi_d \iid \xi_0$ where $\xi_0 \sim \mathcal{N}(0, 1)$ and define the functions $f_0, \ldots, f_d : \R \to \R$ as
\[
    f_j(\beta) = \frac{\lambda}{2}\beta^2 + \nu \xi_j \beta.
\]
Then, the objective perturbation algorithm's output $\widehat{\bm{\beta}}$ can be written as
\[
    \widehat{\bm{\beta}} = \argmin_{\bm{\beta} \in \R^d}\, \rho(\bm{X} \bm{\beta}) - \langle \bm{y}, \bm{X}\bm{\beta}\rangle + \sum_{j=1}^d f_j(\beta_j).
\]
Therefore, even though Theorem 1 of \cite{salehi2019impact} requires the coordinate-wise application of a single, deterministic regularizer $f$, and the above expression uses $d$ distinct, ``randomized regularizers'' $f_1, \ldots, f_d$, one might guess that substituting $f = f_0$ into the above system of six equations and taking the expectation over the randomness of $f_0$ might nevertheless lead to the equations in \cref{thm:logistic-objective-perturbation}\ref{thm:logistic-objective-utility}, which we will later show correctly describe $\widehat{\bm{\beta}}$. Indeed, this is the case, which we now show. First, observe that
\begin{equation}
\label{eq:prox}
    \prox_{t f_0}(x) = \frac{1}{1+t} x + \frac{t}{1+t}\left(-\frac{\nu}{\lambda}\xi_0\right) = \frac{x - t \cdot \frac{\nu}{\lambda} \xi_0}{1 + t}.
\end{equation}
Substituting \eqref{eq:prox} into \eqref{eq:salehi-2} yields
\begin{equation*}
    \gamma = \frac{\sqrt{\delta}}{r} \E\left[Z \cdot \frac{\sigma\tau(\theta\beta^\star_0 +r\sqrt{\delta} Z) - \delta\sigma\tau\nu \xi_0}{1+\lambda\delta\sigma\tau}\right].
\end{equation*}
If the random variables inside the expectation ($Z$, $\beta^\star_0$, and $\xi_0$) are all independent, we have $\E[Z\beta^\star_0] = \E[Z\xi_0] = 0$. We also have $\E[Z^2] = 1$. Thus, the above equation simplifies to
\begin{equation}
    \label{eq:tau}
    \gamma = \frac{\delta \sigma\tau}{1+\lambda\delta\sigma\tau}.
\end{equation}
Next, substituting \eqref{eq:prox} into \eqref{eq:salehi-1} yields
\begin{equation*}
    \kappa^2\alpha = \E\left[\beta^\star_0 \cdot \frac{\sigma\tau(\theta\beta^\star_0 + r\sqrt{\delta}Z) - \delta\sigma\tau\nu\xi_0}{1+\lambda\delta\sigma\tau}\right].
\end{equation*}
Once again, $\E[\beta^\star_0 Z] = \E[\beta^\star_0 \xi_0] = 0$ and we know $\E(\beta^\star_0)^2 = \kappa^2$, so the above equation simplifies to
\begin{equation*}
    \kappa^2 \alpha = \frac{\sigma\tau\theta \kappa^2}{1+\lambda\delta\sigma\tau}.
\end{equation*}
Using \eqref{eq:tau} to eliminate $\tau$ further simplifies this to
\begin{equation}
\label{eq:theta}
    \alpha = \frac{\gamma\theta}{\delta}.
\end{equation}
Finally, substituting \eqref{eq:prox} into \eqref{eq:salehi-3} yields
\[
    \kappa^2\alpha^2 + \sigma^2 = \E\left[\left(\frac{\sigma\tau(\theta\beta+r\sqrt{\delta}Z) - \delta\sigma\tau\nu\xi_0}{1+\lambda\delta\sigma\tau}\right)^2\right].
\]
Again using the fact that pairwise products of $Z$, $\beta^\star_0$, and $\xi_0$ vanish in expectation, this becomes
\begin{equation*}
    \kappa^2\alpha^2 + \sigma^2 = \left(\frac{\sigma\tau}{1+\lambda\delta\sigma\tau}\right)^2 \left(\theta^2\E(\beta^\star_0)^2 + r^2\delta\E[Z^2] + (\nu\delta)^2\E[\xi_0^2]\right).
\end{equation*}
Eliminating $\tau$ with \eqref{eq:tau} and $\theta$ with \eqref{eq:theta} yields
\[
    \kappa^2\alpha^2 + \sigma^2 = \left(\frac{\gamma}{\delta}\right)^2\left(\left(\frac{\alpha\delta}{\gamma}\right)^2\cdot \E(\beta^\star_0)^2 + \frac{r^2}{\delta}\E[Z^2] + (\nu\delta)^2\E[\xi_0^2]\right)
\]
and substituting $\E(\beta^\star_0)^2 = \kappa^2$ and $\E Z^2 = \E\xi_0^2 = 1$ finally yields
\begin{equation}
\label{eq:r}
    \sigma^2 = \frac{\gamma^2r^2}{\delta} + \gamma^2\nu^2.
\end{equation}
At this point, equations \eqref{eq:tau}, \eqref{eq:theta}, and \eqref{eq:r} can be summarized as
\begin{align*}
    \theta = \frac{\alpha \delta}{\gamma}, \qquad
    \tau = \frac{\gamma}{\delta\sigma\big(1-\lambda \gamma\big)}, \qquad
    r = \sqrt{\frac{\delta \sigma^2}{\gamma^2} - \delta \nu^2}.
\end{align*}
Substituting these expressions into \eqref{eq:salehi-4}, \eqref{eq:salehi-5}, \eqref{eq:salehi-6} yields 
the system of three equations in \cref{thm:logistic-objective-perturbation}\ref{thm:logistic-objective-utility}. Therefore, we should expect that for the solution $\alpha^\star, \sigma^\star, \gamma^\star > 0$, the expression \eqref{eq:derivation-conclusion} that captures the behavior of the coordinates of $\widehat{\bm{\beta}}$ should reduce to
\[
    \alpha^\star \beta^\star_0 + \sqrt{(\sigma^\star)^2 - \gamma^2 \nu^2} Z + \gamma \nu \xi_0.
\]
This completes our heuristic derivation, which we emphasize again was not rigorous because we did not satisfy the technical assumptions of \cite{salehi2019impact}, including a single, deterministic regularizer $f : \R \to \R$ applied coordinate-wise, and a ground-truth vector $\bm{\beta}^\star$ with i.i.d. coordinates.

\subsection{Proof of Theorem~\ref{thm:logistic-objective-perturbation}}

The overall structure of our proof of \cref{thm:logistic-objective-perturbation} mirrors analyses for non-private logistic regression estimators that have appeared in the literature \cite{salehi2019impact, han2024entrywise}, but we make several changes to account for the differences between their setting and ours. The main differences include our analysis of the random linear perturbation term introduced for differential privacy, our weakened assumptions on $\bm{\beta}^\star$ (\cite{salehi2019impact} assumed $\bm{\beta}^\star$ to be random with i.i.d. coordinates, but we do not), and our conclusion that includes a description of $\rho'(\bm{X}\bm{\beta}^\star) - \rho'(\bm{X}\bm{\hb})$. Our proof has three steps.

\begin{enumerate}[(1)]
    \item Use universality to relate the perturbed logistic objective to one with a Gaussian design.
    \item Rewrite this modified problem as a min-max optimization amenable to CGMT.
    \item Use CGMT to analyze this min-max optimization with a Gaussian design matrix as $n \to \infty$.
\end{enumerate}

For the first step, we will rely on recent universality results of \cite{han2024entrywise} that they use to study non-private, regularized logistic regression. In their most general framing, these universality results apply to any estimator that can be expressed as the output of a certain class of \emph{first-order} algorithms, or a limit thereof. As we will show, the objective perturbation algorithm for logistic regression falls in this class. This allows us to relate the behavior of the algorithm's output $\bm{\hb} \in \R^d$ (and the corresponding predictions $\rho'(\bm{X}\bm{\hb}) \in [0, 1]^n$) in the setting amenable to differential privacy where $\bm{X}$ has bounded entries to the setting in which the entries of $\bm{X}$ are unbounded, but precisely follow a Gaussian law.

For the second step, as before, we will take the Legendre transform of the logistic loss function in order to isolate the randomness of the design matrix $\bm{X}$ in a single, bilinear term. Compared to the case of robust linear regression in \cref{sec:huber-objective-utility}, some extra care will be needed to ensure that this bilinear term is independent of the mean function $\psi$.

For the third step, we will use CGMT, roughly following the template in \cite{salehi2019impact}, to derive the asymptotic behavior of the estimation error and prediction error of $\widehat{\bm{\beta}}$. Unlike \cite{salehi2019impact}, our analysis will account for the random linear perturbation term, and also explain the behavior of $\rho'(\bm{X}\bm{\beta}^\star) - \rho'(\bm{X}\widehat{\bm{\beta}})$, which is related to the residual error. As before, our analysis will be made more challenging by the fact that we do not assume that $\bm{\beta}^\star$ has i.i.d. coordinates.

One will notice that the order of steps (1), (2), and (3) is different than in our analysis of robust linear regression in \cref{sec:huber-objective-utility}. The reason we must apply universality before transforming the problem into a min-max optimization amenable to CGMT is that, as we will see, this transformation \emph{itself} requires the Gaussianity assumption. This is also why we cannot simply rely on the universality results of \cite{han2023universality} (which require the problem to already be expressed in ``CGMT form'') and must instead use the more recent results of \cite{han2024entrywise}.

\subsubsection{Step 1: GFOM Universality}

The goal of this step is to prove the following lemma, a differentially private analogue of Theorem 4.3 in \cite{han2024entrywise}. The following lemma is a universality result that relates the output of objective perturbation (\cref{alg:objective-perturbation}) with the logistic loss function to its output when we replace $\bm{X}$ with $\frac{1}{\sqrt{d}}\bm{G}$, where $\bm{G}$ has the same shape as $\bm{X}$ but independent, standard Gaussian entries. Unlike $\bm{X}$, whose rows have $\ell^2$ norm bounded by $R$ (as required for differential privacy), the entries of the matrix $\bm{G}$ are unbounded in the worst case. Despite this, the lemma shows that the coordinate-wise average of any pseudo-Lipschitz test function, applied to either the estimate $\bm{\hb}$ or its predictions $\rho'(\bm{X}\bm{\hb})$, is roughly the same regardless of whether we start with $\bm{X}$ or $\frac{1}{\sqrt{d}}\bm{G}$.

\begin{lemma}
\label{thm:logistic-universality}
    Suppose $\bm{X} \in \R^{n \times d}$ follows a subgaussian design and $\bm{y}|(\bm{\beta}^\star,\bm{X}) \sim \mathrm{Bernoulli}(\rho'(\bm{X}\bm{\beta}^\star))$ for some $\bm{\beta}^\star \in \R^d$ satisfying $\norm{\bm{\beta}^\star} \le O(n)$. Let $\bm{G} \in \R^{n \times d}$ have standard Gaussian entries, and let $\bm{\hb}(\bm{A})$ denote the output of \cref{alg:objective-perturbation} with the logistic loss function and design matrix $\bm{A} \in \{\bm{X}, \frac{1}{\sqrt{d}}\bm{G}\}$.
    Then, for any collection of $O(1)$-pseudo-Lipschitz functions $\psi_{1i} : \R^2 \to \R$ and $\psi_{2j} : \R \to \R$ of order $2$ for $i \in [n]$ and $j \in [d]$, we have that w.h.p.,\footnote{Our proof will actually show that the bound holds both in expectation and with high probability.}
    \begin{align*}
        &\biggl\lvert\frac{1}{n}\sum_{i=1}^n \biggl(\psi_{1i}(\langle \bm{x}_i, \bm{\beta}^\star\rangle, \langle \bm{x}_i, \bm{\hb}(\bm{X})\rangle) - \psi_{1i}\Bigl(\Bigl\langle\frac{\bm{g}_i}{\sqrt{d}}, \bm{\beta}^\star\Bigr\rangle, \Bigl\langle\frac{\bm{g}_i}{\sqrt{d}}, \bm{\hb}\Bigl(\frac{\bm{G}}{\sqrt{d}}\Bigr)\Bigr\rangle\Bigr)\biggr)\biggr\rvert \\
        &+ \biggl\lvert\frac{1}{d}\sum_{j=1}^d \biggl(\psi_{2j}(\hb_j(\bm{X})) - \psi_{2j}\Bigl(\hb_j\Bigl(\frac{\bm{G}}{\sqrt{d}}\Bigr)\Bigr)\biggr)\biggr\rvert \le e^{-(\log n)^{\Omega(1)}}.
    \end{align*}
\end{lemma}

\begin{proof}
    As in the proof of Theorem 4.3 of \cite{han2024entrywise} for non-private logistic regression, we will define a suitably smoothed sequence of iterates $\bm{\beta}^{(t)} \in \R^d$, indexed by $t \in \N$, that converge to $\bm{\hb}$ as $t \to \infty$. We will then conclude by \cref{thm:gfom-universality}, the universality of general first order methods (GFOMs). Unlike prior work, our proof will account for the noise introduced by objective perturbation for differential privacy. We also extend prior work by establishing universality for the logits $\bm{X}\bm{\hb}$ in addition to $\bm{\hb}$. Besides these changes, our proof and notation borrow heavily from that of Theorem 4.3 of \cite{han2024entrywise}, with a few simplifications due to our focus on the $\ell^2$ regularizer.

    To begin, we recast the logistic regression model in its \emph{latent-variable} form. To this end, consider the distribution over $\R$ with CDF $\rho'$, called the \emph{logistic distribution}, and define $\eps^\star_1, \ldots, \eps^\star_n \iid \mathrm{Logistic}$. Then, the assumption that $\bm{y} \sim \mathrm{Bernoulli}(\rho'(\bm{X}\bm{\beta}^\star))$ can be rephrased as
    \[
        \bm{y} = \bm{1}[\bm{X}\bm{\beta}^\star + \bm{\eps}^\star > 0].
    \]
    Similarly, the loss value $\ell(\bm{\beta}; (\bm{x}_i, y_i)) = \rho(\langle \bm{x}_i, \bm{\beta}\rangle) - y_i\langle \bm{x}_i, \bm{\beta}\rangle$ incurred on the $i$\textsuperscript{th} data point can be rephrased as $L(\langle \bm{x}_i, \bm{\beta}\rangle, \langle \bm{x}_i, \bm{\beta}^\star \rangle; \eps^\star_i)$, where the function $L : \R^3 \to \R$ is given by
    \[
        L(x, y; \eps) = \rho(-(2 \cdot \bm{1}[y + \eps > 0] - 1)x).
    \]
    Since we plan to take derivatives of $L$ with respect to both its first and second inputs (call these operations $\partial_1$ and $\partial_2$), it will be convenient to consider the following smoothed version of $L$, parameterized by $\sigma > 0$. Letting $\varphi : \R \to [0, 1]$ be any infinitely differentiable, monotonically increasing function such that $\varphi(x) = 0$ for all $x \le -1$ and $\varphi(x) = 1$ for all $x \ge 1$, set $\varphi_\sigma(x) = \varphi(x/\sigma)$ and
    \(
        L_\sigma(x, y; \eps) = \rho(-(2\varphi_\sigma(y + \eps) - 1)x).
    \)
    Similarly, set $\varphi_0(x) = \bm{1}[x > 0]$ and $L_0(x,y;\eps) = L(x,y;\eps)$.

    Next, consider the $\sigma$-smoothed version of the pertubed objective:
    \[
        \mathcal{L}_\sigma(\bm{\beta}) = \sum_{i=1}^n L_\sigma(\langle \bm{x}_i, \bm{\beta}\rangle, \langle \bm{x}_i, \bm{\beta}^\star \rangle; \eps^\star_i) + \frac{\lambda}{2}\norm{\bm{\beta}}^2 + \nu\langle\bm{\xi},\bm{\beta}\rangle.
    \]
    Its gradient is
    \[
        \nabla \mathcal{L}_\sigma(\bm{\beta}) = \sum_{i=1}^n \bm{x}_i\, \partial_1 L_\sigma(\langle \bm{x}_i, \bm{\beta}\rangle, \langle \bm{x}_i, \bm{\beta}^\star \rangle; \eps^\star_i) + \lambda \bm{\beta} + \nu\bm{\xi},
    \]
    where
    \begin{align*}
        \partial_1 L_\sigma(x, y; \eps) &= -(2\varphi_\sigma(y+\eps)-1)\rho'(-(2\varphi_\sigma(y+\eps)-1)x) \in [-1, 1]\\
        \partial_1^2 L_\sigma(x, y; \eps) &= (2\varphi_\sigma(y + \eps) - 1)^2\rho''(-(2\varphi_\sigma(y+\eps)-1)x) \in[0, 1/4].
    \end{align*}
    For some step size $\eta$, define the following gradient descent iterates on the $\sigma$-smoothed version of the perturbed objective:
    \begin{align*}
        \bm{\beta}_\sigma^{(0)} = \bm{0}, \qquad \bm{\beta}^{(t)}_\sigma = \bm{\beta}^{(t - 1)}_\sigma - \eta \nabla \mathcal{L}_\sigma(\bm{\beta}^{(t - 1)}_\sigma).
    \end{align*}
    We now derive several useful properties of $\bm{\beta}^{(t)}_\sigma$, such as its behavior as $t \to \infty$, $\sigma \to 0$ and upper bounds on $\norm{\bm{\beta}^{(t)}_\sigma}$, $\norm{\bm{X}\bm{\beta}^{(t)}_\sigma}_\infty$. The argument for these upper bounds will be similar in spirit to the proof of \cref{thm:objective-uniform-bound} that we already saw and will, at various times, require bounding the effect of the noise introduced for differential privacy.

    \begin{enumerate}[(a)]
        \item \label{sec:logistic-iterates-converge} \textbf{(Behavior as $t \to \infty$)} In the limit, the iterates converge to $\bm{\hb}_\sigma = \argmin_{\bm{\beta} \in \R^d}\, \mathcal{L}_\sigma(\bm{\beta})$. To prove this, we subtract the equation $\eta\nabla \mathcal{L}_\sigma(\bm{\hb}_\sigma) = \bm{0}$ from the equation defining $\bm{\beta}^{(t)}_\sigma$:
        \[
            \bm{\beta}^{(t)}_\sigma - \bm{\hb}_\sigma = \left(\bm{\beta}^{(t - 1)}_\sigma - \eta\nabla\mathcal{L}_\sigma(\bm{\beta}^{(t-1)}_\sigma)\right)-\left(\bm{\hb}_\sigma - \eta \nabla \mathcal{L}_\sigma(\bm{\hb}_\sigma)\right).
        \]
        Observe that both invocations of $\nabla \mathcal{L}_\sigma$ in the above equation contain an term of the form $-\eta \nu\bm{\xi}$ due to objective perturbation, but these terms are canceled out by the subtraction. Thus, the remainder of this part of the proof is identical to the non-private case. We include it only for the sake of completeness:

        We apply the mean value theorem to the $\partial_1 L_\sigma$ terms in $\nabla \mathcal{L}_\sigma$, and use the fact that $\partial_1^2 L_\sigma$ takes values in $[0, 1/4]$ to obtain $\bm{c} \in [0, 1/4]^n$ such that
        \[
            \bm{\beta}^{(t)}_\sigma - \bm{\hb}_\sigma = \left((1 - \eta\lambda)\bm{I} - \eta \sum_{i=1}^n c_i \bm{x}_i\bm{x}_i^\top\right)(\bm{\beta}^{(t - 1)}_\sigma - \bm{\hb}_\sigma).
        \]
        Since $\bm{X}$ follows a subgaussian design, with probability $1 - e^{-\Omega(n)}$, we have $\bm{0} \preceq \eta\sum_{i=1}^n c_i\bm{x}_i\bm{x}_i^\top \preceq O(\eta)\bm{I}$. Therefore, there exists some choice of step size $\eta = \Theta(1/(1 + \lambda))$ such that iterating the above equation for $t = 1, 2, \ldots$ yields, with probability $1 - e^{-\Omega(n)}$,
        \[
            \norm{\bm{\beta}^{(t)}_\sigma - \bm{\hb}_\sigma} \le e^{-\Omega(t)}\norm{\bm{\hb}_\sigma}.
        \]
        \item \label{sec:smoothing-effect} \textbf{(Behavior as $\sigma \to 0$)} We will analyze the behavior of the function $\mathcal{L}_\sigma(\bm{\beta}) - \mathcal{L}(\bm{\beta})$. Clearly, the subtraction cancels out the objective perturbation term, so this part of the proof is also identical to the non-private case. We include it only for the sake of completeness:
        
        The outputs of $\varphi_\sigma$ and $\varphi_0$ only differ on inputs smaller than $\sigma$ in absolute value, and when they differ, they differ by at most $1$.
        Using this observation, followed by Cauchy-Schwarz,
        \begin{align*}
            \abs{\mathcal{L}_\sigma(\bm{\beta}) - \mathcal{L}_0(\bm{\beta})} &= \Abs{\sum_{i=1}^n (L_\sigma - L_0)(\langle \bm{x}_i, \bm{\beta} \rangle, \langle \bm{x}_i, \bm{\beta}^\star\rangle, \eps_i)} \\
            &\le 2\sum_{i=1}^n \abs{\langle \bm{x}_i, \bm{\beta}\rangle} \cdot \abs{(\varphi_\sigma - \varphi_0)(\langle \bm{x}_i, \bm{\beta}^\star\rangle + \eps_i)} \\
            &\le 2\norm{\bm{X}\bm{\beta}}\Bigl(\sum_{i=1}^n \bm{1}[\langle \bm{x}_i, \bm{\beta}^\star\rangle + \eps_i \in [-\sigma, +\sigma]]\Bigr)^{1/2}.
        \end{align*}
        Since the Logistic PDF takes values between $0$ and $1/4$, the terms of the parenthesized sum are independent Bernoulli random variables, each with parameter at most $2\sigma \cdot 1/4$. Consequently, standard concentration inequalities (e.g. Bernstein's) show that for any fixed constant $D > 0$, with probability $1 - O(n^{-D})$, the above sum is at most $O(\sigma n + \log n)$. We also have with probability $1 - e^{-\Omega(n)}$ that $\norm{\bm{X}} = O(1)$. Thus, with probability $1 - O(n^{-D})$,
        \[
            \abs{\mathcal{L}_\sigma(\bm{\beta}) - \mathcal{L}_0(\bm{\beta})} \le O((\sigma n + \log n)^{1/2} \norm{\bm{\beta}}).
        \]
        We apply this inequality to $\bm{\hb}$ and $\bm{\hb}_\sigma$, which minimize $\mathcal{L}$ and $\mathcal{L}_\sigma$, respectively. By the $\lambda$-strong convexity of these functions and \cref{thm:convex-minimization}, we deduce that with probability $1 - O(n^{-D})$,
        \[
            \norm{\bm{\hb}_\sigma - \bm{\hb}} \le O((\sigma n + \log n)^{1/4}(\norm{\bm{\hb}} + \norm{\bm{\hb}_\sigma})^{1/2}).
        \]
        \item \label{sec:coefficients-worst-case} \textbf{(Worst-Case Upper Bounds on $\norm{\bm{\beta}^{(t)}_\sigma}$, $\norm{\bm{\hb}_\sigma}$)} From the definition of $\bm{\beta}^{(t)}_\sigma$ and the fact that $\partial_1 L$ takes values in $[-1, +1]$,
        \[
            \norm{\bm{\beta}^{(t)}_\sigma} \le (1 - \eta\lambda)\norm{\bm{\beta}^{(t - 1)}_\sigma} + \eta\norm{\bm{X}^\top}_{\infty \to 2} + \eta\nu\norm{\bm{\xi}} = O(t\cdot (\norm{\bm{X}^\top}_{\infty \to 2} + \norm{\bm{\xi}})).
        \]
        The fact that $\mathcal{L}_\sigma(\bm{\hb}_\sigma) = \Omega(\norm{\bm{\hb}_\sigma}^2)$ is smaller than $\mathcal{L}_\sigma(\bm{0}) = O(n)$ implies $\norm{\bm{\hb}_\sigma} \le O(\sqrt{n})$.
        
        \item \label{sec:coefficients-high-prob} \textbf{(High Probability Upper Bound on $\norm{\bm{\beta}^{(t)}_\sigma}$, $\norm{\bm{\hb}_\sigma}$)} As in our analysis of $t \to \infty$, applying the mean value theorem to $\partial_1 L_\sigma$ in the equation $\nabla \mathcal{L}_\sigma(\bm{\hb}_\sigma) = \bm{0}$ yields $\bm{c} \in [0, 1/4]^n$ such that
        \[
            \left(\lambda\bm{I} + \sum_{i=1}^n c_i\bm{x}_i\bm{x}_i^\top \right)\bm{\hb}_\sigma + \sum_{i=1}^n \bm{x}_i\,\partial_1 L_\sigma(\bm{0}, \langle \bm{x}_i, \bm{\beta}^\star\rangle, \eps^\star_i) + \nu\bm{\xi} = \bm{0}.
        \]
        Therefore, with probability $1 - e^{-\Omega(n)}$,
        \[
            \norm{\bm{\hb}_\sigma} \le O{\left(\Bigl\lVert{\sum_{i=1}^n \bm{x}_i\,\partial_1 L_\sigma(\bm{0}, \langle \bm{x}_i, \bm{\beta}^\star\rangle, \eps^\star_i)\Bigr\rVert} + \norm{\bm{\xi}}\right)}.
        \]
        The second term measures the norm of the random perturbation vector $\bm{\xi}$, and is clearly at most $O(\sqrt{n \log n})$ with probability $1 - O(n^{-D})$, for any fixed constant $D > 0$. The same bound can be established for the first term, and the proof of this is identical to the non-private case. We include it only for the sake of completeness:

        Define the function $F_\sigma(t) = \E[\varphi_\sigma(t + \eps^\star_0)]$, where $\eps^\star_0 \sim \mathrm{Logistic}$. Using the fact that $\partial_1 L_\sigma(0, y, \eps) = -(2\varphi_\sigma(y + \eps)-1)$, followed by the triangle inequality,
        \begin{align*}
            &\Bigl\lVert\sum_{i=1}^n \bm{x}_i\,\partial_1 L(\bm{0}, \langle \bm{x}_i, \bm{\beta}^\star\rangle, \eps^\star_i)\Bigr\rVert^2 \\
            &= \frac{1}{4}\Bigl\lVert\sum_{i=1}^n \bm{x}_i(2\varphi_\sigma(\langle \bm{x}_i, \bm{\beta}^\star\rangle + \eps^\star_i)-1)\Bigr\rVert^2 \\
            &\le O\Biggl(\underbrace{\Bigl\lVert\sum_{i=1}^n \bm{x}_i(\varphi_\sigma(\langle \bm{x}_i, \bm{\beta}^\star\rangle + \eps^\star_i)-F_\sigma(\langle \bm{x}_i, \bm{\beta}^\star\rangle))\Bigr\rVert^2}_{S_1}  + \underbrace{\Bigl\lVert\sum_{i=1}^n \bm{x}_i(2F_\sigma(\langle \bm{x}_i, \bm{\beta}^\star\rangle)-1)\Bigr\rVert^2}_{S_2}\Biggr).
        \end{align*}
        We bound the terms $S_1$ and $S_2$ separately. First, $S_1 \le O(n \log n)$ with probability $1 - O(n^{-D})$ for any fixed constant $D > 0$ because for each fixed coordinate $j \in [d]$, the $j$\textsuperscript{th} coordinates of the summands across $i \in [n]$ are centered, independent variables with subgaussian norm $O(n^{-1/2})$. To bound $S_2$, we use the fact that $|F'_\sigma(t)| \le \norm{\varphi'}_\infty \le O(1)$ for any $\sigma \ge 0$ and $t \in \R$. Thus, for each fixed coordinate $j \in [d]$, if we let $\bm{x}_{i,(-j)} \in \R^d$ denote the version of $\bm{x}_i$ with the $j$\textsuperscript{th} coordinate zeroed out, then there is some scalar $\tau_j$ with $\abs{\tau_j} = O(1)$ such that
        \[
            F_\sigma(\langle \bm{x}_i, \bm{\beta}^\star\rangle) = F_\sigma(\langle \bm{x}_{i,(-j)}, \bm{\beta}^\star\rangle) + \tau_j x_{ij}\beta^\star_j.
        \]
        Consequently, expanding $S_2$,
        \[
            S_2 \le O{\left(\sum_{j=1}^d\left(\sum_{i=1}^n x_{ij}\right)^2 + \sum_{j=1}^d\left(\sum_{i=1}^n x_{ij}F_\sigma(\langle \bm{x}_{i,(-j)}, \bm{\beta}^\star\rangle)\right)^2 + \sum_{j=1}^d \left(\sum_{i=1}^n x_{ij}^2\right)^2(\beta^\star_j)^2\right)}.
        \]
        Using standard subgaussian concentration inequalities on $\bm{X}$, the fact that $x_{ij}$ is independent of $F_\sigma(\langle \bm{x}_{i,(-j)},\bm{\beta}^\star\rangle)$, and the fact that $\norm{\bm{\beta}^\star}_2 = O(n)$, we conclude that $S_2 \le O(n \log n)$ with probability $1 - O(n^{-D})$ for any constant $D > 0$.

        In summary, we have shown that with probability $1 - O(n^{-D})$,
        \[
            \norm{\bm{\hb}_\sigma} \le O(\sqrt{n \log n}).
        \]
        Combining this with the bound from part \ref{sec:logistic-iterates-converge}, with probability $1 - O(n^{-D})$,
        \[
            \norm{\bm{\beta}^{(t)}_\sigma} \le O(\sqrt{n \log n}).
        \]
        
        \item \textbf{(High Probability Upper Bound on $\norm{\bm{X}\bm{\beta}^{(t)}_\sigma}_\infty$)} \label{sec:logits-high-prob} Fix an index $k \in [n]$, and let $\mathcal{L}_{\sigma; [-k]}$ be the modified loss function that omits the term corresponding to the $k$\textsuperscript{th} sample $\bm{x}_k$. As before, define iterates
        \[
            \bm{\beta}^{(0)}_{\sigma; [-k]} = \bm{0}, \qquad \bm{\beta}^{(t)}_{\sigma; [-k]} = \bm{\beta}^{(t-1)}_{\sigma; [-k]} - \eta\nabla \mathcal{L}_{\sigma; [-k]}(\bm{\beta}^{(t-1)}_{\sigma; [-k]}).
        \]
        Clearly, these converge to $\bm{\hb}_{\sigma; [-k]} = \argmin_{\bm{\beta} \in \R^d} \mathcal{L}_{\sigma; [-k]}(\bm{\beta})$ at the same rate that we proved $\bm{\beta}^{(t)}_\sigma$ converges to $\bm{\hb}_{\sigma}$. We proceed by subtracting the equations defining $\bm{\beta}^{(t)}_{\sigma}$ and $\bm{\beta}^{(t)}_{\sigma;[-k]}$, which cancels out the effect of linear objective perturbation. Therefore, the remainder of this part of the proof is identical to the non-private case. We include it only for the sake of completeness:
        
        Similar to before, we see that there exist $\bm{c} \in [0, 1/4]^n$ and $c_0 \in [-1, 1]$ such that
        \[
            \bm{\beta}^{(t)}_{\sigma} - \bm{\beta}^{(t)}_{\sigma;[-k]} = \left((1 - \eta\lambda)\bm{I} - \eta\sum_{i=1}^n c_i \bm{x}_i\bm{x}_i^\top\right)(\bm{\beta}^{(t-1)}_{\sigma} - \bm{\beta}^{(t-1)}_{\sigma;[-k]}) - \eta c_0 \bm{x}_k.
        \]
        Notice that the final $-\eta c_0\bm{x}_k$ term arises because the gradient of $\mathcal{L}_{\sigma}$ includes a term for the loss of $(\bm{x}_k, y_k)$, whereas $\mathcal{L}_{\sigma; [-k]}$ does not. This implies that for appropriate $\eta = \Theta(1)$, with probability $1 - e^{-\Omega(n)}$,
        \[
            \norm{\bm{\beta}^{(t)}_{\sigma} - \bm{\beta}^{(t)}_{\sigma;[-k]}} \le \frac{1}{2} \norm{\bm{\beta}^{(t-1)}_{\sigma} - \bm{\beta}^{(t-1)}_{\sigma;[-k]}} + O(1).
        \]
        Iterating the above equation for $t = 1, 2, \ldots$ and taking a union bound over $k \in [n]$ yields, with probability $1 - e^{-\Omega(n)}$,
        \[
            \norm{\bm{\beta}^{(t)}_{\sigma} - \bm{\beta}^{(t)}_{\sigma;[-k]}} \le O(1).
        \]
        A nearly identical argument yields, with probability $1 - e^{-\Omega(n)}$,
        \[
            \norm{\bm{\hb}_{\sigma} - \bm{\hb}_{\sigma;[-k]}} \le O(1).
        \]
        Next, by the triangle inequality,
        \[
            \norm{\bm{X} \bm{\beta}^{(t)}_\sigma}_\infty \le \underbrace{\max_{k \in [n]}\,\abs{\langle \bm{x}_k, \bm{\beta}^{(t)}_{\sigma;[-k]}\rangle}}_{S_1} + \underbrace{\max_{k \in [n]}\, \abs{\langle \bm{x}_k, \bm{\beta}^{(t)}_\sigma - \bm{\beta}^{(t)}_{\sigma;[-k]}\rangle}}_{S_2}.
        \]
        We bound $S_1$ and $S_2$ separately. First, for $S_1$, observe that $\bm{x}_k$ and $\bm{\beta}^{(t)}_{\sigma;[-k]}$ are independent by definition, and recall that we showed in part \ref{sec:logistic-iterates-converge} that for any constant $D > 0$, with probability $1 - O(n^{-D})$, we have $\norm{\bm{\beta}^{(t)}_{\sigma;[-k]}} \le O(\sqrt{n \log n})$. Therefore, \cref{thm:subgaussian-inner-product} implies that $S_1 \le O(\log n)$. Next, for $S_2$, recall that we just proved $\norm{\bm{\beta}^{(t)}_{\sigma} - \bm{\beta}^{(t)}_{\sigma;[-k]}} \le O(1)$. Therefore, by Cauchy-Schwarz, $S_2 \le O(\sqrt{\log n})$. Similar arguments hold with $\bm{\hb}_{\sigma;[-k]}$ in place of $\bm{\beta}^{(t)}_{\sigma;[k]}$. We conclude that for any constant $D > 0$, with probability $1 - O(n^{-D})$,
        \[
            \norm{\bm{X}\bm{\beta}^{(t)}_\sigma}_\infty + \norm{\bm{X}\bm{\hb}_\sigma}_\infty \le O(\log n).
        \]
    \end{enumerate}
    
    With our various bounds on the iterates $\bm{\beta}^{(t)}$ in hand, we now are ready to invoke GFOM universality (\cref{thm:gfom-universality}), which is the main step of the proof. We will construct the required GFOM using the following functions $G_{\sigma,i} : \R^{2} \to \R$ for $\sigma \ge 0$ and $i \in [n]$:
    \[
        G_{\sigma,i}(\bm{u}) = \partial_1 L_\sigma([u_1]_M, u_2; \eps_i),
    \]
    where $M \le O(\log n)$ is the bound on $\norm{\bm{X}\bm{\beta}^{(t)}_\sigma}_\infty + \norm{\bm{X}\bm{\hb}_\sigma}_\infty$ from part \ref{sec:logits-high-prob}. Let $G_\sigma : \R^{m \times 2} \to \R$ be the separable function that applies $G_{\sigma, i}$ to its $i$\textsuperscript{th} input. A direct calculation reveals that $\abs{\partial_1\partial_2 L_\sigma(x, y; \eps)} \le O(\abs{x}/\sigma)$, from which it follows that
    \[
        \sup_{\eps \in \R^n} \max_{i \in [n]}\,\norm{G_{\sigma, i}}_{\mathrm{Lip}} + \abs{G_{\sigma, i}(\bm{0})} \le O{\left(1 + \frac{M}{\sigma}\right)}.
    \]
    Now define the GFOM iterates $\bm{u}^{(t)} \in \R^{n \times 2}$ and $\bm{v}^{(t)} \in \R^d$ by $\bm{u}^{(0)} = \bm{0}$ and $\bm{v}^{(0)} = \bm{0}$ and
    \begin{equation*}
    \begin{split}
        \bm{u}^{(t)} &= \bm{X}\left[\bm{v}^{(t-1)}\mid \bm{\beta}^\star\right] \in \R^{n \times 2}, \\
        \bm{v}^{(t)} &= \bm{X}^\top [-\eta G_\sigma(\bm{u}^{(t)})] + (1 - \eta \lambda)\bm{v}^{(t - 1)} - \eta \nu \bm{\xi} \in \R^d.
    \end{split}
    \end{equation*}
    Clearly, in the event $E_\sigma$ that $\norm{\bm{X}\bm{\beta}^{(t)}_\sigma}_\infty \le M$ at each step, the truncation $[\cdot]_M$ in the definition of $G$ never takes effect, so we have for all $t \in \N$ that $\bm{v}^{(t)}$ coincides with $\bm{\beta}^{(t)}_\sigma$ and that the first column of $\bm{u}^{(t)}$ coincides with $\bm{X}\bm{\beta}^{(t)}_\sigma$. By GFOM universality (\cref{thm:gfom-universality}), for any $O(1)$-pseudo-Lipschitz functions $\psi_{1i} : \R^2 \to \R$ and $\psi_{2j} : \R \to \R$ of order $O(1)$ for $i \in [n]$ and $j \in [d]$, there exists a constant $C > 0$ such that if $\bm{G} \in \R^{n \times d}$ has independent standard Gaussian entries,
    \begin{align*}
        &\E\left[\biggl\lvert{\frac{1}{n}\sum_{i=1}^n \bigg(\psi_{1i}(\langle \bm{x}_i, \bm{\beta}^\star\rangle, \langle \bm{x}_i, \bm{\beta}^{(t)}_{\sigma}(\bm{X})\rangle) - \psi_{1i}\Bigl(\Bigl\langle \frac{\bm{g}_i}{\sqrt{d}}, \bm{\beta}^\star\Bigr\rangle, \Bigl\langle \frac{\bm{g}_i}{\sqrt{d}}, \bm{\beta}^{(t)}_{\sigma}\Bigl(\frac{\bm{G}}{\sqrt{d}}\Bigr)\Bigr\rangle\Bigr)\biggr)}\biggr\rvert\bm{1}_{E_\sigma}\right] \\
        &+ \E\left[\biggl\lvert{\frac{1}{d}\sum_{j=1}^d \biggl( \psi_{2j}(\beta^{(t)}_{\sigma,j}(\bm{X})) - \psi_{2j}\Bigl(\beta^{(t)}_{\sigma,j}\Bigl(\frac{\bm{G}}{\sqrt{d}}\Bigr)\Bigr)\biggr)}\biggr\rvert\bm{1}_{E_\sigma}\right] \le {(C\sigma^{-1}\log(n))^{Ct^3}}{n^{-1/(Ct^3)}}.
    \end{align*}
    To conclude the proof, all that remains is to derive a version of the above inequality without the indicator functions $\bm{1}_{E_\sigma}$, and then consider the limit as $t \to \infty$ and $\sigma \to 0$. The remainder of the proof is similar to that of the non-private case; the exception is that we often invoke the bounds from parts \ref{sec:logistic-iterates-converge}, \ref{sec:smoothing-effect}, \ref{sec:coefficients-worst-case}, \ref{sec:coefficients-high-prob}, \ref{sec:logits-high-prob}, whose proofs took into consideration the effect of objective perturbation at various points.
    
    To remove the indicator functions, first recall that we showed in part \ref{sec:logits-high-prob} above that $\Pr[E_\sigma] = 1 - O(n^{-D})$ for any constant $D > 0$. Thus, it suffices to show that even in the \emph{worst} case, which includes the unlikely, probability $O(n^{-D})$ event that $E_\sigma$ does \emph{not} occur, that all of the sums in the above display are bounded in expectation by $O(n^{c})$ for some constant $c > 0$ that does not depend on $D$. Indeed, taking $D > c$, this crude bound will lead to only a mild additive $O(n^{c} \cdot n^{-D}) = n^{-\Omega(1)}$ increase to the GFOM universality error bound of $(C\sigma^{-1}\log(n))^{Ct^3}n^{-1/(Ct^3)}$.

    To derive such a bound, observe that if the functions $\psi_{1i},\psi_{2j}$ are pseudo-Lispchitz \emph{of order $2$}, then for either $\bm{A} \in \{\bm{X}, \frac{1}{\sqrt{d}}\bm{G}\}$,
    \begin{align*}
        &\E\Abs{\frac{1}{d}\sum_{j=1}^d \psi_{1j}(\langle \bm{a}_i, \bm{\beta}^\star\rangle, \langle \bm{a}_i, \beta^{(t)}_{\sigma,j}(\bm{A})\rangle)} + \E\Abs{\frac{1}{d}\sum_{j=1}^d \psi_{2j}(\beta^{(t)}_{\sigma,j}(\bm{A}))} \\
        &\le O{\left(1 + \frac{\E\norm{\bm{A}\bm{\beta}^\star}^2}{n} + \frac{\E\norm{\bm{A}\cdot \bm{\beta}^{(t)}_\sigma(\bm{A})}^2}{n} + \frac{\E\norm{\bm{\beta}^{(t)}_\sigma(\bm{A})}^2}{n}\right)}.
    \end{align*}
    By part \ref{sec:coefficients-worst-case}, we have $\norm{\bm{A}\bm{\beta}^\star}^2 + \norm{\bm{A}\cdot \bm{\beta}^{(t)}_\sigma(\bm{A})}^2 + \norm{\bm{\beta}^{(t)}_\sigma(\bm{A})}^2 \le (tn\norm{\bm{A}}\norm{\bm{\xi}})^{O(1)}$, which is clearly $n^{O(1)}$ in expectation, as desired.

    We have thus removed the indicator functions $\bm{1}_{E_\sigma}$. Specifically, we have shown that for any collection of $O(1)$-pseudo-Lipschitz functions $\psi_{1i}, \psi_{2j} : \R \to \R$ of order $2$ and any constant $D > 0$, there exists a constant $C > 0$ such that with probability $1 - Cn^{-D}$:
    \begin{align*}
        &\E\biggl\lvert{\frac{1}{n}\sum_{i=1}^n \bigg(\psi_{1i}(\langle \bm{x}_i, \bm{\beta}^\star\rangle, \langle \bm{x}_i, \bm{\beta}^{(t)}_{\sigma}(\bm{X})\rangle) - \psi_{1i}\Bigl(\Bigl\langle \frac{\bm{g}_i}{\sqrt{d}}, \bm{\beta}^\star\Bigr\rangle, \Bigl\langle \frac{\bm{g}_i}{\sqrt{d}}, \bm{\beta}^{(t)}_{\sigma}\Bigl(\frac{\bm{G}}{\sqrt{d}}\Bigr)\Bigr\rangle\Bigr)\biggr)}\biggr\rvert \\
        &+ \E\biggl\lvert{\frac{1}{d}\sum_{j=1}^d \biggl( \psi_{2j}(\beta^{(t)}_{\sigma,j}(\bm{X})) - \psi_{2j}\Bigl(\beta^{(t)}_{\sigma,j}\Bigl(\frac{\bm{G}}{\sqrt{d}}\Bigr)\Bigr)\biggr)}\biggr\rvert \le {(C\sigma^{-1}\log n)^{Ct^3}}{n^{-1/(Ct^3)}} + Cn^{-D}.
    \end{align*}
    Next, we consider what happens as $t \to \infty$. Indeed, using the definition of order-$2$ pseudo-Lipschitzness, followed by Cauchy-Schwarz, for both $\bm{A} \in \{\bm{X}, \frac{1}{\sqrt{d}}\bm{G}\}$,
    \begin{align*}
        &\biggl\lvert{\frac{1}{n}\sum_{i=1}^n \biggl( \psi_{1i}(\langle \bm{a}_i, \bm{\beta}^\star\rangle, \langle \bm{a}_i, \bm{\beta}^{(t)}_{\sigma}(\bm{A})\rangle) - \psi_{1i}(\langle \bm{a}_i, \bm{\beta}^\star\rangle, \langle \bm{a}_i, \bm{\hb}_{\sigma}(\bm{A})\rangle)\biggr)}\biggr\rvert
        \\
        &\qquad\le O{\left( \frac{1}{n}\sum_{i=1}^n \abs{\langle \bm{a}_i, \bm{\beta}^{(t)}_{\sigma}(\bm{A}) - \bm{\hb}_{\sigma}(\bm{A})\rangle} \cdot (1 + \abs{\langle \bm{a}_i, \bm{\beta}^\star\rangle} + \abs{\langle \bm{a}_i, \bm{\beta}^{(t)}_{\sigma}(\bm{A})\rangle} + \abs{\langle \bm{a}_i, \bm{\hb}_{\sigma}(\bm{A})\rangle}) \right)} \\
        &\qquad\le O{\left(\frac{1}{n} \cdot \norm{\bm{A}} \cdot \norm{\bm{\beta}^{(t)}_\sigma(\bm{A}) - \bm{\hb}_\sigma(\bm{A})} \cdot \bigl(\sqrt{n} + \norm{\bm{A}} \cdot (\norm{\bm{\beta}^\star} + \norm{\bm{\beta}^{(t)}_{\sigma}(\bm{A})} + \norm{\bm{\hb}_{\sigma}(\bm{A})})\bigr)\right)}.
    \end{align*}
    Similarly, for either $\bm{A} \in \{\bm{X}, \frac{1}{\sqrt{d}}\bm{G}\}$,
    \begin{align*}
        &\biggl\lvert{\frac{1}{d}\sum_{j=1}^d \biggl( \psi_{2j}(\beta^{(t)}_{\sigma,j}(\bm{A})) - \psi_{2j}(\hb_{\sigma,j}(\bm{A}))\biggr)}\biggr\rvert
        \\
        &\qquad\le O{\left( \frac{1}{d}\sum_{j=1}^d \abs{\beta^{(t)}_{\sigma,j}(\bm{A}) - \hb_{\sigma, j}(\bm{A})} \cdot (1 + \abs{\beta^{(t)}_{\sigma,j}(\bm{A})} + \abs{\hb_{\sigma, j}(\bm{A})}) \right)} \\
        &\qquad\le O{\left(\frac{1}{n} \cdot \norm{\bm{\beta}^{(t)}_\sigma(\bm{A}) - \bm{\hb}_\sigma(\bm{A})} \cdot (\sqrt{n} + \norm{\bm{\beta}^{(t)}_{\sigma}(\bm{A})} + \norm{\bm{\hb}_{\sigma}(\bm{A})})\right)}.
    \end{align*}
    As before, we consider separately the cases in which $E_\sigma$ does and does not occur. In the case that $E_\sigma$ does occur, our high-probability bound  $\norm{\bm{\beta}^{(t)}_\sigma(\bm{A}) - \bm{\hb}_\sigma(\bm{A})} \le e^{-\Omega(t)}\norm{\bm{\hb}_{\sigma}(\bm{A})}$ from part \ref{sec:logistic-iterates-converge} holds, as does our high probability bound $\norm{\bm{\hb}_\sigma(\bm{A})} + \norm{\bm{\beta}^{(t)}_\sigma(\bm{A})} \le O(\sqrt{n\log n})$ from part \ref{sec:coefficients-high-prob}, as well as the bound $\norm{\bm{A}} \le O(1)$. In the case that $E_\sigma$ does \emph{not} occur, our crude bounds on all the aforementioned terms of order $n^{O(1)}$ from part \ref{sec:coefficients-worst-case} still hold.
    Therefore, taking expectations of both sides of the above two displays, we see that for any constant $D > 0$, there exists a constant $C > 0$ such that
    \begin{align*}
        &\E\biggl\lvert{\frac{1}{n}\sum_{i=1}^n \biggl( \psi_{1i}(\langle \bm{a}_i, \bm{\beta}^\star\rangle, \langle \bm{a}_i, \bm{\beta}^{(t)}_{\sigma}(\bm{A})\rangle) - \psi_{1i}(\langle \bm{a}_i, \bm{\beta}^\star\rangle, \langle \bm{a}_i, \bm{\hb}_{\sigma}(\bm{A})\rangle)\biggr)}\biggr\rvert \\
        &+ \E\biggl\lvert{\frac{1}{d}\sum_{j=1}^d \biggl( \psi_{2j}(\beta^{(t)}_{\sigma,j}(\bm{A})) - \psi_{2j}(\hb_{\sigma,j}(\bm{A}))\biggr)}\biggr\rvert  \le e^{-t/C}(\log n)^C + Cn^{-D}.
    \end{align*}
    Just as the above inequality uses the bound from part \ref{sec:logistic-iterates-converge} (along with parts \ref{sec:coefficients-worst-case} \ref{sec:coefficients-high-prob},  \ref{sec:logits-high-prob}) to determine the effect of taking $t \to \infty$, so too can we use the bound from part \ref{sec:smoothing-effect} (along with parts \ref{sec:coefficients-worst-case} \ref{sec:coefficients-high-prob},  \ref{sec:logits-high-prob}) to determine the effect of taking $\sigma \to 0$:
    \begin{align*}
        &\E\biggl\lvert{\frac{1}{n}\sum_{i=1}^n \biggl( \psi_{1i}(\langle \bm{a}_i, \bm{\beta}^\star\rangle, \langle \bm{a}_i, \bm{\hb}_{\sigma}(\bm{A})\rangle) - \psi_{1i}(\langle \bm{a}_i, \bm{\beta}^\star\rangle, \langle \bm{a}_i, \bm{\hb}(\bm{A})\rangle)\biggr)}\biggr\rvert \\
        &+ \E\biggl\lvert{\frac{1}{d}\sum_{j=1}^d \biggl( \psi_{2j}(\hb_{\sigma,j}(\bm{A})) - \psi_{2j}(\hb_{j}(\bm{A}))\biggr)}\biggr\rvert \le \left(\sigma + \frac{\log n}{n}\right)^{1/4}(\log n)^{C} + Cn^{-D}.
    \end{align*}
    At this point, we have shown that for any constant $D > 0$, there exists a constant $C > 0$ such that
    \begin{align*}
        &\E\biggl\lvert\frac{1}{n}\sum_{i=1}^n \biggl(\psi_{1i}(\langle \bm{x}_i, \bm{\beta}^\star\rangle, \langle \bm{x}_i, \bm{\hb}(\bm{X})\rangle) - \psi_{1i}\Bigl(\Bigl\langle\frac{\bm{g}_i}{\sqrt{d}}, \bm{\beta}^\star\Bigr\rangle, \Bigl\langle\frac{\bm{g}_i}{\sqrt{d}}, \bm{\hb}\Bigl(\frac{\bm{G}}{\sqrt{d}}\Bigr)\Bigr\rangle\Bigr)\biggr)\biggr\rvert \\
        &+ \E\biggl\lvert\frac{1}{d}\sum_{j=1}^d \biggl(\psi_{2j}(\hb_j(\bm{X})) - \psi_{2j}\Bigl(\hb_j\Bigl(\frac{\bm{G}}{\sqrt{d}}\Bigr)\Bigr)\biggr)\biggr\rvert \\
        &\le {(C\sigma^{-1}\log n)^{Ct^3}}{n^{-1/(Ct^3)}} + e^{-t/C}(\log n)^C + \left(\sigma + \frac{\log n}{n}\right)^{1/4}(\log n)^{C} + Cn^{-D}.
    \end{align*}
    Taking $t = (\log n)^{1/8}$ and $\sigma = e^{-(\log n)^{1/8}}$ simplifies the entire bound to $e^{-\Omega((\log n)^{1/8})}$. Passing from a bound in expectation to a high-probability bound via Markov's inequality concludes the proof.
\end{proof}

\subsubsection{Step 2: Legendre Transform}

The goal of this step is to prove the following lemma, which relates the output of \cref{alg:objective-perturbation} with logistic loss to the min-max optimization of a certain random variable. As before, the lemma technically holds in a worst-case sense, and only later will the randomness of $\bm{X}$ become necessary.

In order to state the lemma, first recall from \cref{sec:preliminaries} that $\rho^\star$ denotes the convex conjugate of $\rho$. It is $4$-strongly convex, and its negation is commonly known as the \emph{binary entropy function} (in nats):
\[
    -\rho^\star(s) = \begin{cases}
        s\log\left(\frac{1}{s}\right) + (1 - s)\log\left(\frac{1}{1 - s}\right) &\text{if } 0 < s < 1, \\
        0 &\text{if }s = 0 \text{ or } s = 1.
    \end{cases}
\]

\begin{lemma}
\label{thm:objective-logistic-legendre}
    Fix any $\bm{\beta}^\star \in \R^d$ and let $\widehat{\bm{\beta}}$ be the output of \cref{alg:objective-perturbation} when instantiated with
    \[
        \ell(\bm{\beta}; (\bm{x}, y)) = \rho'(\langle \bm{x}, \bm{\beta}\rangle) - y\langle \bm{x}, \bm{\beta}\rangle.
    \]
    Let $Q_{\bm{\beta}, \bm{v}}$ be the following random variable indexed by $\bm{\beta} \in \R^d$ and $\bm{v} \in [0, 1]^n$:
    \[
        Q_{\bm{\beta}, \bm{v}} = \langle \bm{X} \bm{\beta}, \bm{v} - \bm{y}\rangle + \frac{\lambda}{2}\norm{\bm{\beta}}^2 + \nu \langle \bm{\xi}, \bm{\beta}\rangle - \rho^\star(\bm{v}).
    \] Let $\bm{\hv} = \rho'(\bm{X}\widehat{\bm{\beta}})$. Then,
    $(\widehat{\bm{\beta}}, \widehat{\bm{v}})$
    is the unique point in $\R^d \times [0, 1]^n$ satisfying
    \[
        \max_{\bm{v} \in [0, 1]^n}\, Q_{\widehat{\bm{\beta}}, \bm{v}} = \min_{\bm{\beta} \in \R^d}\, Q_{\bm{\beta}, \widehat{\bm{v}}}.
    \]
    We call it the saddle point or Nash equilibrium of $Q_{\bm{\beta}, \bm{v}}$.
\end{lemma}

\begin{proof}
    \cref{alg:objective-perturbation} simply computes
    \[\widehat{\bm{\beta}} = \argmin_{\bm{\beta} \in \R^d}\, \rho(\bm{X}\bm{\beta}) - \langle \bm{y}, \bm{X}\bm{\beta}\rangle + \frac{\lambda}{2}\norm{\bm{\beta}}^2 + \nu\langle \bm{\xi}, \bm{\beta}\rangle.\]
    We shall take the \emph{Legendre transform} of $\rho$ in the above expression. In our case, this simply amounts to applying the following identity regarding $\rho$, which is valid for all $\bm{t} \in \R^n$: \[\rho(\bm{t}) = \max_{\bm{v} \in [0, 1]^n}\, \langle \bm{\bm{t}, \bm{v}}\rangle - \rho^\star(\bm{v}).\] The maximum is achieved iff $\bm{v} = \rho'(\bm{t})$. Setting $\bm{t} = \bm{X}\bm{\beta}$,
    \[
        \widehat{\bm{\beta}} = \argmin_{\bm{\beta} \in \R^d} \max_{\bm{v} \in [0, 1]^n}\, \langle \bm{X}\bm{\beta}, \bm{v} - \bm{y}\rangle - \rho^\star(\bm{v}) + \frac{\lambda}{2}\norm{\bm{\beta}}^2 + \nu \langle \bm{\xi}, \bm{u}\rangle.
    \]
    Note that this objective function is simply $Q_{\bm{u}, \bm{v}}$. Given $\bm{\beta}$, the maximum over $\bm{v}$ is achieved iff $\bm{v} = \rho'(\bm{X}\bm{\beta})$, which coincides with $\bm{\hv}$ when $\bm{\beta} = \bm{\hb}$.
    The terms $\frac{\lambda}{2}\norm{\bm{\beta}}^2$ and $-\rho^\star(\bm{v})$ ensure that $Q_{\bm{\beta}, \bm{v}}$ is $\lambda$-strongly convex in $\bm{\beta} \in \R^d$ and $4$-strongly concave in $\bm{v} \in [0, 1]^n$, so we conclude by the minimax theorem that $(\widehat{\bm{\beta}}, \widehat{\bm{v}})$ is the unique saddle point of $Q_{\bm{\beta}, \bm{v}}$.
\end{proof}

Next, we apply a trick from \cite{salehi2019impact} to ensure that the random bilinear term is independent of the remaining terms---a prerequisite for applying CGMT. In order to state the lemma, given a ground-truth coefficient vector $\bm{\beta}^\star \in \R^d$, define the linear subspaces
\[
    \spn(\bm{\beta}^\star) = \{t \bm{\beta}^\star : t \in \R\},
    \qquad
    \spn(\bm{\beta}^\star)^\perp = \{\bm{\beta} \in \R^d : \langle \bm{\beta}, \bm{\beta}^\star\rangle = 0\}.
\]
Given vectors $\bm{\beta}^\| \in \spn(\bm{\beta}^\star)$ and $\bm{\beta}^\perp \in \spn(\bm{\beta}^\star)^\perp$, we will sometimes refer to their sum as
\[
    \bm{\beta} = \bm{\beta}^\| + \bm{\beta}^\perp.
\]
Conversely, given a vector $\bm{\beta} \in \R^d$, we will sometimes refer to its projections onto $\spn(\bm{\beta}^\star)$ and $\spn(\bm{\beta}^\star)^\perp$ by $\bm{\beta}^\|$ and $\bm{\beta}^\perp$, respectively. Note that the following lemma relies on the rotational invariance provided by Gaussianity.

\begin{lemma}
\label{thm:objective-logistic-independence}
    Suppose that $\bm{X} = \frac{1}{\sqrt{d}}\bm{G}$ with entries $G_{ij} \iid \mathcal{N}(0, 1)$. Let \(\bm{f} = \bm{G} \cdot \frac{\bm{\beta}^\star}{\norm{\bm{\beta}^\star}}\), let $\bm{H}$ be an independent copy of $\bm{G}$, and let $Q'_{\bm{\beta}, \bm{v}}$ be the following random variable indexed by $\bm{\beta} \in \R^d$ and $\bm{v} \in [0, 1]^n$:
    \[
        Q'_{\bm{\beta}, \bm{v}} = \frac{1}{\sqrt{d}}\langle \bm{H}\bm{\beta}^\perp, \bm{v} - \bm{y}\rangle + \frac{\norm{\bm{\beta}^\|}}{\sqrt{d}} \langle \bm{f}, \bm{v} - \bm{y}\rangle + \frac{\lambda}{2}\norm{\bm{\beta}}^2 + \nu \langle \bm{\xi}, \bm{\beta} \rangle - \rho^\star(\bm{v}).
    \]
    Then, for any closed sets $\cS_{\bm{\beta}} \subseteq \R^d$ and $\cS_{\bm{v}} \subseteq [0, 1]^n$, the random variable $\min_{\bm{\beta} \in \cS_{\bm{\beta}}} \max_{\bm{v} \in \cS_{\bm{v}}}\, Q_{\bm{\beta}, \bm {v}}$ has the same distribution as $\min_{\bm{\beta} \in \cS_{\bm{\beta}}} \max_{\bm{v} \in \cS_{\bm{v}}}\, Q'_{\bm{\beta}, \bm {v}}$. Moreover, $\bm{f} \sim \mathcal{N}(\bm{0}, \bm{I}_n)$ and
    \[
        \bm{y}| \bm{f} \sim \mathrm{Bernoulli}\mathopen{}\left(\rho'\mathopen{}\left(\frac{\norm{\bm{\beta}^\star}}{\sqrt{d}}\bm{f}\right)\mathclose{}\right)\mathclose{}.
    \]
\end{lemma}

\begin{proof}
    Substituting $\bm{X} = \frac{1}{\sqrt{d}}\bm{G}$ and $\bm{\beta} = \bm{\beta}^\perp + \bm{\beta}^\|$ into definition of $Q_{\bm{\beta}, \bm{v}}$ yields
    \[
        Q_{\bm{\beta}, \bm{v}} = \frac{1}{\sqrt{d}}\langle \bm{G}\bm{\beta}^\perp, \bm{v} - \bm{y}\rangle + \frac{1}{\sqrt{d}}\langle \bm{G}\bm{\beta}^\|, \bm{v} - \bm{y} \rangle + \frac{\lambda}{2}\norm{\bm{\beta}}^2 + \nu\langle \bm{\xi}, \bm{\beta}\rangle - \rho^\star(\bm{v}).
    \]
    Observe that $\bm{G}\bm{\beta}^\|$ and $\bm{y} \sim \mathrm{Bernoulli}(\rho'(\frac{1}{\sqrt{d}}\bm{G}\bm{\beta}^\star))$ both depend on $\bm{G}$ only through the projections of its rows onto $\spn(\bm{\beta}^\star)$. By Gaussianity, these projections are independent of their projections onto $\spn(\bm{\beta}^\star)^\perp$, justifying the replacement of $\bm{G}\bm{\beta}^\perp$ with $\bm{H}\bm{\beta}^\perp$. The facts that $\bm{f} \sim \mathcal{N}(\bm{0}, \bm{I}_n)$ and that $\bm{G}\bm{\beta}^\| = \norm{\bm{\beta}^\|}\bm{f}$ and $\bm{G}\bm{\beta}^\star = \norm{\bm{\beta}^\star}\bm{f}$ are also standard properties of multivariate Gaussians.
\end{proof}

\subsubsection{Step 3: CGMT Analysis}

In this step, we analyze the random variable $Q'_{\bm{u}, \bm{v}}$ using CGMT (\cref{thm:cgmt}), roughly following the strategy of \cite{salehi2019impact}.

As in the case of robust linear regression, our proof differs from past work in several ways. On the one hand, our proof is simpler, partly because it focuses on $\ell^2$ regularization and partly because we avoid introducing several extraneous scalar- and vector-valued variables. On the other hand, our proof is somewhat more complex, partly due to our consideration of the perturbation vector $\bm{\xi}$ introduced for differential privacy, and partly due to differences in the statement of our assumptions and conclusion, as well as our consideration of $\rho'(\bm{X}\bm{\hb})$ in addition to $\bm{\hb}$.

To begin, we recall and analyze the auxiliary random variable to which CGMT (\cref{thm:cgmt}) pertains. We only define this random variable in the case that $\bm{X} = \frac{1}{\sqrt{d}}\bm{G}$ and $\bm{f} = \bm{G} \cdot \frac{\bm{\beta}^\star}{\norm{\bm{\beta}^\star}}$, in which case by \cref{thm:objective-logistic-independence} we have that $\bm{f} \sim \mathcal{N}(\bm{0}, \bm{I}_n)$ and
\[
    \bm{y} \sim \mathrm{Bernoulli}\mathopen{}\left(\rho'\mathopen{}\left(\frac{\norm{\bm{\beta}^\star}}{\sqrt{d}}\bm{f}\right)\mathclose{}\right)\mathclose{}.
\]
In terms of $\bm{g} \sim \mathcal{N}(\bm{0}, \bm{I}_d)$ and $\bm{h} \sim \mathcal{N}(\bm{0}, \bm{I}_n)$, the auxiliary random variable is
\[
    Q''_{\bm{\beta}, \bm{v}} = \frac{\norm{\bm{\beta}^\perp}}{\sqrt{d}}\langle\bm{h}, \bm{v} - \bm{y}\rangle - \frac{\norm{\bm{v} - \bm{h}}}{\sqrt{d}}\langle\bm{g}, \bm{\beta}^\perp\rangle + \frac{\norm{\bm{\beta}^\|}}{\sqrt{d}} \langle \bm{f}, \bm{v} - \bm{y}\rangle + \frac{\lambda}{2}\norm{\bm{\beta}}^2 + \nu \langle \bm{\xi}, \bm{\beta} \rangle - \rho^\star(\bm{v}).
\]
Eventually, we will use CGMT (\cref{thm:cgmt}) to relate $Q''_{\bm{\beta}, \bm{v}}$ via $Q'_{\bm{\beta}, \bm{v}}$ to $Q_{\bm{\beta}, \bm{v}}$.

\begin{lemma}
\label{thm:logistic-objective-aux}
    Define $\alpha^\star, \sigma^\star, \gamma^\star > 0$ as in 
    \cref{thm:logistic-objective-perturbation}\ref{thm:logistic-objective-utility}, and consider the pair $(\bm{\tb}, \bm{\tv})$ with
    \begin{align*}
        \bm{\tb} &= \alpha^\star \bm{\beta}^\star + \sqrt{(\sigma^\star)^2 - (\gamma^\star \nu)^2} \bm{g} - \gamma^\star \nu \bm{\xi}, \\
        \bm{\tv} &= \bm{y} + \frac{1}{\gamma^\star}\left(\alpha^\star \kappa \bm{f} + \sigma^\star \bm{h}  - \prox_{\gamma^\star \rho}(\alpha^\star \kappa \bm{f} + \sigma^\star \bm{h} + \gamma^\star \bm{y})\right) \\
        &= \bm{y} + (-1)^{\bm{y}} \odot \rho'(\prox_{\gamma^\star \rho}((-1)^{\bm{y}} \odot (\alpha^\star\kappa \bm{f} + \sigma^\star \bm{h}))).
    \end{align*}
    Then, under the assumptions of \cref{thm:logistic-objective-perturbation}\ref{thm:logistic-objective-utility}, there exists a constant $c^\star \in \R$ such that w.h.p.
    \begin{itemize}
        \item The function $\bm{\beta} \mapsto Q''_{\bm{\beta}, \bm{\tv}}$ is $\lambda$-strongly convex in $\bm{\beta} \in \R^d$,
        \item The function $\bm{v} \mapsto Q''_{\bm{\tb}, \bm{v}}$ is $4$-strongly concave in $\bm{v} \in [0, 1]^n$,
        \item The pair $(\bm{\tb}, \bm{\tv})$ satisfies 
        \[
            c^\star n - n^{1 - \Omega(1)} \le \min_{\bm{\beta} \in \R^d} Q''_{\bm{\beta}, \widetilde{\bm{v}}} \le  Q''_{\widetilde{\bm{\beta}}, \widetilde{\bm{v}}} \le \max_{\bm{v} \in [0, 1]^n} Q''_{\widetilde{\bm{\beta}}, \bm{v}} \le c^\star n + n^{1 - \Omega(1)}.
        \]
    \end{itemize}
    We call $(\widetilde{\bm{\beta}}, \widetilde{\bm{v}})$ an \emph{approximate saddle point} of $Q''_{\bm{\beta}, \bm{v}}$.
\end{lemma}

\begin{proof}
    We first remark that in this proof, we will repeatedly use the identities $\prox_{\gamma\rho}(x + \gamma) = -\prox_{\rho}(-x)$ and $\prox'_{\gamma \rho}(x) = 1/(1 + \gamma\rho''(\prox_{\gamma\rho}(x)))$, which are proved in \cite{salehi2019impact}. 
    To show that $(\widetilde{\bm{\beta}}, \widetilde{\bm{v}})$ is an approximate saddle point of $Q''_{\bm{\beta}, \bm{v}}$, it suffices to show that each of $\widetilde{\bm{\beta}}$ and $\widetilde{\bm{v}}$ is an \emph{approximately best response} to the other, and that $\abs{Q''_{\widetilde{\bm{\beta}}, \widetilde{\bm{v}}} - c^\star n} \le n^{1 - \Omega(1)}$. Formally, we say that $\widetilde{\bm{\beta}}$ is an approximate best response to $\widetilde{\bm{v}}$ if
    \[
        Q''_{\widetilde{\bm{\beta}}, \widetilde{\bm{v}}} \le \min_{\bm{u} \in \R^d} Q''_{\bm{\beta}, \widetilde{\bm{v}}} + n^{1 - \Omega(1)}.
    \]
    Similarly, we say that $\widetilde{\bm{v}}$ is an approximate best response to $\widetilde{\bm{\beta}}$ if
    \[
        Q''_{\widetilde{\bm{\beta}}, \widetilde{\bm{v}}} \ge \max_{\bm{v} \in [0, 1]^n} Q''_{\widetilde{\bm{\beta}}, \bm{v}} - n^{1 - \Omega(1)}.
    \]
    In order to prove these inequalities, we will check that certain derivatives approximately vanish.
    \begin{enumerate}[(a)]
        \item \textbf{($\widetilde{\bm{\beta}}$ is an approximate best response to $\widetilde{\bm{v}}$)} We first study
        \(
            \min_{\bm{\beta} \in \R^d}\, Q''_{\bm{\beta}, \bm{\tv}}.
        \)
        If we can show that $Q''_{\bm{\beta}, \widetilde{\bm{v}}}$ is a $\lambda$-strongly convex function of $\bm{\beta}$, then minimizing $Q''_{\bm{\beta}, \widetilde{\bm{v}}}$ over $\bm{\beta} \in \R^d$ reduces to finding a point $\bm{\beta} \in \R^d$ at which the gradient $\nabla_{\bm{\beta}} Q''_{\bm{\beta}, \widetilde{\bm{v}}}$ has small norm. To this end, observe that the term $\frac{\lambda}{2}\norm{\bm{\beta}}^2$ in $Q''_{\bm{\beta}, \bm{\tv}}$ is $\lambda$-strongly convex in $\bm{\beta}$. The term $-\frac{\norm{\widetilde{\bm{v}} - \bm{y}}}{\sqrt{d}}\langle \bm{g}, \bm{\beta}^\perp\rangle$ is a linear function of $\bm{\beta}$, and hence convex. Finally, the term $\frac{\norm{\bm{\beta^\perp}}}{\sqrt{d}}\langle\bm{h},\widetilde{\bm{v}} - \bm{y}\rangle$ will be convex, as well, if we can show that $\langle \bm{h},\widetilde{\bm{v}} - \bm{y}\rangle$ is positive with high probability. Quantities like $\langle \bm{h},\widetilde{\bm{v}} - \bm{y}\rangle$ can be easily computed in the limit by combining the definitions of $\widetilde{\bm{\beta}}$ and $\widetilde{\bm{v}}$ in this lemma's statement with the assumptions on $\bm{f}$, $\bm{g}$, $\bm{h}$, $\bm{\beta}^\star$, $\bm{y}$, and $\bm{\xi}$. Indeed, by the definition of pseudo-Lipschitz convergence (Definition~\ref{def:pl-convergence}), if $f_0, g_0, h_0 \iid \mathcal{N}(0, 1)$ and $y_0|f_0 \sim \mathrm{Bernoulli}(\rho'(\kappa f_0))$, then w.h.p.,
        {\allowdisplaybreaks\begin{align*}
           \frac{1}{d}\langle \widetilde{\bm{\beta}}, \bm{\beta}^\star \rangle = \frac{1}{\sqrt{d}}\norm{\widetilde{\bm{\beta}}^\|} &=
           \alpha^\star\kappa \pm n^{-\Omega(1)}, \\
           \frac{1}{\sqrt{d}}\norm{\widetilde{\bm{\beta}}^\perp} &=
           \sigma^\star \pm n^{-\Omega(1)}, \\
           \frac{1}{d}\langle \bm{g}, \widetilde{\bm{\beta}}^\| \rangle &=
           0 \pm n^{-\Omega(1)}, \\
           \frac{1}{d}\langle \bm{g}, \widetilde{\bm{\beta}}^\perp \rangle &=
           \sqrt{(\sigma^\star)^2 - (\gamma^\star \nu)^2} \pm n^{-\Omega(1)}, \\
           \frac{1}{\sqrt{d}}\norm{\widetilde{\bm{v}} - \bm{y}} &=
           \sqrt{\frac{1}{\delta}\E[\rho'(\prox_{\gamma^\star \rho}((-1)^{y_0} (\alpha^\star\kappa f_0 + \sigma^\star h_0)))^2}] \pm n^{-\Omega(1)}, \\
           \frac{1}{d}\langle \bm{h}, \widetilde{\bm{v}} - \bm{y} \rangle &=
           \frac{1}{\gamma^\star \delta}\left(\sigma^\star - \E[\prox_{\gamma^\star \rho}(\alpha^\star \kappa f_0 + \sigma^\star h_0 + \gamma^\star y_0)h_0]\right) \pm n^{-\Omega(1)}, \\
           \frac{1}{d}\langle \bm{f}, \widetilde{\bm{v}} - \bm{y} \rangle &=
           \frac{1}{\gamma^\star \delta}\left(\alpha^\star\kappa - \E[\prox_{\gamma^\star \rho}(\alpha^\star \kappa f_0 + \sigma^\star h_0 + \gamma^\star y_0)f_0]\right) \pm n^{-\Omega(1)}.
        \end{align*}}
        We can simplify the expression for $\norm{\widetilde{\bm{v}} - \bm{y}}$ using the fact that $\Pr[y_0 = 0 \mid f_0] = \rho'(-\kappa f_0)$ and $\Pr[y_0 = 1 \mid f_0] = \rho'(\kappa f_0)$, followed by equation \eqref{eq:logistic-1} (see \cref{thm:logistic-objective-perturbation}\ref{thm:logistic-objective-utility}):
        \begin{align*}
            \frac{1}{\sqrt{d}}\norm{\widetilde{\bm{v}} - \bm{y}}
            &= \sqrt{\frac{1}{\delta}\E[2\rho'(-\kappa f_0)\rho'(\prox_{\gamma^\star \rho}(\alpha^\star\kappa f_0 + \sigma^\star h_0))^2}] \pm n^{-\Omega(1)} \\
            &= \sqrt{\left(\frac{\sigma^\star}{\gamma^\star}\right)^2 - \nu^2} \pm n^{-\Omega(1)}.
        \end{align*}
        Next, we can simplify the expression for $\langle \bm{f}, \bm{\tv} - \bm{y}\rangle$ using Stein's lemma, followed by the formula for $\prox'_{\gamma^\star \rho}$ given in \cref{sec:preliminaries}, and finally equations \eqref{eq:logistic-2} and \eqref{eq:logistic-3}:
        \begin{align*}
            \frac{1}{d}\langle \bm{f}, \widetilde{\bm{v}} - \bm{y} \rangle &=
            \frac{\kappa}{\gamma^\star \delta}\biggr(1 - \E[2\rho'(-\kappa Z_1)\prox'_{\gamma^\star \rho}(\alpha^\star \kappa Z_1 + \sigma^\star Z_2)] \\ &\qquad\qquad + \frac{1}{\alpha}\E[2\rho''(-\kappa f_0)\prox_{\gamma^\star \rho}(\alpha^\star\kappa f_0 + \sigma^\star h_0)]\biggr) \pm n^{-\Omega(1)} \\
            &= -\alpha^\star\kappa\lambda \pm n^{-\Omega(1)}.
        \end{align*}
        Similarly, we can simplify the expression for $\langle \bm{h}, \bm{\tv} - \bm{y} \rangle$ using Stein's lemma, followed by the formula for $\prox'_{\gamma^\star \rho}$ given in \cref{sec:preliminaries}, and finally equation \eqref{eq:logistic-3}:
        \begin{align*}
            \frac{1}{d}\langle \bm{h}, \widetilde{\bm{v}} - \bm{y} \rangle &=
            \frac{\sigma^\star}{\gamma^\star \delta}\left(1 - \E[2\rho'(-\kappa Z_1)\prox'_{\gamma^\star \rho}(\alpha^\star \kappa Z_1 + \sigma^\star Z_2)]\right) \pm n^{-\Omega(1)} \\
            &= \sigma^\star\left(\frac{1}{\gamma^\star} - \lambda\right) \pm n^{-\Omega(1)}.
        \end{align*}
        Again by equation \eqref{eq:logistic-3}, we see that $1/\gamma^\star > \lambda^\star$, so the quantity above is positive with high probability, which establishes that $Q''_{\bm{\beta}, \bm{\tv}}$ is $\lambda$-strongly convex in $\bm{\beta}$ with high probability. Therefore, all that remains is to evaluate its gradient at $\bm{\tb}$. To this end, we calculate
        \[
            \nabla_{\bm{\beta}^\perp} Q''_{\bm{\beta}, \bm{v}} = \frac{1}{\sqrt{d}} \frac{\bm{\beta}^\perp}{\norm{\bm{\beta}^\perp}} \langle \bm{h}, \bm{v} - \bm{y}\rangle - \frac{\norm{\bm{v - \bm{y}}}}{\sqrt{d}} \bm{g}^\perp + \lambda\bm{\bm{\beta}}^\perp + \nu\bm{\xi}^\perp.
        \]
        Evaluating at $(\widetilde{\bm{\beta}}, \widetilde{\bm{v}})$ and substituting our expressions for $\langle \bm{h}, \widetilde{\bm{v}} - \bm{y}\rangle$, $\norm{\bm{\tb}^\perp}$, and $\norm{\bm{\tv} - \bm{y}}$ yields
        \[
            \nabla_{\bm{\beta}^\perp} Q''_{\bm{\tb}, \bm{\tv}} = \left(\frac{1}{\gamma^\star} \pm n^{-\Omega(1)}\right)\bm{\tb}^\perp - \left(\sqrt{\left(\frac{\sigma^\star}{\gamma^\star}\right)^2 - \nu^2} \pm n^{-\Omega(1)}\right)\bm{g}^\perp + \nu\bm{\xi}^\perp.
        \]
        By substituting the definition of $\bm{\tb}$ in this lemma's statement and using the fact that for any constant $c > 0$, we have $\norm{\widetilde{\bm{\beta}}} + \norm{\bm{g}} + \norm{\bm{\xi}} \le O(n^{\frac{1}{2} + c})$ w.h.p. (\cref{thm:uniform-bound-from-convergence}), we see that the above gradient clearly has $\ell^2$ norm
        \[
            \norm{\nabla_{\bm{\beta}^\perp} Q''_{\widetilde{\bm{\beta}}, \widetilde{\bm{v}}}} \le n^{\frac{1}{2} - \Omega(1)}.
        \]
        Similarly, we calculate
        \[
            \nabla_{\bm{\beta}^\|} Q''_{\bm{\beta}, \bm{v}} = \frac{1}{\sqrt{d}} \frac{\bm{\beta}^\|}{\norm{\bm{\beta}^\|}} \langle \bm{f}, \bm{v} - \bm{y} \rangle + \lambda \bm{\beta}^\|.
        \]
        Evaluating at $(\widetilde{\bm{\beta}}, \widetilde{\bm{v}})$ and substituting our expressions for $\langle \bm{f}, \widetilde{\bm{v}} - \bm{y}\rangle$ and $\norm{\bm{\beta}^\|}$ yields
        \[
            \norm{\nabla_{\bm{\beta}^\|} Q''_{\widetilde{\bm{\beta}}, \widetilde{\bm{v}}}} \le n^{\frac{1}{2} - \Omega(1)}.
        \]
        Since we have shown that the derivatives of $Q''_{\bm{\beta}, \bm{\tv}}$ with respect to $\bm{\beta}^\|$ and $\bm{\beta}^\perp$ both have norm $n^{\frac{1}{2} - \Omega(1)}$, by $\lambda$-strong convexity in $\bm{\beta}$, we have with high probability that
        \[
            Q''_{\widetilde{\bm{\beta}}, \widetilde{\bm{v}}} \le \min_{\bm{\beta} \in \R^d} Q''_{\bm{\beta}, \widetilde{\bm{v}}} + n^{1 - \Omega(1)}.
        \]
        \item \textbf{($\widetilde{\bm{v}}$ is an approximate best response to $\widetilde{\bm{\beta}}$)} We study 
        \(
            \max_{\bm{v} \in [0, 1]^n} Q''_{\widetilde{\bm{\beta}}, \bm{v}}.
        \)
        Analogously to the previous part, we first verify $4$-strong concavity with respect to $\bm{v}$. To this end, observe that the term $-\rho^\star(\bm{v})$ in $Q''_{\bm{\tb}, \bm{v}}$ is $4$-strongly concave in $\bm{v}$. The term $\frac{\norm{\bm{\tb}}}{\sqrt{d}}\langle\bm{h},\bm{v} - \bm{y}\rangle$ is a linear function of $\bm{v}$, and hence concave. Finally, the term $-\frac{\norm{\bm{v} - \bm{y}}}{\sqrt{d}}\langle\bm{g}, \bm{\tb}\rangle$ is concave, as well, since we have already shown that $\langle \bm{g}, \widetilde{\bm{\beta}}\rangle$ is positive with high probability. Now, all that remains is to evaluate the gradient of $Q''_{\widetilde{\bm{\beta}}, \bm{v}}$ at $\widetilde{\bm{v}}$. To this end, we calculate
        \[
            \nabla_{\bm{v}}Q''_{\bm{\beta}, \bm{v}} = \frac{\norm{\bm{\beta}^\perp}}{\sqrt{d}}\bm{h} - \frac{1}{\sqrt{d}}\frac{\bm{v} - \bm{y}}{\norm{\bm{v} - \bm{y}}}\langle \bm{g}, \bm{\beta}\rangle + \frac{\norm{\bm{\beta}^\|}}{\sqrt{d}}\bm{f} - (\rho^\star)'(\bm{v}).
        \]
        Evaluating at $(\widetilde{\bm{\beta}}, \widetilde{\bm{v}})$ and substituting our expressions for $\norm{\bm{\beta}^\perp}$, $\norm{\bm{\tv} - \bm{y}}$, $\langle \bm{g}, \bm{\tb}\rangle$,  $\norm{\bm{\beta}^\|}$ yields
        \[
            \nabla_{\bm{v}}Q''_{\widetilde{\bm{\beta}}, \widetilde{\bm{v}}} = (\alpha^\star\kappa \pm n^{-\Omega(1)})\bm{f} + (\sigma^\star \pm n^{-\Omega(1)})\bm{h} + (\gamma^\star \pm n^{-\Omega(1)})\bm{y} - ((\gamma^\star \pm n^{-\Omega(1)})\bm{\tv} + (\rho^\star)'(\bm{\tv})).
        \]
        By substituting the definition of $\widetilde{\bm{v}}$ in this lemma's statement and using standard properties about the relationship between $\rho^\star$ and $\prox_{\gamma \rho}$ (see \cref{sec:preliminaries}), we see that
        \[
            \norm{\nabla_{\bm{v}} Q''_{\widetilde{\bm{\beta}}, \widetilde{\bm{v}}}} \le n^{\frac{1}{2} - \Omega(1)}.
        \]
        By $4$-strong concavity, we have with high probability that
        \[
            Q''_{\widetilde{\bm{\beta}}, \widetilde{\bm{v}}} \ge \max_{\bm{v} \in [0, 1]^n} Q''_{\widetilde{\bm{\beta}}, \bm{v}} - n^{1 - \Omega(1)}.
        \]
    \end{enumerate}
    To conclude the proof, we remark that plugging our estimates into the definition of $Q''_{\bm{\beta}, \bm{v}}$ similarly implies that there exists a constant $c^\star \in \R$ such that w.h.p.,
    \[
        \abs{Q''_{\widetilde{\bm{\beta}}, \widetilde{\bm{v}}} - c^\star n} \le n^{1 - \Omega(1)}.
    \]
\end{proof}

\subsubsection{Putting Steps 1, 2, and 3 Together}

\begin{proof}(\cref{thm:logistic-objective-perturbation})
    Note that part \ref{thm:logistic-objective-privacy} follows immediately from \cref{thm:our-objective-perturbation-privacy}, the observation that $\rho'' : \R \to [0, 1/4]$, and a change of variables in the case that $R \neq 1$. Therefore, we focus on part \ref{thm:logistic-objective-utility}.
    By \cref{thm:logistic-objective-aux}, there exist constants $c^\star \in \R$ and $c_{\mathrm{cgmt}} > 0$ such that w.h.p,
    \begin{equation*}
        c^\star n - O(n^{1 - c_{\mathrm{cgmt}}}) \le \min_{\bm{\beta} \in \R^d} Q''_{\bm{\beta}, \widetilde{\bm{v}}} \le  Q''_{\widetilde{\bm{\beta}}, \widetilde{\bm{v}}} \le \max_{\bm{v} \in [0, 1]^n} Q''_{\widetilde{\bm{\beta}}, \bm{v}} \le c^\star n + O(n^{1 - c_{\mathrm{cgmt}}}).
    \end{equation*}
    By \cref{thm:uniform-bound-from-convergence} and \cref{thm:objective-uniform-bound}, for any arbitrarily small constant $c_{\mathrm{diam}} > 0$, there exists an upper bound $L_{\bm{\beta}}= O(n^{\mathrm{c_{diam}}})$ such that w.h.p.,
    \[
        \bm{\tb}, \bm{\hb} \in [-L_{\bm{\beta}}, +L_{\bm{\beta}}]^d.
    \]
    For brevity, set $\cS_{\bm{\beta}} = [-L_{\bm{\beta}}, +L_{\bm{\beta}}]^d$ and $\cS_{\bm{v}} = [-1, +1]^n$. Then, $\bm{\tb} \in \cS_{\bm{\beta}}$ implies that w.h.p.,
    \begin{equation}
    \label{eq:logistic-aux-value-with-rate}
        \max_{\bm{v} \in \cS_{\bm{v}}} \min_{\bm{\beta} \in \cS_{\bm{\beta}}} \, Q''_{\bm{\beta}, \bm{v}} \le c^\star n + O(n^{1 - c_{\mathrm{cgmt}}}).
    \end{equation}
    Next, we would like to relate $Q''_{\bm{\beta}, \bm{v}}$ to $Q'_{\bm{\beta}, \bm{v}}$ via CGMT (\cref{thm:cgmt}). First, define the function
    \[
        \psi(\bm{\beta}^\perp, \bm{v}) = \min_{\substack{\bm{\beta}^\| \in (\cS_{\bm{\beta}} - \bm{\beta}^\perp) \cap \spn(\bm{\beta}^\star)}}\; \frac{\norm{\bm{\beta}^\|}}{\sqrt{d}} \langle \bm{f}, \bm{v} - \bm{y}\rangle + \frac{\lambda}{2}\norm{\bm{\beta}^\perp + \bm{\beta}^\|}^2 + \nu \langle \bm{\xi}, \bm{\beta}^\perp + \bm{\beta}^\| \rangle - \rho^\star(\bm{v}).
    \]
    By \cref{thm:objective-logistic-independence}, we can express $Q'_{\bm{\beta}, \bm{v}}$ in the form amenable to CGMT using $\psi$:
    \[
        \min_{\bm{\beta} \in \cS_{\bm{\beta}}} \max_{\bm{v} \in \cS_{\bm{v}}} \, Q'_{\bm{\beta}, \bm{v}} = \max_{\bm{v} \in \cS_{\bm{v}}}\; \min_{\bm{\beta}^\perp \in \cS_{\bm{\beta}} \cap \spn(\bm{\beta}^\star)^\perp}\; \frac{1}{\sqrt{d}}\langle \bm{H}\bm{\beta}^\perp, \bm{v} - \bm{y}\rangle + \psi(\bm{\beta}^\perp, \bm{v}).
    \]
    As per the discussion surrounding \cref{thm:objective-logistic-independence}, the function $\psi$ is independent of $\bm{H}$ (indeed, even though $\bm{y}$ appears in the definition of the $\psi$ function, $\bm{y} \sim \mathrm{Bernoulli}(\rho'(\frac{\norm{\bm{\beta}^\star}}{\sqrt{d}}\bm{f}))$ depends on the design matrix only through $\bm{f}$, not $\bm{H}$). Thus, by \eqref{eq:logistic-aux-value-with-rate} and CGMT (\cref{thm:cgmt}), we have that w.h.p.,
    \[
       \min_{\bm{\beta} \in \cS_{\bm{\beta}}} \max_{\bm{v} \in \cS_{\bm{v}}}\, Q'_{\bm{\beta}, \bm{v}} \le c^\star n + O(n^{1 - c_{\mathrm{cgmt}}}).
    \]
    By the definition of $\bm{\tb}$ in the statement of \cref{thm:logistic-objective-aux},
    along with our assumption that $\bm{\beta}^\star \rightsquigarrow \beta^\star_0$,
    \begin{equation*}
        (\bm{\beta}^\star, \bm{\xi}, \bm{\tb}) \rightsquigarrow \left(\beta_0^\star, \; \xi_0, \; \alpha^\star \beta^\star_0 + \sqrt{(\sigma^\star)^2 - (\gamma^\star \nu)^2} Z + \gamma^\star\nu \xi_0\right).
    \end{equation*}
    For brevity, let $\beta_0 \in \R$ denote the third random variable in the above triple. Then, the above assertion is that $(\bm{\beta}^\star, \bm{\xi}, \bm{\tb}) \rightsquigarrow (\beta^\star_0, \xi_0, \beta_0)$, and we want to show that $(\bm{\beta}^\star, \bm{\xi}, \bm{\hb}) \rightsquigarrow (\beta^\star_0, \xi_0, \beta_0)$, as well. To this end, we use the same ``excision'' technique as we did in the proof of \cref{thm:main-huber-objective-perturbation}\ref{thm:main-huber-objective-utility} in \cref{sec:huber-put-steps-together}: fix an order-$k$ pseudo-Lipschitz function $f : \R^3 \to \R$, and excise from $\cS_{\bm{\beta}}$ the open set $\cT_{\bm{\beta}}$ where $\bm{\beta} \in \cT_{\bm{\beta}}$ iff the average value of $f$ over the coordinates of $(\bm{\beta}^\star, \bm{\xi}, \bm{\beta})$ differs from the expected value of $f$ over the randomness of $(\beta^\star_0, \xi_0, \beta_0)$ by strictly less than $n^{-c'_{\mathrm{slack}}}$, for a sufficiently small constant $c'_{\mathrm{slack}}$. Note that this excision adjusts the definition of the mean function $\psi$, as well. Then, applying CGMT and universality as before, along with the strong convexity afforded by \cref{thm:logistic-objective-aux}, yields w.h.p.,
    \[
        \min_{\bm{\beta} \in \cS_{\bm{\beta}} \setminus \cT_{\bm{\beta}}} \max_{\bm{v} \in \cS_{\bm{v}}} Q'_{\bm{\beta}, \bm{v}} \ge c^\star n + \Omega(n^{1 - c_{\mathrm{total}}}),
    \]
    for a constant $0 < c_{\mathrm{total}} < c_{\mathrm{cgmt}}$. Thus, w.h.p.,
    \[
        \min_{\bm{\beta} \in \cS_{\bm{\beta}} \setminus \cT_{\bm{\beta}}} \max_{\bm{v} \in \cS_{\bm{v}}} Q'_{\bm{\beta}, \bm{v}} > \min_{\bm{\beta} \in \cS_{\bm{\beta}}} \max_{\bm{v} \in \cS_{\bm{v}}} Q'_{\bm{\beta}, \bm{v}}
    \]
    In other words, if $\bm{X} = \frac{1}{\sqrt{d}}\bm{G}$ is Gaussian, then w.h.p. the minimizer $\bm{\hb} \in \cS_{\bm{\beta}}$ lies inside $\cT_{\bm{\beta}}$, meaning that $(\bm{\beta}^\star, \bm{\xi}, \bm{\hb}) \rightsquigarrow(\beta^\star_0, \xi_0, \beta_0)$, as desired. This completes our characterization of the estimation error of $\bm{\hb}$ in the Gaussian case. For the difference $\rho'(\bm{X}\bm{\beta}^\star) - \rho'(\bm{X}\bm{\hb})$, we carry out an entirely analogous dual argument with the roles of $\bm{\beta}$ and $\bm{v}$ exchanged. As before, we start by noting that by CGMT, we have w.h.p. that
    \[
        \max_{\bm{v} \in \cS_{\bm{v}}} \min_{\bm{u} \in \cS_{\bm{u}}} Q'_{\bm{\beta}, \bm{v}} \ge c^\star n - O(n^{1 - c_{\mathrm{cgmt}}}).
    \]
    By the definition of $\bm{\tv}$ in the statement of \cref{thm:logistic-objective-aux}, along with the fact that $\bm{f} = \bm{G} \cdot \frac{\bm{\beta}^\star}{\norm{\bm{\beta}^\star}}$ and $\bm{\beta}^\star \rightsquigarrow \beta^\star_0$ with $\E(\beta^\star_0)^2 = \kappa^2$, we have for dummy variables $f_0, h_0 \iid \mathcal{N}(0, 1)$ and $y_0 | f_0 \sim \mathrm{Bernoullii}(\rho'(\kappa f_0)))$ that
    \[
        \left(\frac{\norm{\bm{\beta}^\star}}{\sqrt{d}}\bm{f}, \bm{\tv}\right) \rightsquigarrow \left(\kappa f_0, \; \alpha^\star \kappa f_0 + \sigma^\star h_0 + y_0 - \prox_{\gamma^\star \rho}(\alpha^\star \kappa f_0 + \sigma^\star h_0 + \gamma^\star y_0)\right).
    \]
    For brevity, let $v_0 \in \R$ denote the second random variable in the above pair. Then the above assertion is that $(\frac{\norm{\bm{\beta}^\star}}{\sqrt{d}}\bm{f}, \bm{\tv}) \rightsquigarrow (\kappa f_0, v_0)$, and we want to show that $(\bm{X}\bm{\beta}^\star, \rho'(\bm{X}\bm{\hb})) \rightsquigarrow (\kappa f_0, v_0)$. Note that $\frac{\norm{\bm{\beta}^\star}}{\sqrt{d}}\bm{f} = \bm{X}\bm{\beta}^\star$ when $\bm{X} = \frac{1}{\sqrt{d}}\bm{G}$. Also note that by \cref{thm:objective-logistic-legendre}, we have $\bm{\hv} = \rho'(\bm{X}\bm{\hb})$. Thus, applying the excision argument yet again to pass from $(\frac{\norm{\bm{\beta}^\star}}{\sqrt{d}}\bm{f}, \bm{\tv})$ and $(\frac{\norm{\bm{\beta}^\star}}{\sqrt{d}}\bm{f}, \bm{\hv})$ yields $(\bm{X}\bm{\beta}^\star, \rho'(\bm{X}\bm{\hb})) \rightsquigarrow (\kappa f_0, v_0)$, as desired.
    This concludes the proof in the Gaussian case, that $\bm{X} = \frac{1}{\sqrt{d}}\bm{G}$. For a general subgaussian design, we apply \cref{thm:logistic-universality}, along with the definition of slow pseudo-Lipschitz convergence (Definition~\ref{def:slow-pl-convergence}), to conclude that
    \[
        (\bm{\beta}^\star, \bm{\xi}, \widehat{\bm{\beta}} - \bm{\beta}^\star) \dashrightarrow (\beta^\star_0, \; \xi_0, \; \alpha^\star \beta^\star_0 + \sqrt{(\sigma^\star)^2 - (\gamma^\star \nu)^2} Z + \gamma^\star\nu \xi_0),
    \]
    and that for $Z_1 = f_0$ and $Z_2 = h_0$,
    \[
        (\bm{X}\bm{\beta}^\star, \, \rho'(\bm{X} \widehat{\bm{\beta}})) \dashrightarrow \left(\kappa Z_1, \; \alpha^\star \kappa Z_1 + \sigma^\star Z_2 + y_0 - \prox_{\gamma^\star \rho}(\alpha^\star \kappa Z_1 + \sigma^\star Z_2 + \gamma^\star y_0)\right).
    \]
\end{proof}

\section{Privacy of Objective Perturbation with Small \texorpdfstring{$\lambda$}{Lambda}}
\label{sec:objective-perturbation-privacy}

In this section, we present our improved privacy proof for objective perturbation (\cref{alg:objective-perturbation}), from which we derived Theorems \ref{thm:main-huber-objective-perturbation}\ref{thm:main-huber-objective-privacy} and \ref{thm:logistic-objective-perturbation}\ref{thm:logistic-objective-privacy}. The main result of this section, \cref{thm:our-objective-perturbation-privacy}, is an extension of \cref{thm:redberg-privacy} (Theorems 3.1 and 3.2 of \cite{redberg2023improving}). While \cref{thm:our-objective-perturbation-privacy} will apply to all strictly positive regularization strengths $\lambda > 0$ and perturbation magnitudes $\nu > 0$, \cref{thm:redberg-privacy} required that $\lambda > s$ for some strictly positive $s$ depending on the smoothness of the loss function. For example, $s = 1$ in the case of robust linear regression, and $s = 1/4$ in the case of logistic regression.

\begin{theorem}
\label{thm:our-objective-perturbation-privacy}
    Suppose that $\ell_0 : \R^2 \to \R$ satisfies $\abs{\partial_1 \ell_0(\eta, y)} \le L$ and $0 \le \partial_1^2 \ell_0(\eta, y) \le s$ for some constants $L, s > 0$ and for all $\eta, y \in \R$. Then objective perturbation (\cref{alg:objective-perturbation}) with $R = 1$, any $\lambda, \nu > 0$, the GLM loss function $\ell(\bm{\beta}; (\bm{x}, y)) = \ell_0(\langle \bm{x}, \bm{\beta}\rangle, y)$, and any strictly positive $\lambda, \nu > 0$, satisfies $(\pe, \pd)$-differential privacy for any $\pe \ge 0$ and
    \[
        \pd = \begin{cases}
            2 \cdot \mathrm{HockeyStick}(\tpe, \frac{L}{\nu}) &\text{if }\hpe \ge 0,\\
            (1 - e^{\hpe}) + 2e^{\hpe}\cdot \mathrm{HockeyStick}\left(\frac{L^2}{2\nu^2}, \frac{L}{\nu}\right) &\text{otherwise,}
        \end{cases}
    \]
    where we set $\tpe = \pe - \log(1 + s/\lambda)$ and $\hpe = \tpe - L^2/2\nu^2$.

    The algorithm also satisfies $(\pa, \pe)$-R\'{e}nyi differential privacy for any $\pa > 1$ and
    \[
        \pe = \log\left(1 + \frac{s}{\lambda}\right) + \frac{L^2}{2\nu^2} + \frac{1}{\pa - 1}\log \mathbb{E}_{X \sim \mathcal{N}{\left(0, \frac{L^2}{\nu^2}\right)}}\bigl[e^{(\pa - 1)\abs{X}}\bigr].
    \]

\end{theorem}

Before we prove \cref{thm:our-objective-perturbation-privacy}, we make a couple comments. First, the only differences between the statement of this theorem and the statement of \cref{thm:redberg-privacy} are our removal of the assumption that $\lambda > s$, and, relatedly, our replacement of the quantity $- \log(1 - s / \lambda)$ in both the approximate DP and RDP bounds with the strictly smaller quantity $\log(1 + s / \lambda)$. We also remark that the conclusion of \cref{thm:our-objective-perturbation-privacy} implies a $\pr$-zCDP bound with constant $\pr$ for any $\lambda,\nu > 0$, justifying our informal commentary in the main body of the paper:

\begin{corollary}
\label{thm:rho-zcdp-bound}
    Under the same conditions as \cref{thm:our-objective-perturbation-privacy}, \cref{alg:objective-perturbation} satisfies $\pr$-zCDP for
    \[
        \pr = \log{\left(1 + \frac{s}{\lambda}\right)} + \frac{L^2}{2\nu^2} + \sqrt{\frac{2}{\pi}} \cdot \frac{L}{\nu}.
    \]
\end{corollary}

\begin{proof}
    By the definition of zCDP (Definition~\ref{def:zcdp}), we must show that the expression for $\eps$ in the RDP bound of \cref{thm:our-objective-perturbation-privacy} is at most $\pr \cdot \pa$, for all $\pa > 1$. Using a standard formula for the MGF of the \emph{folded Gaussian distribution}, we see that if $X \sim \mathcal{N}(0, \sigma^2)$, then
    \[
        \log \E[e^{t\abs{X}}] = \frac{\sigma^2t^2}{2} + \log(2\Phi(\sigma t)).
    \]
    Substituting $\sigma = L/\nu$ and $t = \pa - 1$ yields $(\pa, \pe)$-RDP for all $\pa > 1$ and
    \[
        \pe = \log\left(1 + \frac{s}{\lambda}\right) + \frac{L^2}{2\nu^2} + \left(\frac{L^2}{2\nu^2}(\pa-1) + \frac{\log(2\Phi(\frac{L}{\nu}(\pa - 1)))}{\pa - 1}\right).
    \]
    Thus, the expression for $\pe$ has three terms: a constant term $\log(1 + s/\lambda)$, a linear term $(L^2/2\nu^2) \cdot \pa$, and third term $\log(2\Phi(\frac{L}{\nu}(\pa - 1)))/(\pa - 1)$, which one can check by hand strictly decreases from $\sqrt{2/\pi}(L/\nu)$ to $0$ as $\pa$ increases from $1$ to $\infty$. We conclude that $\pe \le \pr \cdot \pa$ for the claimed value $\pr$.
\end{proof}

\begin{proof}(\cref{thm:our-objective-perturbation-privacy})
    The main difference between this proof and the proof of \cref{thm:redberg-privacy} (Theorems 3.1 and 3.2 in \cite{redberg2023improving}) is our more careful bounding of a certain additive ``change-of-variables'' term in the privacy loss random variable. For clarity, we focus our attention on the part of the proof that changes.

    To begin, let $\bm{\hb} \in \R^d$ be the output of the algorithm when run on the data set $(\bm{X}, \bm{y}) \in \R^{n \times d} \times \R^n$. Let $(\bm{X}', \bm{y}') \in \R^{(n \pm 1) \times d} \times \R^{(n \pm 1)}$ denote an add/remove-one adjacent data set. Because of the random perturbation term $\bm{\xi} \in \R^d$ in the algorithm, $\bm{\hb}$ is a random variable even for fixed $\bm{X}$ and $\bm{y}$. Denote its density at $\bm{\beta} \in \R^d$ by $\mathsf{PDF}(\bm{\beta}; (\bm{X}, \bm{y}))$ and consider the \emph{privacy loss random variable}
    \[
        Z = \log{\left(\frac{\mathsf{PDF}(\bm{\hb}; (\bm{X}, \bm{y}))}{\mathsf{PDF}(\bm{\hb}; (\bm{X}', \bm{y}'))}\right)}.
    \]
    Note that both the numerator and the denominator involve $\bm{\hb}$, which is the output of the algorithm on $(\bm{X}, \bm{y})$, not $(\bm{X}', \bm{y}')$. In order to simplify $Z$, notice that there is a one-to-one correspondence between possible instantiations of the perturbation vector $\bm{\xi}$ and possible outputs $\bm{\beta}$ of the algorithm. Indeed, the algorithm outputs $\bm{\beta}$ on input $(\bm{X}, \bm{y})$ if and only if the chosen perturbation vector is
    \[
        g(\bm{\beta}; (\bm{X}, \bm{y})) = -\frac{\lambda}{\nu} \bm{\beta} - \frac{1}{\nu}\sum_{i=1}^n \partial_1 \ell_0 (\langle \bm{x}_i, \bm{\beta}\rangle, y_i) \cdot \bm{x}_i.
    \]
    Let $J_g(\bm{\beta}; (\bm{X}, \bm{y})) \in \R^{d \times d}$ denote the Jacobian matrix of the function $g$ with respect to $\bm{\beta}$, and let $\mathsf{PDF}_{\mathcal{N}}(\bm{z})$ denote the standard Gaussian density at $\bm{z} \in \R^d$. Using the change-of-variables formula, we can simplify the expression for $Z$ to
    \[
        Z = \underbrace{\log{\left(\frac{\abs{\det J_g(\bm{\hb}; (\bm{X}, \bm{y}))}}{\abs{\det J_g(\bm{\hb}; (\bm{X}', \bm{y}'))}}\right)}}_{(*)} + \underbrace{\log{\left(\frac{\mathsf{PDF}_{\mathcal{N}}(g(\bm{\hb}; (\bm{X}, \bm{y}))}{\mathsf{PDF}_{\mathcal{N}}(g(\bm{\hb}; (\bm{X}', \bm{y}'))}\right)}}_{(**)}.
    \]
    To analyze $(*)$, the log determinant ratio,
    observe first that the function $J_g$ satisfies
    \[
        -J_g(\bm{\beta}; (\bm{X}, \bm{y})) = \frac{\lambda}{\nu} \bm{I}_d + \frac{1}{\nu}\sum_{i=1}^n \partial_1^2 \ell_0(\langle \bm{x}_i, \bm{\beta}\rangle, y_i) \cdot \bm{x}_i\bm{x}_i^\top
    \]
    Recall that the data sets $(\bm{X}, \bm{y})$ and $(\bm{X}', \bm{y}')$ differ in a single data point, which is present in one data set but not the other. Denote this point by $(\bm{x}_0, y_0) \in \R^d \times \R$, where $\norm{\bm{x}_0} \le R = 1$. Suppose first that $(\bm{x}_0, y_0)$ is present in $(\bm{X}, \bm{y})$ but not $(\bm{X}', \bm{y}')$. Then, we have that
    \[
        -J_g(\bm{\beta}; (\bm{X}, \bm{y})) = -J_g(\bm{\beta}; (\bm{X}', \bm{y}')) + \frac{1}{\nu} \partial_1^2 \ell_0(\langle \bm{x}_0, \bm{\beta}\rangle, y_0) \cdot \bm{x}_0\bm{x}_0^\top.
    \]
    Thus, by the matrix determinant lemma, we have that
    \[
        \frac{\det(-J_g(\bm{\beta}; (\bm{X}, \bm{y})))}{\det(-J_g(\bm{\beta}; (\bm{X}', \bm{y}')))} = 1 + \frac{1}{\nu}\partial_1^2\ell_0(\langle \bm{x}_0, \bm{\beta}\rangle, y_0) \cdot \bm{x}_0^\top (-J_g(\bm{\beta}; (\bm{X}', \bm{y}')))^{-1} \bm{x}_0.
    \]
    Observe that $-J_g(\bm{\beta}; (\bm{X}', \bm{y}')) \succeq \frac{\lambda}{\nu}\bm{I}_d$, that $\norm{\bm{x}_0} \le 1$, and that $\partial_1^2\ell(\langle \bm{x}_0, \bm{\beta}\rangle, y_0) \in [0, s]$. Using these three bounds, it follows that
    \[
        1 \le \frac{\det(-J_g(\bm{\beta}; (\bm{X}, \bm{y})))}{\det(-J_g(\bm{\beta}; (\bm{X}', \bm{y}')))} \le 1 + \frac{s}{\lambda}.
    \]
    Similarly, in the case that $(\bm{x}_0, y_0)$ is present in $(\bm{X}', \bm{y}')$ but not $(\bm{X}, \bm{y})$, we have that
    \[
         \left(1 + \frac{s}{\lambda}\right)^{-1} \le \frac{\det(-J_g(\bm{\beta}; (\bm{X}, \bm{y})))}{\det(-J_g(\bm{\beta}; (\bm{X}', \bm{y}')))} \le 1.
    \]
    In either case, we have that $\abs{(*)} \le \log(1 + s/\lambda)$ in absolute value, improving on the $-\log(1 - s/\lambda)$ upper bound on $\abs{(*)}$ from the proof of \cref{thm:redberg-privacy} (Theorems 3.1 and 3.2 of \cite{redberg2023improving}), which only held for $\lambda > s$. The remainder of the proof, which involves bounding $(**)$ and then analyzing various R\'{e}nyi and hockey stick divergences, is identical to the proof of \cref{thm:redberg-privacy}, and we omit it for brevity.
\end{proof}

\section{Utility of Output Perturbation}
\label{sec:output-perturbation}

In this section, we consider two immediate corollaries of \cref{thm:main-huber-objective-perturbation} and \cref{thm:logistic-objective-perturbation}, respectively. These corollaries give us the privacy-utility tradeoffs for \emph{output perturbation} (\cref{alg:output-perturbation}), when applied to the problems of robust linear regression and logistic regression. To arrive at these corollaries, we simply observe that output perturbation can be viewed as objective perturbation with $\nu = 0$ (i.e. no linear perturbation term), plus an additive Gaussian noise term to the output $\bm{\hb}$. Privacy follows immediately from \cref{thm:hockey-stick} (the analytic Gaussian mechanism of \cite{balle2018analytic}), along with the $\ell^2$ sensitivity bound of $L/\lambda$ afforded by $L$-Lipschitzness and $\lambda$-strong convexity. In Figure \ref{fig:output-estimation}, we validate the predictions of these corollaries against random, synthetic data, showing that output perturbation essentially incurs an additive increase in error that shifts the entire error curve of the non-private case ``up'' by $\nu^2$.

\begin{algorithm}[t]
    \caption{Output Perturbation}
    \label{alg:output-perturbation}
    \begin{algorithmic}[1]
        \State \textbf{input:} design matrix $\bm{X} \in \R^{n \times d}$ with $\norm{\bm{x}_i} \le R$, response vector $\bm{y} \in \R^n$, loss function $\ell : \R^d \times \R^{d + 1} \to \R$, regularization strength $\lambda > 0$, perturbation strength $\nu > 0$.

        \vspace{0.5cm}
        
        \State Randomly sample the additive perturbation term: \[\bm{\xi} \sim \mathcal{N}(\bm{0}, \bm{I}_d)\]
        
        \State Optimize the regularized objective function: \[\widetilde{\bm{\beta}} = \argmin_{\bm{\beta} \in \R^d}\, \sum_{i=1}^n \ell(\bm{\beta}; (\bm{x}_i, y_i)) + \frac{\lambda}{2}\norm{\bm{\beta}}^2\]
        
        \State Add the perturbation: \[\widehat{\bm{\beta}} = \widetilde{\bm{\beta}} + \nu\bm{\xi}\]
        
        \State \Return $\widehat{\bm{\beta}}$
    \end{algorithmic}
\end{algorithm}

\begin{corollary}[Output Perturbation for Robust Linear Regression]
\label{thm:main-huber-output-perturbation}
    Let $\bm{\hb}$ denote the output of \cref{alg:output-perturbation} with parameters $R, \lambda, \nu > 0$ and instantiated with the $L$-Lipschitz Huber loss function \[\ell(\bm{\beta}; (\bm{x}, y)) = H_L(y - \langle \bm{x}, \bm{\beta}\rangle).\]
    \begin{enumerate}[(a)]
        \item \emph{\textbf{(Privacy)}} $\widehat{\bm{\beta}}$ satisfies $(\pe, \pd)$-differential privacy for any $\pe \ge 0$ and
        \[
            \pd = \mathrm{HockeyStick}\mathopen{}\left(\pe, \frac{LR}{\lambda \nu}\right)\mathclose{}.
        \]
        \item \label{thm:main-huber-output-utility} \emph{\textbf{(Utility)}} Suppose the following hold for some $\bm{\beta}^\star \in \R^d$ and $\bm{\varepsilon}^\star \in \R^n$ as $n \to \infty$ and $d/n \to \delta$:
        \begin{enumerate}[(i)]
            \item $\bm{X} \in B_R(\bm{0})^n \subseteq \R^{n \times d}$ follows a subgaussian design and $\bm{y} = \bm{X}\bm{\beta}^\star + \bm{\varepsilon}^\star$.
            \item There exist random variables $\beta^\star_0, \varepsilon^\star_0\in \R$ such that $\bm{\beta}^\star \rightsquigarrow \beta^\star_0$ and $\bm{\varepsilon}^\star \rightsquigarrow \varepsilon^\star_0$.
        \end{enumerate}
        
        Suppose there exist $\sigma^\star, \tau^\star > 0$ solving the following system of two scalar equations in two variables $(\sigma, \tau)$, which we write in terms of a dummy variable $Z \sim \mathcal{N}(0, 1)$ and $\kappa^2 = \E(\beta^\star_0)^2$ as
        \begin{subequations}
        \begin{align}
            \sigma^2 &= \tau^2\left(\frac{1}{\delta} \E\mathopen{}\left[\frac{\sigma Z + \varepsilon^\star_0}{1 + \tau}\right]_L^2 + \lambda^2\kappa^2\right), \label{eq:huber-1-out} \\
            \tau &= \frac{1}{\lambda \delta}\left(\delta - \frac{\tau}{1+\tau}\Pr\mathopen{}\left[-L < \frac{\sigma Z + \varepsilon^\star_0}{1 + \tau} < L\right]\right). \label{eq:huber-2-out}
        \end{align}
        \end{subequations}
        Then, in terms of $(\sigma^\star, \tau^\star)$, the estimation error $\widehat{\bm{\beta}} - \bm{\beta}^\star$ satisfies, for $\xi_0 \sim \mathcal{N}(0, 1)$,
        \[
            (\bm{\beta}^\star, \, \bm{\xi}, \, \widehat{\bm{\beta}} - \bm{\beta}^\star) \rightsquigarrow \left(\beta_0^\star, \; \xi_0, \; \tau^\star{\left(\sqrt{\frac{1}{\delta} \E\mathopen{}\left[\frac{\sigma^\star Z + \varepsilon^\star_0}{1 + \tau^\star}\right]_L^2} Z - \lambda \beta^\star_0\right)} + \nu\xi_0\right).
        \]        
    \end{enumerate}
\end{corollary}

\begin{corollary}[Output Perturbation for Logistic Regression]
\label{thm:logistic-output-perturbation}
    Let $\bm{\hb}$ be the output of \cref{alg:output-perturbation} with parameters $R, \lambda, \nu > 0$ and instantiated with the logistic loss function \[\ell(\bm{\beta}; (\bm{x}, y)) = \rho(\langle \bm{x}, \bm{\beta}\rangle) - y\langle\bm{x}, \bm{\beta}\rangle.\]
    \begin{enumerate}[(a)]
        \item \emph{\textbf{(Privacy)}} $\bm{\hb}$ satisfies $(\pe, \pd)$-differential privacy for any $\pe \ge 0$ and
        \[
            \pd = \mathrm{HockeyStick}\mathopen{}\left(\pe, \frac{R}{\lambda \nu}\right)\mathclose{}.
        \]
        \item \label{thm:logistic-output-utility} \emph{\textbf{(Utility)}} Suppose the following assumptions hold for some $\bm{\beta}^\star \in \R^d$ as $n \to \infty$ and $d/n \to \delta$:
        \begin{enumerate}[(i)]
            \item $\bm{X} \in B_R(\bm{0})^n \subseteq \R^{n \times d}$ follows a subgaussian design and $\bm{y} \sim \mathrm{Bernoulli}(\rho'(\bm{X}\bm{\beta}^\star))$.
            \item There exists a random variable $\beta^\star_0 \in \R$ with $\kappa^2 = \E(\beta^\star_0)^2$ such that $\bm{\beta}^\star \rightsquigarrow \beta^\star_0$.
        \end{enumerate}
        Note that $\bm{\xi} \rightsquigarrow \xi_0$ for $\xi_0 \sim \mathcal{N}(0, 1)$.
        
        Suppose there are unique $\sigma^\star, \alpha^\star, \gamma^\star > 0$ solving the following system of three scalar equations in three variables $(\sigma, \alpha, \gamma)$, which we write in terms of dummy variables $Z_1, Z_2 \iid \mathcal{N}(0, 1)$ as
        \begin{subequations}
        \begin{align}
            \sigma^2 &= \frac{\gamma^2}{\delta}\E\mathopen{}\left[2\rho'(-\kappa Z_1)\rho'(\prox_{\gamma\rho}(\kappa\alpha Z_1+\sigma Z_2) \big)^2\right], \label{eq:logistic-1-out} \\
            \alpha &= -\frac{1}{\delta}\E[ 2\rho''(-\kappa Z_1)\prox_{\gamma \rho}\big(\kappa \alpha Z_1+\sigma Z_2\big)], \label{eq:logistic-2-out} \\
            \gamma &= \frac{1}{\lambda \delta}\mathopen{}\left(\delta - 1 + \E\mathopen{}\left[\frac{2\rho'(-\kappa Z_1)}{1+\gamma \rho''\big(\prox_{\gamma\rho}(\kappa\alpha Z_1 + \sigma Z_2)\big)}\right]\right). \label{eq:logistic-3-out}
        \end{align}
        \end{subequations}
        Then, the estimation error $\widehat{\bm{\beta}} - \bm{\beta}^\star$ satisfies
        \[
            (\bm{\beta}^\star, \, \bm{\xi}, \, \widehat{\bm{\beta}}) \dashrightarrow \left(\beta^\star_0, \, \xi_0, \, \alpha^\star \beta^\star_0 + \sigma^\star Z + \nu \xi_0\right).
        \]
    \end{enumerate}
\end{corollary}

\begin{figure}
    \centering
    
    \includegraphics[width=0.495\textwidth]{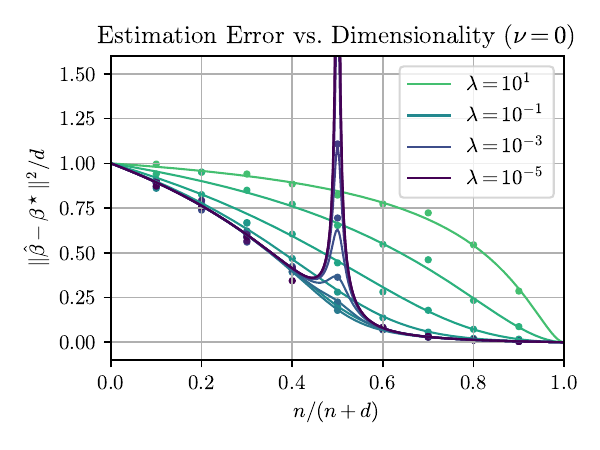}
    \includegraphics[width=0.495\textwidth]{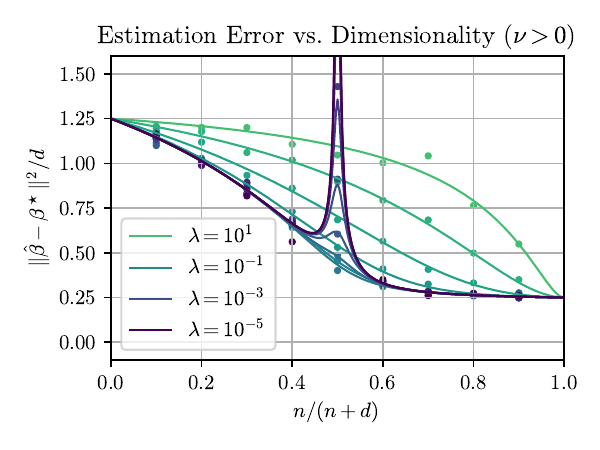}
    (a) \cref{alg:output-perturbation} with Huber loss ($L = 10$) and $\bm{\eps}^\star \sim \mathcal{N}(\bm{0}, (1/5)^2\bm{I}_n)$ and $\bm{y} = \bm{X}\bm{\beta}^\star + \bm{\eps}^\star$.
    \label{fig:huber-output-estimation}
    
    \vspace{0.5cm}
    
    \includegraphics[width=0.495\textwidth]{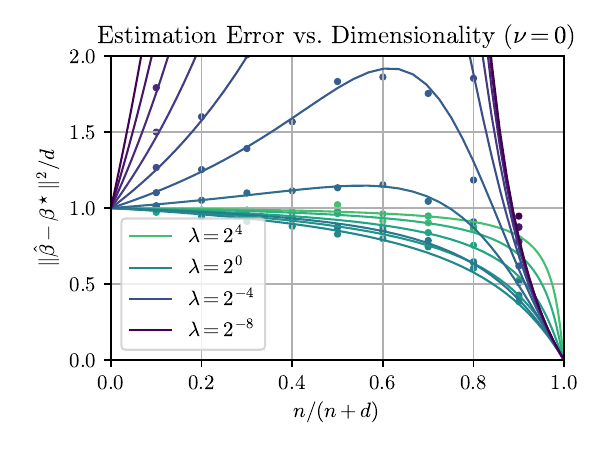}
    \includegraphics[width=0.495\textwidth]{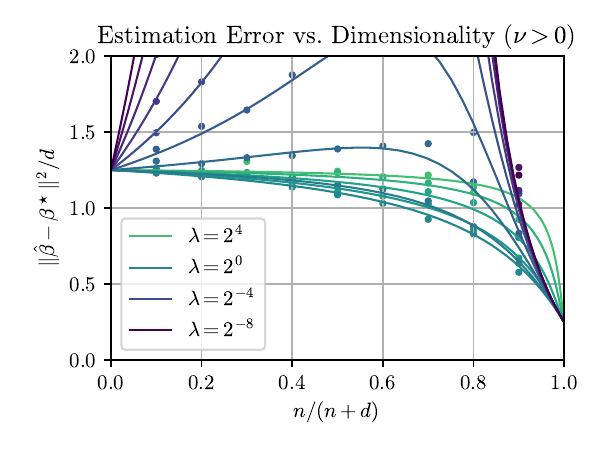}
    (b) \cref{alg:output-perturbation} with logistic loss and $\bm{y} \sim \mathrm{Bernoulli}(\rho'(\bm{X}\bm{\beta}^\star))$.
    \label{fig:logsitic-output-estimation}
    
    \vspace{0.5cm}
    
    \caption{Predictions of Corollaries \ref{thm:main-huber-output-perturbation} and \ref{thm:logistic-output-perturbation} on the estimation error of \cref{alg:output-perturbation}. In all plots, curves correspond to theoretical predictions, and dots correspond to the mean over $100$ simulations of the algorithm on synthetic data with $n \times d = 1000$. In the left plots, the perturbation magnitude is $\nu = 0$, but in the right plots, $\nu = 1/2$. In all plots, the signal strength is $\kappa = 1$, and we consider $\bm{\beta}^\star \sim \mathcal{N}(\bm{0}, \kappa^2\bm{I}_d)$, along with $\bm{X} \sim \frac{1}{\sqrt{d}} \mathrm{Uniform}(\{-1,+1\}^{n \times d})$.}
    \label{fig:output-estimation}
\end{figure}

\section{Utility of Noisy Stochastic Gradient Descent}
\label{sec:dp-sgd}

\begin{figure}
    \centering
    
    \includegraphics[width=0.495\textwidth]{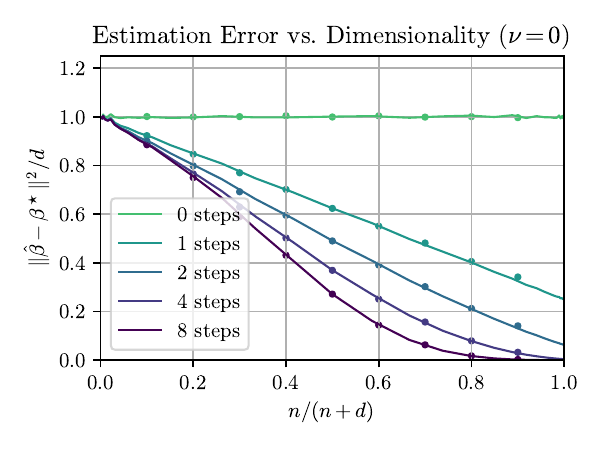}
    \includegraphics[width=0.495\textwidth]{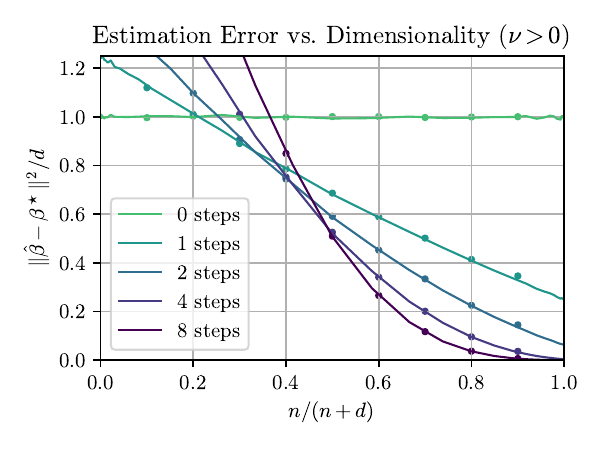}
    (a) \cref{alg:dp-sgd} with Huber loss, ``conditional expectation'' version ($L = 10$, $\eps^\star_0 \sim \mathcal{N}(0, (1/5)^2)$).
    
    \vspace{0.5cm}

    \includegraphics[width=0.495\textwidth]{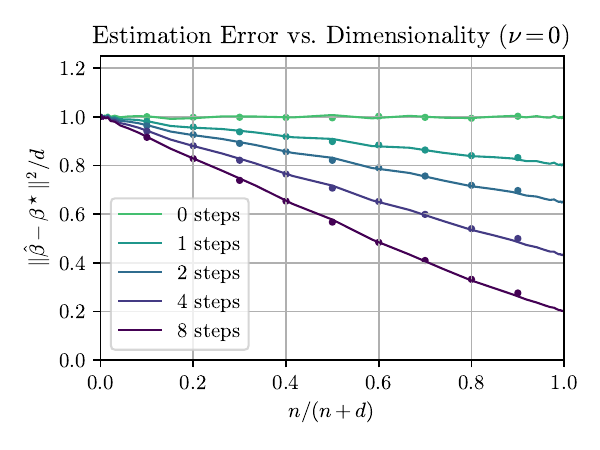}
    \includegraphics[width=0.495\textwidth]{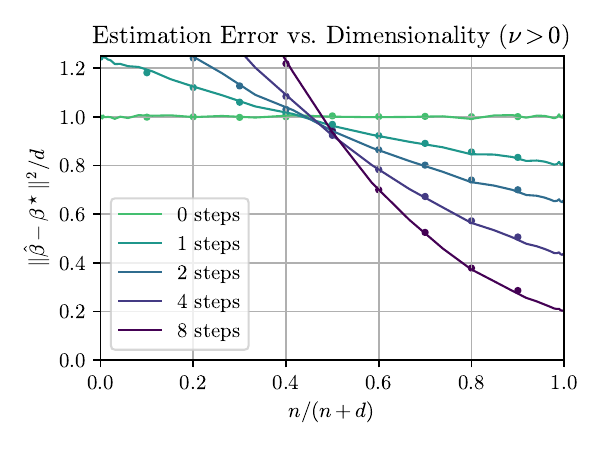}
    (b) \cref{alg:dp-sgd} with logistic loss, ``conditional expectation'' version.
    
    \vspace{0.5cm}
    
    \caption{Predictions of Theorems \ref{thm:dp-sgd-huber} and \ref{thm:dp-sgd-logistic} on the estimation error of \cref{alg:dp-sgd}. In all plots, curves correspond to theoretical predictions, and dots correspond to the mean over $10^4$ simulations of the algorithm on synthetic data with $n \times d = 1000$. In the left plots, the noise magnitude is $\nu = 0$, but in the right plots, $\nu = 1/10$. In all plots, the signal strength is $\kappa = 1$, the step size is $\gamma = 1/2 \cdot 1/(1+\delta)$, and we consider $\bm{\beta}^\star \sim \mathcal{N}(\bm{0}, \kappa^2\bm{I}_d)$, along with $\bm{X} \sim \frac{1}{\sqrt{d}} \mathrm{Uniform}(\{-1,+1\}^{n \times d})$ and $\bm{y} = \bm{X}\bm{\beta}^\star$.}
    \label{fig:dp-sgd}
\end{figure}

\begin{algorithm}[t]
    \caption{Noisy Gradient Descent}
    \label{alg:dp-sgd}
    \begin{algorithmic}[1]
        \State \textbf{input:} design matrix $\bm{X} = [\bm{x}_1\,\cdots\,\bm{x}_n]^\top \in \R^{n \times d}$ with $\norm{\bm{x}_i} \le R$, response vector $\bm{y} \in \R^n$, loss function $\ell : \R^d \times \R^{d + 1} \to \R$, step size $\gamma > 0$, noise magnitude $\nu > 0$.

        \vspace{0.5cm}

        \State Initialize \(\bm{\beta}^{(0)} = \bm{0}\)

        \For{$t = 0, \ldots, T - 1$}
            \State Sample the $t$\textsuperscript{th} gradient perturbation vector:
            \[
                \bm{\xi}^{(t)} \sim \mathcal{N}(\bm{0}, \bm{I}_d)
            \]
            \State Take a noisy gradient step:
            \[
                \bm{\beta}^{(t + 1)} = \bm{\beta}^{(t)} - \gamma\left(\sum_{i=1}^n \nabla_{\bm{\beta}}\,\ell(\bm{\beta}^{(t)};(\bm{x}_i, y_i)) + \nu \bm{\xi}^{(t)}\right)
            \]
        \EndFor
        
        \State \Return $\widehat{\bm{\beta}} = \bm{\beta}^{(T)}$
    \end{algorithmic}
\end{algorithm}

In this section, we present initial results for noisy stochastic gradient descent (DP-SGD) in the proportional regime that parallel our main results for objective perturbation (Theorems \ref{thm:main-huber-objective-perturbation} and \ref{thm:logistic-objective-perturbation}) and for output perturbation (Corollaries \ref{thm:main-huber-output-perturbation} and \ref{thm:logistic-output-perturbation}). The results in this section follow immediately from existing results in the non-private literature \cite{gerbelot2024meanfield, han2024entrywise}, and as such, we consider the results of this section as a valuable point of reference for situating our other results, rather than a main contribution in their own right.

Due to certain separability requirements on the loss function in \cite{gerbelot2024meanfield}, which builds on \cite{celentano2020gfom}, as well as certain smoothness requirements in both of \cite{celentano2020gfom, gerbelot2024meanfield}, we cannot use these results in a black-box fashion to directly analyze the utility of DP-SGD on the robust linear regression and logistic regression tasks that we focused on in the other sections of this paper. Instead, we shift our attention to the following, non-standard, ``conditional expectation'' versions of the problems, which satisfy the necessary separability and smoothness requirements.

For robust linear regression, we previously assumed a labeled data set $(\bm{X}, \bm{y})$ with $\bm{y} = \bm{X}\bm{\beta}^\star + \bm{\eps}^\star$ for suitably bounded $\bm{\eps}^\star$ and used the loss function $H_L(\bm{y} - \bm{X}\bm{\beta})$. We will now instead assume a labeled data set $(\bm{X}, \bm{y})$ with  $\bm{y} = \bm{X}\bm{\beta}^\star$ and use the ``conditional expectation'' loss function \[\ell_{\mathrm{RobustLinearCE}}(\bm{\beta}; (\bm{X}, \bm{y})) = \mathbb{E}_{\bm{\eps}^\star}[H_L(\bm{y} + \bm{\eps}^\star - \bm{X}\bm{\beta}) \mid \bm{X}, \bm{\beta}, \bm{\beta}^\star].\] For logistic regression, we previously assumed a labeled data set with $\bm{y} \sim \mathrm{Bernoulli}(\rho'(\bm{X}\bm{\beta}^\star))$ and used the loss function $\rho(\bm{X}\bm{\beta}) - \langle \bm{y}, \bm{X}\bm{\beta}\rangle$. We will now instead assume  $\bm{y} = \bm{X}\bm{\beta}^\star$ and use the ``conditional expectation'' loss function \[\ell_{\mathrm{LogisticCE}}(\bm{\beta}; (\bm{X}, \bm{y})) = \mathbb{E}_{\bm{\hat{y}} \sim \mathrm{Bernoulli}(\rho'(\bm{y}))}[\rho(\bm{X}\bm{\beta}) - \langle \bm{\hat{y}}, \bm{X}\bm{\beta}\rangle \mid \bm{X}, \bm{\beta}, \bm{\beta}^\star].\] These conditional expectation-based versions are admittedly unwieldy, but have the advantage of rigorously satisfying the assumptions needed to apply the results of \cite{celentano2020gfom, gerbelot2024meanfield}, because they are coordinate-wise separable and vary smoothly in $\bm{X}\bm{\beta}$ and $\bm{X}\bm{\beta}^\star$. For example, although \cite{celentano2020gfom, gerbelot2024meanfield} both observe that their theorems happen to give correct predictions about the behavior of (non-private) SGD for the standard formulation of logistic regression, they acknowledge that their theorems do not technically apply to this loss function because it takes as input $\bm{y} = \bm{1}[\bm{X}\bm{\beta}^\star + \bm{\eps}^\star > 0] \in \{0, 1\}^n$ for $\bm{\eps}^\star \sim \mathrm{Logistic}$, which is discontinuous with respect to $\bm{X}\bm{\beta}^\star$.

For simplicity, we consider private, full batch gradient descent (\cref{alg:dp-sgd}), although we could have just as easily analyzed the stochastic version for a mini-batch of size $\Omega(n)$ using the same results from \cite{gerbelot2024meanfield, han2024entrywise}. We validate the predictions of the following theorems against simulated data in Figure \ref{fig:dp-sgd}.

\begin{theorem}[DP-SGD for Robust Linear Regression, ``Conditional Expectation'' Version]
\label{thm:dp-sgd-huber}
Let $\bm{\hb}$ be the output of $T$ iterations of \cref{alg:dp-sgd} with parameters $R, \gamma, \nu > 0$ and instantiated with the loss function
\[
    \ell(\bm{\beta}; (\bm{x}, y)) = \E_{\eps^\star_0}\bigl[H_L(y + \eps^\star_0 - \langle \bm{x}, \bm{\beta}\rangle) \mid \bm{\beta}, (\bm{x}, y)\bigr]
\]
for some fixed, continuous random variable $\eps_0^\star \in \R$.
\begin{enumerate}[(a)]
    \item \label{thm:dp-sgd-huber-privacy} \emph{\textbf{(Privacy)}} $\bm{\hb}$ satisfies $\pr$-zCDP for $\pr = T \cdot \frac{L^2R^2}{2\nu^2}$.
    \item \label{thm:dp-sgd-huber-utility} \emph{\textbf{(Utility)}} Suppose the following assumptions hold for some $\bm{\beta}^\star \in \R^d$ as $n \to \infty$ and $d/n \to \delta$:
    \begin{enumerate}[(i)]
        \item $\bm{X} \in B_R(\bm{0})^n \subseteq \R^{n \times d}$ follows a subgaussian design and $\bm{y} = \bm{X}\bm{\beta}^\star$.
        \item There exists a random variable $\beta^\star_0 \in \R$ such that $\bm{\beta}^\star \rightsquigarrow \beta^\star_0$.
    \end{enumerate}
    Consider the following $O(T^2)$ equations that recursively define, for all $t, s \in \{0, 1, \ldots, T\}$, the random variables $\bm{\theta}^{(t)}, \bm{\eta}^{(t)}, \bm{u}^{(t)}, \bm{\omega}^{(t)} \in \R^2$ and $\partial \bm{\eta}^{(t)}/\partial \bm{\omega}^{(s)} \in \R^{2 \times 2}$, as well as the deterministic $2 \times 2$ matrices $\bm{R_\theta}(t, s), \bm{R_g}(t, s), \bm{\Gamma}^{(t)}, \bm{C_\theta}(t, s), \bm{C_g}(t, s) \in \R^{2 \times 2}$.

    The random variables are
    {\allowdisplaybreaks\begin{align*}
        \bm{\theta}^{(0)} &= \left[\begin{array}{c}
            0 \\
            \beta^\star_0
        \end{array}\right] \in \R^2, \\
        \bm{\theta}^{(t + 1)} &= (1 + \bm{\Gamma}^{(t)})\bm{\theta}^{(t)} - \gamma \nu \left[\begin{array}{c} \xi^{(t)}_0 \\ 0\end{array}\right] + \sum_{k=0}^{t - 1}\bm{R_g}(t, k)\bm{\theta}^{(k)} + \bm{u}^{(t)} \in \R^2, \\
        \bm{\eta}^{(t)} &= -\gamma\sum_{k=0}^{t-1} \bm{R_\theta}(t, k) \left[\begin{array}{c}
            \E_{\eps^\star_0}[\eta^{(k)}_1 - \eta^{(k)}_2 - \eps^\star_0]_L \\
            0
        \end{array}\right] + \bm{\omega}^{(t)} \in \R^2, \\
        \frac{\partial \bm{\eta}^{(t)}}{\partial \bm{\omega}^{(s)}} &= \left\{\begin{array}{ll}
            \bm{0} &\text{if } t < s, \\
            \bm{I}_2 &\text{if } t = s, \\
            -\gamma\sum_{k=0}^{t-1}\bm{R_\theta}(t, k)\Pr_{\eps^\star_0}\left[\abs{\eta^{(k)}_1 - \eta^{(k)}_2 - \eps^\star_0} < L\right]\left[\begin{array}{cc}
            1 & -1 \\
            0 & 0
        \end{array}\right] \frac{\partial \bm{\eta}^{(k)}}{\partial \bm{\omega}^{(s)}} &\text{if } t > s
        \end{array}\right\} \in \R^{2 \times 2},
    \end{align*}}
    where $\xi^{(0)}_0, \ldots, \xi^{(T)}_0 \iid \mathcal{N}(0, 1)$, the random vector $[(\bm{u}^{(0)})^\top \mid \cdots \mid (\bm{u}^{(t)})^\top]^\top \in \R^{2(t + 1)}$ (for each $t = 0, \ldots, T$) follows a centered multivariate Gaussian distribution with covariance matrix entries given by $\bm{C_g}$, and the random vector $[(\bm{\omega}^{(0)})^\top \mid \cdots \mid (\bm{\omega}^{(t)})^\top]^\top \in \R^{2(t + 1)}$ (for each $t = 0, \ldots, T$) follows a centered multivariate Gaussian distribution with covariance matrix entries given by $\bm{C_\theta}$.
    
    The deterministic $2 \times 2$ matrices are
    {\allowdisplaybreaks\begin{align*}
        \bm{R_\theta}(t + 1, s) &= \left\{\begin{array}{ll}
            \bm{0} &\text{if }t < s, \\
            \bm{I}_2 &\text{if }t = s, \\
            (\bm{I}_2 + \bm{\Gamma}^{(t)})\bm{R_\theta}(t, s) + \sum_{k=0}^{t-1} \bm{R_g}(t, k)\bm{R_\theta}(k, s) &\text{if }t > s
        \end{array}\right\} \in \R^{2 \times 2} \\
        \bm{R_g}(t, s) &= -\frac{\gamma}{\delta}\E\left[\Pr_{\eps^\star_0}\left[\abs{\eta^{(t)}_1 - \eta^{(t)}_2 - \eps^\star_0} < L\right]\left[\begin{array}{cc}
            1 & -1 \\
            0 & 0
        \end{array}\right] \frac{\partial\bm{\eta}^{(t)}}{\partial \bm{\omega}^{(s)}}\right] \in \R^{2 \times 2},\\
        \bm{\Gamma}^{(t)} &= -\frac{\gamma}{\delta}\E\left[\Pr_{\eps^\star_0}\left[\abs{\eta^{(t)}_1 - \eta^{(t)}_2 - \eps^\star_0} < L\right]\left[\begin{array}{cc}
            1 & -1 \\
            0 & 0
        \end{array}\right]\right] \in \R^{2 \times 2},\\
        \bm{C_\theta}(t, s) &= \E[\bm{\theta}^{(t)}(\bm{\theta}^{(s)})^\top] \in \R^{2 \times 2}, \\
        \bm{C_g}(t, s) &= \frac{\gamma^2}{\delta}\E\left[\begin{array}{cc}
            \E_{\eps^\star_0}[\eta^{(t)}_1 - \eta^{(t)}_2 - \eps^\star_0]_L\E_{\eps^\star_0}[\eta^{(s)}_1 - \eta^{(s)}_2 - \eps^\star_0]_L & 0 \\
            0 & 0
        \end{array}\right] \in \R^{2 \times 2}.
    \end{align*}
    Then, for any pseudo-Lipschitz functions $\psi: \R^{T + 1} \to \R$ and $\phi : \R^T \to R$,
    \begin{align*}
        \frac{1}{d}\sum_{j=1}^d \psi(\beta^\star_j, \beta^{(1)}_j, \ldots, \beta^{(T)}_j) &\toP \E[\psi(\beta^\star_0, \theta^{(1)}_1, \ldots, \theta^{(T)}_1)], \\
        \frac{1}{n}\sum_{i=1}^n \phi(\langle \bm{x}_i, \bm{\beta}^\star\rangle, \langle \bm{x}_i, \bm{\beta}^{(1)}\rangle, \ldots, \langle \bm{x}_i, \bm{\beta}^{(T-1)}\rangle) &\toP \E[\phi(\omega^{(0)}, \eta^{(1)}_1, \ldots, \eta^{(T)}_1)].
    \end{align*}}
\end{enumerate}
    
\end{theorem}

\begin{proof}
    Part \ref{thm:dp-sgd-huber-privacy} follows from standard properties of zCDP, such as \cref{thm:gaussian-mechanism-zcdp}. For part \ref{thm:dp-sgd-huber-utility}, consider the generalized first-order method initialized at $\bm{v}^{(0)} = [\bm{0} \mid \bm{\beta}^\star] \in \R^{d \times 2}$ and updated according to the following rule:
    \begin{align*}
        \bm{v}^{(t + 1)} &= h^{(t)}(\bm{v}^{(t)}; \bm{\xi}^{(t)}) + \bm{X}^\top g^{(t)}(\bm{r}^{(t)}) \in \R^{d \times 2}, \\
        \bm{r}^{(t)} &= \bm{X}\sum_{k=0}^t \bm{v}^{(k)} \in \R^{n \times 2},
    \end{align*}
    where we have defined the functions $g^{(t)} : \R^{n \times 2} \to \R^{n \times 2}$ and $h^{(t)} : \R^{d \times 2} \to \R^{d \times 2}$ in terms of a dummy variable $\bm{\eps}^\star = (\eps^\star_1, \ldots, \eps^\star_n) \iid \eps^\star_0$ as follows:
    \[
        g^{(t)}([\bm{r}^{(t)}_1 \mid \bm{r}^{(t)}_2]) = \left[-\gamma \E_{\bm{\eps}^\star}[\bm{r}^{(t)}_1 - \bm{r}^{(t)}_2 - \bm{\eps}^\star]_L \;\bigg|\; \bm{0}\right]
    \]
    and
    \[
        h^{(t)}([\bm{v}^{(t)}_1 \mid \bm{v}^{(t)}_2];\, \bm{\xi}^{(t)}) = \left[-\gamma \nu \bm{\xi}^{(t)} \;\Big|\; \bm{0}\right].
    \]
    It is straightforward to verify that $\bm{v}^{(t)} = [\bm{\beta}^{(t)} \mid \bm{\beta}^\star] \in \R^{d \times 2}$, where $\bm{\beta}^{(t)}$ is the $t$\textsuperscript{th} iterate of DP-SGD (\cref{alg:dp-sgd}). If $\bm{X}$ were an isotropic Gaussian, our desired result would follow directly from Theorem 3.2 of \cite{gerbelot2024meanfield} applied to this sequence, provided that we can verify their assumptions (A1), (A2), (A3.b) and (A4). Assumptions (A1) and (A2) are equivalent to our assumption (i) in part \ref{thm:dp-sgd-huber-utility}. The separability and smoothness of assumption (A3.b) are easily verified from the above expressions for $g^{(t)}$ and $h^{(t)}$. For assumption (A4), observe that $\bm{\beta}^\star \rightsquigarrow \beta^\star_0$ implies that $\frac{1}{d}\norm{\bm{\beta}^\star}^2$ converges to a finite constant, namely $\E(\beta^\star_0)^2$, as $n \to \infty$. Substituting the functions $g$ and $h$ and the partial derivatives of $g$ into the system of equations in their theorem, and then taking the limit as $n \to \infty$ and $d/n \to \delta$, yields our claimed system of equations, which we have partially simplified. For $\bm{X}$ following a general, subgaussian design, we simply apply \cref{thm:gfom-universality} (GFOM Universality, Theorem 3.2 of \cite{han2024entrywise}).
\end{proof}

\begin{theorem}[DP-SGD for Logistic Regression, ``Conditional Expectation'' Version]
\label{thm:dp-sgd-logistic}
Let $\bm{\hb}$ be the output of $T$ iterations of \cref{alg:dp-sgd} with parameters $R, \gamma, \nu > 0$, instantiated with the loss function
\[
    \ell(\bm{\beta}; (\bm{x}, y)) = \E_{\hat{y} \sim \mathrm{Bernoulli}(\rho'(y))}\bigl[\rho(\langle \bm{x}, \bm{\beta}\rangle) - \hat{y} \langle\bm{x}, \bm{\beta}\rangle \mid \bm{\beta}, (\bm{x}, y)\bigr]
\]
\begin{enumerate}[(a)]
    \item \label{thm:dp-sgd-logistic-privacy} \emph{\textbf{(Privacy)}} $\bm{\hb}$ satisfies $\pr$-zCDP for $\pr = T \cdot \frac{R^2}{2\nu^2}$.
    \item \label{thm:dp-sgd-logistic-utility} \emph{\textbf{(Utility)}} Suppose the following assumptions hold for some $\bm{\beta}^\star \in \R^d$ as $n \to \infty$ and $d/n \to \delta$:
    \begin{enumerate}[(i)]
        \item $\bm{X} \in B_R(\bm{0})^n \subseteq \R^{n \times d}$ follows a subgaussian design and $\bm{y} = \bm{X}\bm{\beta}^\star$.
        \item There exists a random variable $\beta^\star_0 \in \R$ such that $\bm{\beta}^\star \rightsquigarrow \beta^\star_0$.
    \end{enumerate}
    Consider the following $O(T^2)$ equations that recursively define, for all $t, s \in \{0, 1, \ldots, T\}$, the random variables $\bm{\theta}^{(t)}, \bm{\eta}^{(t)}, \bm{u}^{(t)}, \bm{\omega}^{(t)} \in \R^2$ and $\partial \bm{\eta}^{(t)}/\partial \bm{\omega}^{(s)} \in \R^{2 \times 2}$, as well as the deterministic $2 \times 2$ matrices $\bm{R_\theta}(t, s), \bm{R_g}(t, s), \bm{\Gamma}^{(t)}, \bm{C_\theta}(t, s), \bm{C_g}(t, s) \in \R^{2 \times 2}$.

    The random variables are
    {\allowdisplaybreaks\begin{align*}
        \bm{\theta}^{(0)} &= \left[\begin{array}{c}
            0 \\
            \beta^\star_0
        \end{array}\right] \in \R^2, \\
        \bm{\theta}^{(t + 1)} &= (1 + \bm{\Gamma}^{(t)})\bm{\theta}^{(t)} - \gamma \nu \left[\begin{array}{c} \xi^{(t)}_0 \\ 0\end{array}\right] + \sum_{k=0}^{t - 1}\bm{R_g}(t, k)\bm{\theta}^{(k)} + \bm{u}^{(t)} \in \R^2, \\
        \bm{\eta}^{(t)} &= -\gamma \sum_{k=0}^{t-1} \bm{R_\theta}(t, k) \left[\begin{array}{c}
            \rho'(\eta^{(k)}_1)\rho'(-\eta^{(k)}_2) - \rho'(-\eta^{(k)}_1)\rho'(\eta^{(k)}_2)\\
            0
        \end{array}\right] + \bm{\omega}^{(t)} \in \R^2, \\
        \frac{\partial \bm{\eta}^{(t)}}{\partial \bm{\omega}^{(s)}} &= \left\{\begin{array}{ll}
            \bm{0} &\text{if } t < s, \\
            \bm{I}_2 &\text{if } t = s, \\
            -\gamma \sum_{k=0}^{t-1}\bm{R_\theta}(t, k)\left[\begin{array}{cc}
            \rho''(\eta^{(k)}_1) & -\rho''(\eta^{(k)}_2) \\
            0 & 0
        \end{array}\right] \frac{\partial \bm{\eta}^{(k)}}{\partial \bm{\omega}^{(s)}} &\text{if } t > s
        \end{array}\right\} \in \R^{2 \times 2},
    \end{align*}}
    where $\xi^{(0)}_0, \ldots, \xi^{(T)}_0 \iid \mathcal{N}(0, 1)$, the random vector $[(\bm{u}^{(0)})^\top \mid \cdots \mid (\bm{u}^{(t)})^\top]^\top \in \R^{2(t + 1)}$ (for each $t = 0, \ldots, T$) follows a centered multivariate Gaussian distribution with covariance matrix entries given by $\bm{C_g}$, and the random vector $[(\bm{\omega}^{(0)})^\top \mid \cdots \mid (\bm{\omega}^{(t)})^\top]^\top \in \R^{2(t + 1)}$ (for each $t = 0, \ldots, T$) follows a centered multivariate Gaussian distribution with covariance matrix entries given by $\bm{C_\theta}$.
    
    The (deterministic / constant) $2 \times 2$ matrices are
    {\allowdisplaybreaks\begin{align*}
        \bm{R_\theta}(t + 1, s) &= \left\{\begin{array}{ll}
            \bm{0} &\text{if }t < s, \\
            \bm{I}_2 &\text{if }t = s, \\
            (\bm{I}_2 + \bm{\Gamma}^{(t)})\bm{R_\theta}(t, s) + \sum_{k=0}^{t-1} \bm{R_g}(t, k)\bm{R_\theta}(k, s) &\text{if }t > s
        \end{array}\right\} \in \R^{2 \times 2} \\
        \bm{R_g}(t, s) &= -\frac{\gamma}{\delta}\E\left[\left[\begin{array}{cc}
            \rho''(\eta^{(t)}_1) & -\rho''(\eta^{(t)}_2) \\
            0 & 0
        \end{array}\right] \frac{\partial\bm{\eta}^{(t)}}{\partial \bm{\omega}^{(s)}}\right] \in \R^{2 \times 2},\\
        \bm{\Gamma}^{(t)} &= -\frac{\gamma}{\delta}\E\left[\begin{array}{cc}
            \rho''(\eta^{(t)}_1) & -\rho''(\eta^{(t)}_2) \\
            0 & 0
        \end{array}\right] \in \R^{2 \times 2},\\
        \bm{C_\theta}(t, s) &= \E[\bm{\theta}^{(t)}(\bm{\theta}^{(s)})^\top] \in \R^{2 \times 2}, \\
        \bm{C_g}(t, s) &= \frac{\gamma^2}{\delta}\E\left[\begin{array}{cc}
            \substack{(\rho'(\eta^{(t)}_1)\rho'(-\eta^{(t)}_2) - \rho'(-\eta^{(t)}_1)\rho'(\eta^{(t)}_2)) \\ \times (\rho'(\eta^{(s)}_1)\rho'(-\eta^{(s)}_2) - \rho'(-\eta^{(s)}_1)\rho'(\eta^{(s)}_2))} & 0 \\
            0 & 0
        \end{array}\right] \in \R^{2 \times 2}.
    \end{align*}
    Then, for any pseudo-Lipschitz functions $\psi: \R^{T + 1} \to \R$ and $\phi : \R^T \to \R$,
    \begin{align*}
        \frac{1}{d}\sum_{j=1}^d \psi(\beta^\star_j, \beta^{(1)}_j, \ldots, \beta^{(T)}_j) &\toP \E[\psi(\beta^\star_0, \theta^{(1)}_1, \ldots, \theta^{(T)}_1)], \\
        \frac{1}{n}\sum_{i=1}^n \phi(\langle \bm{x}_i, \bm{\beta}^\star\rangle, \langle \bm{x}_i, \bm{\beta}^{(1)}\rangle, \ldots, \langle \bm{x}_i, \bm{\beta}^{(T-1)}\rangle) &\toP \E[\phi(\omega^{(0)}, \eta^{(1)}_1, \ldots, \eta^{(T)}_1)].
    \end{align*}}
\end{enumerate}
    
\end{theorem}

The proof of \cref{thm:dp-sgd-logistic} is essentially the same as that of \cref{thm:dp-sgd-huber}, so we omit it.

\bibliographystyle{alpha}
\bibliography{sources}

\end{document}